\documentclass[10pt,twocolumn,letterpaper]{article}

\usepackage{cvpr}              %
\usepackage[accsupp]{axessibility}  %

\PassOptionsToPackage{dvipsnames,table}{xcolor}

\usepackage{times}
\usepackage{epsfig}
\usepackage{graphicx}
\usepackage{amsmath}
\usepackage{amssymb}
\usepackage{xspace}
\usepackage{subcaption}
\usepackage{multirow}
\usepackage{booktabs}
\usepackage{multicol}
\usepackage{makecell}
\usepackage[table]{xcolor}  
\usepackage{tikz}
\usepackage{tkz-euclide}
\usepackage{pgfplots}
\usepackage{amsthm}
\newcommand{\PAR}[1]{\vskip4pt \noindent{\bf #1~}}

\newcommand{\M}[1]{\mathtt{#1}}
\newcommand{\V}[1]{\mathbf{#1}}

\def\sfc{{\texttt{5pt}}\xspace}

\def\sp{{\texttt{P3P}}\xspace}

\def\sft{{\texttt{5pt+P3P}}\xspace}
\def\sftENM{{\texttt{5pt+P3P+ENM}}\xspace}
\def\sftRCENM{{\texttt{5pt+P3P+R+F+ENM}}\xspace}

\def\sftRC{{\texttt{5pt+P3P+R+F}}\xspace}

\def\sftm{{\texttt{4p3v(M)}}\xspace}
\def\sftmENM{{\texttt{4p3v(M)+ENM}}\xspace}
\def\sftmR{{\texttt{4p3v(M)+R}}\xspace}
\def\sftmRC{{\texttt{4p3v(M)+R+F}}\xspace}
\def\sftmRCENM{{\texttt{4p3v(M)+R+F+ENM}}\xspace}
\def\sftmd{{\texttt{4p3v(M$\pm \delta$)}}\xspace}
\def\sftmdENM{{\texttt{4p3v(M$\pm \delta$)+ENM}}\xspace}
\def\sftmdR{{\texttt{4p3v(M$\pm \delta$)+R}}\xspace}
\def\sftmdRC{{\texttt{4p3v(M$\pm \delta$)+R+F}}\xspace}
\def\sftmdRCENM{{\texttt{4p3v(M$\pm \delta$)+R+F+ENM}}\xspace}

\def\sfhc{{\texttt{4p3v(HC)}}\xspace}

\def\sfa{{\texttt{4p3v(A)}}\xspace}
\def\sfaf{{\texttt{4p3v(A)}}\xspace}
\def\sfafENM{{\texttt{4p3v(A)+ENM}}\xspace}
\def\sfafR{{\texttt{4p3v(A)+R}}\xspace}
\def\sfafRC{{\texttt{4p3v(A)+R+F}}\xspace}
\def\sfafRCENM{{\texttt{4p3v(A)+R+F+ENM}}\xspace}

\newtheorem{theorem}{Theorem}

\newtheorem{lemma}[theorem]{Lemma}

\pgfplotsset{
compat=1.9,
legend image code/.code={
\draw[mark repeat=4,mark phase=1]
plot coordinates {
(0cm,0.02cm)
(0.1cm,0.02cm)        
(0.2cm,0.02cm)
(0.3cm,0.02cm)       
};%
}
}

\definecolor{Seaborn1}{HTML}{1f77b4}
\definecolor{Seaborn2}{HTML}{ff7f0e}
\definecolor{Seaborn3}{HTML}{2ca02c}
\definecolor{Seaborn4}{HTML}{d62728}
\definecolor{Seaborn5}{HTML}{9467bd}
\definecolor{Seaborn6}{HTML}{8c564b}
\definecolor{Seaborn7}{HTML}{e377c2}
\definecolor{Seaborn8}{HTML}{7f7f7f}
\definecolor{Seaborn9}{HTML}{bcbd22}
\definecolor{Seaborn10}{HTML}{17becf}

\pgfplotsset{width=10cm,compat=1.9}

\tikzstyle{bedge}=[line width=0.4pt, color=blue]
\tikzstyle{redge}=[line width=0.4pt, color=red]
\tikzstyle{gedge}=[line width=0.4pt, color=green]
\tikzstyle{bsvertex}=[circle, draw=black, fill=cyan, inner sep=0pt, minimum size=0.4pt]
\tikzstyle{rsvertex}=[circle, draw=black, fill=red, inner sep=0pt, minimum size=0.2pt]

\newcommand{\iccvnew}{
\begin{tikzpicture}
    \node[anchor=south west,inner sep=0] (image) at (0,0) {\includegraphics[width=0.85\columnwidth]{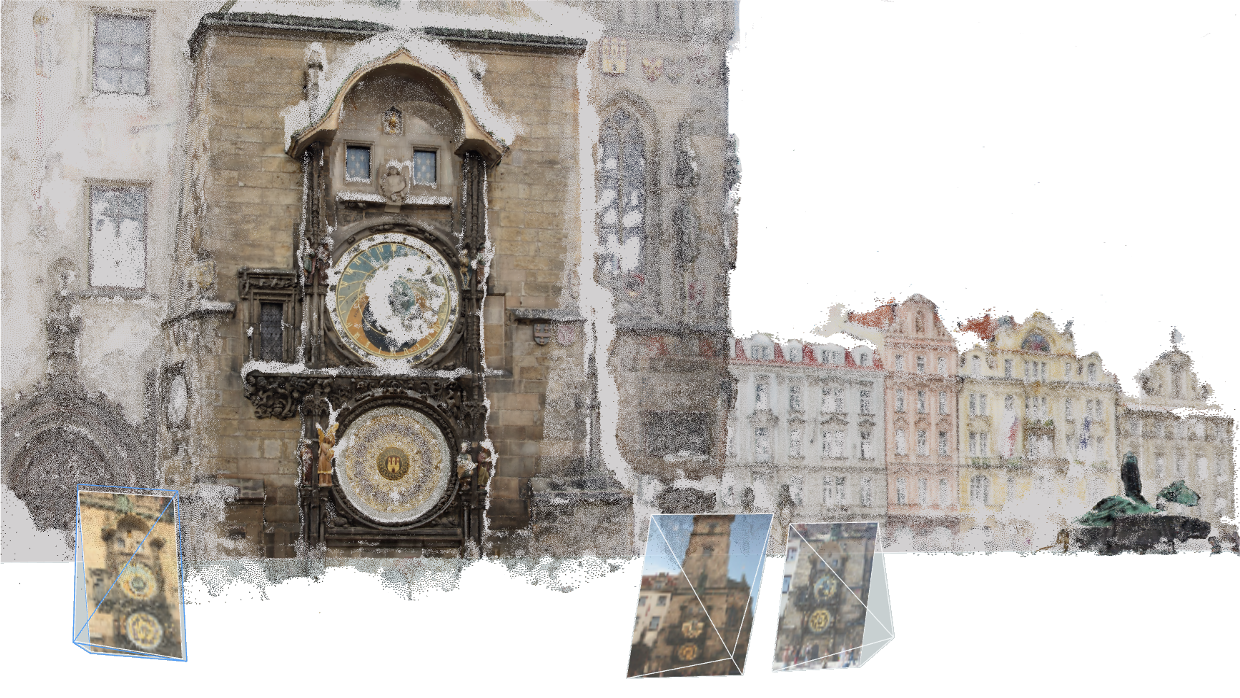}};
    \begin{scope}[x={(image.south east)},y={(image.north west)}]
        
        \draw[redge] (0.4521, 0.815) -- (0.06, 0.08);
        \draw[redge] (0.314, 0.627) -- (0.06, 0.08);
        \draw[redge] (0.39, 0.403) -- (0.06, 0.08);
        \draw[redge] (0.3321, 0.885) -- (0.06, 0.08);

        \draw[gedge] (0.4521, 0.815) -- (0.7202, 0.085);
        \draw[gedge] (0.314, 0.627) -- (0.7202, 0.085);
        \draw[gedge] (0.39, 0.403) -- (0.7202, 0.085);
        \draw[gedge] (0.3321, 0.885) -- (0.7202, 0.085);

        \draw[bedge] (0.4521, 0.815) -- (0.59, 0.06);
        \draw[bedge] (0.314, 0.627) -- (0.59, 0.06);
        \draw[bedge] (0.39, 0.403) -- (0.59, 0.06);
        \draw[bedge] (0.3321, 0.885) -- (0.59, 0.06);

        \draw[->, orange, thick, dashed] plot [smooth, tension=1] coordinates {(0.17, 0.05) (0.3, -0.02) (0.49, 0.04)};
        \node at (0.32, -0.07) {$\texttt{\scriptsize Approx. geometry}$};

        \draw[->, orange, thick, dashed] plot [smooth, tension=1] coordinates {(0.7302, 0.085) (0.7502, 0.425) (0.6202, 0.685)};
        \node at (0.7102, 0.675) {$\texttt{\scriptsize P3P}$};

    \end{scope}
\end{tikzpicture}
}

\usepackage[table]{xcolor}
\usepackage{letltxmacro}
\usepackage{xparse}

\definecolor{cvprblue}{rgb}{0.21,0.49,0.74}
\usepackage[pagebackref,breaklinks,colorlinks,allcolors=cvprblue]{hyperref}

\def\paperID{7411} %
\def\confName{CVPR}
\def\confYear{2025}

\title{
Practical solutions to the relative pose of three calibrated cameras}

\author{Charalambos Tzamos$^{\textrm{1,}\ast}$ \quad
Viktor Kocur$^{\textrm{2},\ast}$ \quad
Yaqing Ding$^\textrm{1}$ \quad
Daniel Barath$^\textrm{3}$ \quad
Zuzana Berger Haladová$^\textrm{2}$\\
Torsten Sattler$^\textrm{4}$ \quad
Zuzana Kukelova$^\textrm{1}$\vspace{10pt}\\
$^\textrm{1}$Visual Recognition Group, Faculty of Electrical Engineering, Czech Technical University in Prague\\
$^\textrm{2}$Faculty of Mathematics, Physics and Informatics, Comenius University in Bratislava\\
$^\textrm{3}$ETH Zürich (Zurich), HUN-REN SZTAKI (Budapest) \\
$^\textrm{4}$Czech Institute of Informatics, Robotics and Cybernetics, Czech Technical University in Prague\\
}

\definecolor{best}{rgb}{1.0, 0.6, 0.6}    
\definecolor{second}{rgb}{1.0, 0.8, 0.6}  
\definecolor{third}{rgb}{1.0, 1.0, 0.6}  

\LetLtxMacro{\OriginalCellColor}{\cellcolor}

\RenewDocumentCommand\cellcolor{m m}{%
  \OriginalCellColor{#1}
  #2
}

\begin{document}

\maketitle
\begin{abstract}
 We study the challenging problem of estimating the relative pose of three calibrated cameras from four point correspondences.
 We propose novel efficient solutions to this problem that are based on the simple idea of using four correspondences to estimate an approximate geometry of the first two views. We model this geometry either as an affine or a fully perspective geometry estimated using one additional approximate correspondence.
We generate such an approximate correspondence using a very simple and efficient strategy, where the new point is the mean point of three corresponding input points. 
The new solvers are efficient and easy to implement, since they are based on existing efficient minimal solvers, i.e., the 4-point affine fundamental matrix, the well-known 5-point relative pose solver, and the \texttt{P3P} solver. Extensive experiments on real data show that the proposed solvers, when properly coupled with local optimization, achieve state-of-the-art results, with the novel solver based on approximate mean-point correspondences being more robust and accurate than the affine-based solver. 
\end{abstract}
\vspace{-5.8mm}%
\section{Introduction}
\label{sec:intro}
\makeatletter
\def\blfootnote{\gdef\@thefnmark{}\@footnotetext}
\makeatother
\blfootnote{$^\ast$ Equal contribution}
Camera geometry estimation is crucial in many computer vision applications, \eg, visual navigation~\cite{DBLP:journals/ram/ScaramuzzaF11}, Structure-from-Motion~\cite{Snavely-IJCV-2008},
augmented 
reality~\cite{Castle08ISWC}, self-driving cars~\cite{hane20173d}, 
and 
visual localization~\cite{Sattler16PAMI}.
Due to
noise and outliers in  input correspondences, 
the predominant way for camera geometry estimation is to use a hypothesis-and-test framework, \eg,  RANSAC~\cite{Fischler-Bolles-ACM-1981,Chum-2003, DBLP:journals/pami/RaguramCPMF13, barath2017graph}. 
For RANSAC-like methods, using as few (ideally the minimal number of) correspondences as possible for estimation 
is important since the number of RANSAC iterations (and  thus its run-time)
grows exponentially with the number of correspondences required for the model estimation.

\begin{figure}[t]
     \centering
     \iccvnew
     \caption{
   Visualization of the four-points-in-three-views (4p3v) problem and our solution %
   based on %
   using four correspondences to efficiently estimate an approximate geometry of the first two views and then register the third view using a \texttt{P3P} solver~\cite{lambda-twist}.
      }\vspace{-12pt}
     \label{fig:teaser}
\end{figure}

Minimal camera geometry problems often result in complex systems of 
polynomial
equations.
Efficient algebraic methods helped solve many previously unsolved 
problems
~\cite{Stewenius-CVPR-2005,bujnak_cvpr2008,larsson2019revisiting,DBLP:conf/cvpr/KukelovaP07,kukelova2013real,Stewenius-ISPRS-2006}.
Still,  
they fail to generate efficient and/or numerically stable solutions for some configurations.
In this paper, 
we study one such challenging problem: 
Estimating the relative pose of three calibrated cameras. 
This problem has received attention
for a long time~\cite{holt1995,Quan2006,leonardos_cvpr2015,Martyushev16,Aholt2014}. 
However, due to its complexity, it is still not considered fully solved. There are no efficient and practical solutions for most of the minimal configurations of point and/or line correspondences~\cite{Kileel2017}.
One such configuration that is particularly interesting is the notoriously difficult configuration of four points in three calibrated views~\cite{Quan2006, Nister-5pt-PAMI-2004}, known as the 4p3v problem.  

State-of-the-art algebraic and numerical methods are known to fail in generating efficient and numerically stable solutions to the 4p3v problem.  
The existing methods for solving this problem are only approximate~\cite{Hruby_cvpr2022,DBLP:journals/ijcv/NisterS06}.
By solving only for one (or a few) solutions from the 272 solutions 
of the 4p3v problem~\cite{Hruby_cvpr2022}, and by discretely sampling the space of potential solutions~\cite{DBLP:journals/ijcv/NisterS06}, the existing 4p3v methods can often fail, \ie, the returned solution can be, in general, arbitrarily far from the geometrically correct solution. 
To decrease the failure rate, both methods~\cite{Hruby_cvpr2022,DBLP:journals/ijcv/NisterS06} require a lot of tuning and are challenging to re-implement.\footnote{%
There is no publicly available implementation for~\cite{DBLP:journals/ijcv/NisterS06}. 
Implementation of~\cite{Hruby_cvpr2022} %
is quite complex and requires a non-negligible effort to run.}

In contrast to complex solutions to the configuration of four points in three views, there is a simple and efficient solution for the configuration where five correspondences are detected in the first two views and three of these points are also visible in the third view.  
In this case, the \sft solver first estimates the relative pose of two 
cameras from five correspondences using an efficient \texttt{5pt} solver~\cite{Nister-5pt-PAMI-2004}, and then registers the third camera using a \texttt{P3P} solver~\cite{lambda-twist}.

Following the idea of first estimating the relative pose of two cameras, 
we propose two novel approaches 
for solving
the 4p3v problem. 
Our solutions are based on the simple idea
of using four correspondences to estimate an approximate geometry of the first two views and then registering the third camera using a \texttt{P3P} solver~\cite{lambda-twist}.

In the first
approach, 
we approximate the geometry of the first two views using affine cameras. The approximation using affine cameras was shown to provide good accuracy for relative pose estimation if coupled with geometry refitting and local optimization inside RANSAC~\cite{pritts_ivcnz13}. We denote the proposed affine-based 4p3v solver as \sfaf.

In the second approach, we use a %
less restrictive approximation.
In this case,
we estimate the full perspective geometry, \ie, the full 5DoF essential matrix, using four 
input
correspondences and one additional approximate 
correspondence in two views.
We generate the approximate correspondence using the locations of three of the 
four input 
correspondences. %
Under the assumption of a para-perspective projection, \ie, of affine geometry, the mean point of three 3D points is projected to the mean points in both images~\cite{Zhang2014}.
Thus,
the new correspondence is generated as the correspondence between the mean points of three corresponding points detected in two views. 
This approximate mean-point correspondence can also be seen as a correspondence under the $1^{st}$-order approximation of the homography defined by the plane passing through the three %
corresponding 3D points. We denote the proposed %
mean point-based 4p3v solver as \sftm.

While the proposed \sftm solver returns significantly more accurate poses than the \sfaf solver, both solvers may not be accurate enough if not properly treated inside RANSAC.
Thus, in this paper, we propose several ways to improve the accuracy of the proposed approximate solvers.
(i) To compensate for noise in the mean point correspondences, we introduce the \sftmd solver that generates two additional correspondences in the vicinity of the mean point, and in the first step calls the \texttt{5pt} solver~\cite{Nister-5pt-PAMI-2004} three times. 
(ii) To improve the approximate geometry estimated between the first two views, we refit this geometry using a non-minimal relative pose solver and inliers obtained from the approximate geometry (ENM). 
(iii) We use the fourth correspondence in the third view to filter out geometrically infeasible solutions (+F) and to refine the solutions 
on four input correspondences in three views
using just a few iterations of Levenberg-Marquardt (LM) refinement (+R). 
While conceptually 
simple and efficient, the novel solvers achieve state-of-the-art results 
on real data.

\noindent The contributions of the paper are as follows: 
\begin{itemize}
\item We propose two groups of novel solutions for the well-known and challenging 4p3v problem: %
\sfa- and \sftm-based solvers. 
 These solutions are based on the simple idea of using four correspondences to estimate an approximate geometry in the first two views. 
 Compared to state-of-the-art 4p3v solvers~\cite{Hruby_cvpr2022,DBLP:journals/ijcv/NisterS06}, which are non-trivial and difficult to re-implement such that they are numerically stable and fast, our new solutions
can be easily implemented using existing efficient implementations of the linear \texttt{4pt} affine fundamental matrix solver~\cite{hartley2006multiple}, the \texttt{5pt} solver~\cite{Nister-5pt-PAMI-2004} and the \texttt{P3P} solver~\cite{lambda-twist}. 
The source code is available at \ \href{https://github.com/kocurvik/threeview}{https://github.com/kocurvik/threeview} \ and \ \href{https://doi.org/10.5281/zenodo.16599943}{https://doi.org/10.5281/zenodo.16599943}. The data used are available and can be accessed at \ \href{https://doi.org/10.5281/zenodo.16603086}{https://doi.org/10.5281/zenodo.16603086}.
 \item We present several ways of improving the accuracy and the speed of the proposed solvers, as well as a strategy for efficiently using these approximate solvers inside a RANSAC-style paradigm. 
We show that the new solvers achieve state-of-the-art results in terms of accuracy on real data. While both \sfaf- and \sftm-based solvers achieve comparable results when coupled with the suggested non-minimal geometry refitting and local optimization, %
our \sftm-based solvers are more robust to the scene geometry and RANSAC inlier thresholds.
\item To 
our knowledge, we are the first to extensively evaluate solutions to the 4p3v problem on a large variety of real-world scenes and within state-of-the-art RANSAC frameworks, and to compare them to the baseline %
\sft solver.
We report results on 3 datasets, consisting of 18 different scenes and altogether 90,000 camera triplets.
\end{itemize}

\section{Related work}
Estimating the relative pose of three cameras from a minimal number of point and line  correspondences 
is known as an extremely challenging problem. 
For three uncalibrated cameras, 6 point correspondences are necessary to estimate the trifocal tensor, with a solution known for a long time~\cite{Quan_pami95,Torr97a}. 
 Solutions to three minimal combinations of points and lines are presented in~\cite{Oskarsson_bmvc2004}. 
The
configuration of 9 lines is
more challenging and was solved only recently 
~\cite{larsson2017efficient}.  
Yet,
the
solver is far from practical
and
runs 
17.8s.

For calibrated cameras, the configuration that attracts most of the attention is the configuration of four points in three views (the 4p3v problem). Note that this is not a minimal configuration since it generates 12 constraints for 11 degrees-of-freedom (DoF).
The 4p3v problem is known to be extremely difficult to solve.
Several papers present mostly theoretical results~\cite{leonardos_cvpr2015,Martyushev16,Aholt2014}.
For four triplets of exact points without noise, it is shown that the 4p3v problem has, in general, a unique solution~\cite{holt1995,Quan2006}.

To the best of our knowledge, there are only two reasonably efficient solutions to the 4p3v problem reported in the literature.
The first solver~\cite{DBLP:journals/ijcv/NisterS06} is based on a similar idea as our solvers, \ie, to first estimate the relative pose of two cameras and then register the third camera using a \texttt{P3P} solver~\cite{lambda-twist}. 
To compute the pose of the first two cameras using information only from four point correspondences, the solver needs  additional information about the position of one epipole.  
The paper shows that the four point correspondences between two calibrated views constrain the epipole in each image to lie on a curve of degree ten.
Thus, the solver performs one-dimensional exhaustive search and sweeps a $10^{th}$-degree curve of possible epipoles.
For each potential epipole, it computes the relative pose of two cameras, registers the third camera using three triangulated points, and finally extracts the solution minimizing the reprojection error of the fourth point in the third view. 
Evaluation of the solver on one potential epipole is fast.  
Yet, 
in contrast to our proposed \sftm-based solvers, the error of the sampled epipole for one fixed point on the curve is not bounded since the true epipole can lie anywhere on the curve, and in many scenarios it is very far from the image center (\eg, outside the image).
Thus, to obtain reasonably accurate and stable results, usually 1,000 candidates need to be evaluated. 
Even then, %
refinement at multiple local minima is required to improve the accuracy. The runtimes reported for this solver were $1-12ms$ depending on the number of points searched. On a small number of synthetic experiments, the paper shows that the translation error returned by the proposed solver is usually $1-10\deg$ higher than the error returned by the \sft solver.
The literature does not compare against~\cite{DBLP:journals/ijcv/NisterS06} as there is no publicly available implementation %
and it is hard to re-implement. 

The second efficient solver to the 4p3v problem was published only recently~\cite{Hruby_cvpr2022}. In this paper, the authors first transform the 4p3v problem into a minimal problem by considering a line passing through the last correspondence in the third view.
The resulting system of equations is solved using an efficient Homotopy continuation (HC) method~\cite{Fabbri_CVPR2020,SommeseAndrewJ2005Tnso}. 
To avoid computing large numbers of spurious solutions, an MLP-based classifier is trained. For a given problem $p$, it selects one or several starting problem-solution pairs (so-called anchors), such that the geometrically meaningful/correct solution of $p$ can be obtained by HC starting from this anchor. This strategy is fast, running $16.3\mu s$ on average per solution. 
However, it has a high failure rate. The success rate of the 4p3v solver reported in~\cite{Hruby_cvpr2022} %
on two test %
datasets and data without noise is $26.3\%$.  
\cite{Hruby_cvpr2022} 
 do not show results %
for a real scenario, \ie, a RANSAC-like framework with noisy data. 
Providing such an evaluation, we show that our much simpler solvers outperform~\cite{Hruby_cvpr2022}. 
Solutions to the 4p3v problem for orthographic and 
para-perspective views were presented in~\cite{xu_ortho4p3v,Higgins-4p3v91}. In~\cite{Higgins-4p3v91}, the author suggested an iterative approach for updating to perspective views, but  
reported results only on a few synthetic instances. %
According to our 
extensive 
experiments, the update does not work on real data with general perspective cameras. This solver is returning large errors even after incorporating it into RANSAC with local optimization.  We see two main reasons: (i) The set of inliers that satisfy an approximate para-perspective camera model in all three images is usually quite small, and (ii) the %
estimated
model 
is often very far from the perspective optimum, and thus %
local optimization does not converge to a good solution.
In~\cite{Duff_PL1P,Kileel2017,duff2019plmp}, the authors aim to classify and derive 
the number of solutions for different minimal configurations of points and lines in three calibrated views.
Solutions to two minimal configurations
combining
points and lines 
were proposed in~\cite{Fabbri_CVPR2020} %
and solved using a HC %
method~\cite{SommeseAndrewJ2005Tnso}. 
Due to their complexity, 
the solvers are not practical. %
A GPU HC method
was also used to solve minimal problems of four points/six lines in three views %
for a generalized 
camera in~\cite{Ding_2023_ICCV} and of four points in three cameras with an unknown shared focal length (4p3vf) in~\cite{Chien_2022_CVPR,Cin_2024_CVPR}.
Efficient GPU implementations 
of the 4p3vf solvers
run $16.7ms$ to $154ms$. These times are still too slow for practical applications.

An approximation of perspective cameras using affine ones was used in several papers on camera geometry estimation~\cite{Collins2014InfinitesimalPP}. 
\cite{oberkampf1996} and~\cite{Horaud97} propose iterative approaches to the absolute camera pose estimation problem, 
\ie, the PnP problem.  
Both methods first compute the pose for an affine camera 
(weak-
or para-perspective). Then the error induced by the affine camera approximation is used to adjust the constraints on the pose and recompute it. 
These methods do not guarantee convergence to the perspective model solution.
In~\cite{pritts_ivcnz13}, affine cameras are used to efficiently solve the relative pose problem of two uncalibrated cameras from two affine correspondences. 
These affine correspondences are 
transformed to six point or two ellipse correspondences and 
then
used to estimate the affine fundamental matrix. 
The paper also proposes a RANSAC framework that uses local optimization~\cite{lebedaLO} to estimate the full fundamental matrix using inliers from the approximate affine model.

Points that are sampled based on feature geometry to generate point correspondences from affine or scale- and orientation-invariant feature correspondences were also  used %
in~\cite{perd2006epipolar,BEEP,pritts-cvpr2018}.
In contrast,
our M-based solvers %
use 
only point correspondences, without associated feature geometry, to generate an additional point correspondence. %

\section{Estimating the relative pose of three cameras}
\label{sec:solvers}
In this section, we describe different solutions for estimating the relative pose of three calibrated cameras. 
We start with a baseline solution for the minimal configuration of three points visible in all three cameras and two additional points visible in two of the three cameras (the (5,5,3) configuration). 
Next, we present our novel solutions for the configuration of four points visible in all three cameras (the (4,4,4) configuration). 
This configuration generates an over-constrained problem. In this case, we have one more constraint than DoF. 
A minimal solution would need to drop one constraint, \eg, %
by considering only a line passing through one of the points in the third view~\cite{Hruby_cvpr2022} or by considering a ``half" point correspondence.
 Since, in practice, we always have full correspondences and sampling one less point in one view leads to an under-constrained problem, the (4,4,4) configuration, is usually considered  ``minimal".

\PAR{5pt+P3P solver:}
The \sft  solver first estimates the relative pose of two cameras from 5 image point correspondences using the efficient \texttt{5pt} solver~\cite{Nister-5pt-PAMI-2004}. 
Next, the three points %
visible in all three views are triangulated. Finally, the third camera is registered using the three 2D-3D point correspondences and the well-known efficient \texttt{P3P} solver~\cite{lambda-twist}. 
 This straightforward solver, which is based on existing efficient solvers~\cite{Nister-5pt-PAMI-2004,lambda-twist}, was discussed in several works~\cite{Duff_PL1P,Nister-5pt-PAMI-2004,Rodehorst2017,DBLP:journals/ijcv/NisterS06}. 
 \cite{DBLP:journals/ijcv/NisterS06} showed that the \sft solver performs better than their dedicated 4p3v solver on synthetic data.
Yet, %
the most recent works~\cite{Fabbri_CVPR2020,Hruby_cvpr2022} that study the three view relative pose problem %
do not discuss the \sft solver %
and do not use it as a baseline for comparison.  
 To the best of our knowledge, the performance of this solver on real data and within state-of-the-art RANSAC frameworks in the context of the 4p3v problem has not been extensively studied. %
 Our paper fills 
 this gap in the literature.

\subsection{Approximate solutions to the 4p3v problem}
In contrast to the (5,5,3) configuration, the (4,4,4) configuration, \ie, the configuration of four points in three views, leads to significantly more complex equations. State-of-the-art algebraic and numerical methods are known to fail in generating efficient and numerically stable solutions to these equations. 
Thus, %
solutions to the 4p3v problem require some approximations to be practical in real-world applications. 
This section introduces several practical approximate solutions to the 4p3v problem. 
Similarly to the \sft solver, we decompose the problem into the problem of first estimating the relative pose of two cameras and then registering the third camera with an efficient \texttt{P3P} solver~\cite{lambda-twist}. The idea is to assume an approximate geometry only in the first two views and use four input point correspondences to efficiently estimate this geometry. Next we describe two groups of such solvers.

 \PAR{4p3v(A) solver:}
The first group of solvers %
approximates the geometry in the first two views using affine cameras. Although the approximation can be quite far from the correct geometry, %
\cite{pritts_ivcnz13} showed that an approximate  affine camera model provides %
good accuracy for relative pose estimation if coupled with geometry refitting (using the full 7DoF fundamental matrix) and local optimization inside RANSAC. 
The proposed \sfaf solver first uses four point correspondences in two views to efficiently estimate the affine fundamental matrix $\M F_\M{A}$~\cite{hartley2006multiple} and then registers the third camera using three triangulated points and a \texttt{P3P} solver~\cite{lambda-twist}.

 \PAR{4p3v(M) solver:}
The affine camera model used in the \sfa-based solvers can be quite imprecise. This motivates us to introduce solvers that are based on a less restrictive approximation. 
In this case, we approximate only one point correspondence in two views. 
Under the assumption of a para-perspective projection, \ie, of affine geometry, it is known that the mean point of three 3D points is projected to the mean points of their projections in both images~\cite{Zhang2014}.
Thus, we generate one new approximate correspondence in two views as the correspondence between the mean points of three corresponding points detected in these views. 
Let $\V m^l$ be the mean point 
of three points $\left\{\V x_i^l,\V x_j^l,\V x_k^l\right\}, \; i,j,k \in \left\{1,\dots, 4\right\}$, in  the view $l \in \left\{1,2\right\}$, then $\V m^1 \leftrightarrow \V m^2$ is our new approximate correspondence.
The new \sftm solver solves the 4p3v problem by first estimating the full 5DoF essential matrix from four original correspondences $\V x_i^1 \leftrightarrow \V x_i^2,\;  i = 1,\dots, 4$  and the $5^{th}$ correspondence  $\V m^1 \leftrightarrow \V m^2$. This is done using the efficient \texttt{5pt} relative pose solver~\cite{Nister-5pt-PAMI-2004}. 
As with previous solvers, %
the third camera is registered using a  \texttt{P3P} solver~\cite{lambda-twist}. 
In short, %
the \sftm solver solves the 4p3v problem using the \sft solver and one approximate correspondence.

As we show on large amounts of synthetic and real data (see Sec.~\ref{sec:experiments} and Supp.~mat.~(SM)), the  \sftm solver provides much more accurate estimates than the \sfaf solver. %
Moreover, even in the basic form (without the modifications presented in the next section), on many scenes it performs comparably to the state-of-the-art HC solver~\cite{Hruby_cvpr2022}.
This can be attributed to several facts and observations:
 (1)  The $\V m^1 \leftrightarrow \V m^2$ correspondence does not need to be seen as a correspondence of points that are projections of the mean point of three 3D points  $\V X_i, \V X_j,$ and $\V X_k$ (in which case both $\V m^1$ and $\V m^2$ would have some error).  
 We can look at this correspondence as a correspondence in which we fix a point in one view, \eg, $\V m^1$, and generate a corresponding point in the second view. 
 In this case, to generate a good correspondence, we only require 
 the point in the second view to be reasonably close to the epipolar line defined by the mean point $\V m^1$ in the first view, \ie, the 2D point does not need to correspond to one particular 3D point with a given depth.
(2) 
{
 The ray from the center of the first camera through the mean point $\V m^1$  intersects the plane of $\V X_i, \V X_j,$ and $\V X_k$ at  point $\V M$, which lies inside the triangle formed by these 3D points. The projection of $\V M$ into the second camera lies inside the triangle formed by $\left\{\V x_i^2,\V x_j^2,\V x_k^2\right\}$. By construction, this projection lies on the epipolar line $\V E \V m^1$. 
 Thus the epipolar line defined by  $\V m^1$  passes through the triangle defined by $\left\{\V x_i^2,\V x_j^2,\V x_k^2\right\}$. 
 Consequently, the maximum distance of $\V m^2$ in the second image from the epipolar line is bounded by the maximum distance of $\V m^2$ from $\left\{\V x_i^2,\V x_j^2,\V x_k^2\right\}$.  For a formal lemma and proof, 
 see SM.}
%
%
%
%
(3) For practical applications, when used in RANSAC, it is not necessary that each triplet $\left\{\V x_i^2,\V x_j^2,\V x_k^2\right\}$ generates a good correspondence $\V m^1 \leftrightarrow \V m^2$. 
Samples with a high level of noise in the mean-point correspondence are filtered inside RANSAC.\footnote{%
This property was also used in the %
HC solver~\cite{Hruby_cvpr2022}, which completely fails for many samples. %
These samples are filtered within RANSAC.} 
On a large number of different scenes, we observed that even if some image pairs have triplets of points that generate very noisy mean-point correspondences, there are usually enough triplets for which the noise in $\V m^2$ is reasonably small to 
lead to good estimates. 
(4) Four point correspondences in two views usually fix the space of possible poses such that the $5^{th}$ correspondence, even if noisy, 
often generates  a pose that is not very far from the ground truth pose. 
Such a pose is usually sufficient for filtering out outliers and 
a good initialization for non-linear optimization on the original four points in three views and subsequent local optimization on detected inliers.
We support our observations by experiments on a large amount of data. 

\subsection {Making approximate solvers practical}
\label{sec:making_solvers_practical}
The \sfa-based solvers, when used without any modifications, generally provide imprecise results even when used inside RANSAC with local optimization on three views. While the accuracy of the pure \sftm solver inside LO-RANSAC~\cite{lebedaLO} (RANSAC with Local Optimization) is much better, there is still room for improvement. Here we present several simple modifications of these solvers that significantly boost their performance.

 \PAR{4p3v(M$\pm \delta$) solver:}
The mean point correspondence used in the \sftm solver can provide a good approximation of a correct correspondence. %
Yet, as mentioned above, it can also be noisy.
In the \sftmd solver, we thus, in addition to the mean point $\V m^2 = \left[x,y\right]$ of 
three points $\left\{\V x_i^2,\V x_j^2,\V x_k^2\right\}$
in the second image, generate two additional points. These points are (1) $\V {m}^2_{\pm\delta} = \left[x\pm\delta,y\right]$ if the longest dimension of the triangle  $\mathcal{T}^2 = \Delta\left\{\V x_i^2,\V x_j^2,\V x_k^2\right\}$ is in the x-direction or (2) $\V {m}^2_{\pm\delta} = \left[x,y \pm\delta\right]$ if it is in the y-direction.
All three points, \ie, $\V m^2$, $\V m^2_{-\delta}$, and $\V m^2_{+\delta}$ are placed in  correspondence with the mean point $\V m^1$.
The \sftmd solver in the first step calls the \sfc solver~\cite{Nister-5pt-PAMI-2004} three times, with the $5^{th}$ correspondence being either $\V m^1 \leftrightarrow \V m^2$, $\V m^1 \leftrightarrow \V m^2_{-\delta}$, or $\V m^1 \leftrightarrow \V m^2_{+\delta}$. The results of these three \sfc solvers are collected to create hypotheses for the relative pose of the first two cameras inside RANSAC. The shift $\delta$ is selected relative to the size of the triangle $\mathcal{T}^2$. %

 \PAR{Early non-minimal refitting (ENM):} The geometry estimated between the first two cameras using the proposed \sfa and \sftm- based solvers is only approximate. In %
 standard LO-RANSAC, such a geometry is optimized in local optimization (LO) after registering the third view. However, this can lead to propagation of errors into triangulated points and subsequently into errors in the pose of the third camera. Inliers in three views \wrt such cameras together with imprecise pose initializations may not be sufficient for LO to converge to a good solution. We observed this especially for \sfa-based solvers. Fortunately, the design of our solvers allows us to optimize the geometry already %
 after estimating the approximate geometry between the first two cameras. Here, even a very imprecise geometry returned by the \sfa solvers is usually sufficient to filter out outliers. Thus, after running the \texttt{4pt} $\M F_{\M{A}}$/\texttt{5pt} solver in the first step of the proposed solvers, we use the estimated approximate models to detect inliers in two views. %
 We then refit the estimated geometry using the non-minimal version of the \texttt{5pt} solver~\cite{Nister-5pt-PAMI-2004},
 which instead of a 4-dim null space of a $5\times 9$ matrix uses the last four vectors from the SVD/QR decomposition of a $n \times 9$ matrix.

\PAR{$\mathbf{4^{th}}$ point in the third view:} 
\sfa and \sftm solvers 
actually solve the configuration %
(4,4,3), \ie, they do not use the information from the 
point $\V{x}^3_4$ in the third view.
The information from %
$\V{x}^3_4$
 can be used 
 in two different ways: 
 
  \noindent \textbf{Filtering (}\texttt{+F}\textbf{):} $\V{x}^3_4$ can be used to filter out geometrically infeasible solutions returned by
 the proposed solvers. Note that the P3P and \texttt{5pt} solvers used inside the proposed methods and the $\delta$-based strategy return multiple solutions that can be evaluated \wrt $\V{x}^3_4$ to filter out spurious solutions. 
 Since the returned solutions can be affected by %
 the proposed approximation, 
 we do not simply select the solution with the smallest error on $\V{x}^3_4$, but we keep all solutions that have an  epipolar error on $\V{x}^3_4$ smaller than twice the threshold used inside RANSAC. 
 Our experiments show that this filtering can improve the speed of the proposed solvers. 
 However, as a trade-off, there is sometimes a small drop in the accuracy of the solvers since, in some cases, geometrically correct solutions are filtered out. 
 
 \noindent \textbf{Refinement (}\texttt{+R}\textbf{):} 
 $\V{x}^3_4$ can be used to refine the solutions returned by the proposed
 solvers.
 We want solutions that minimize the 
 epipolar error on the original 4 points in all 3 views, \ie, that solve the original (4,4,4) configuration. 
 (4,4,4)
 is an overconstrained configuration; thus, for noisy data, there is, in general, no solution with zero error on all 4 points in 3 views. 
 We refine the poses by minimizing the epipolar error of 
 the original four points in three views using %
 LM optimization, initialized using the solutions from the \sfaf- and \sftm-based solvers. 
 Experiments with different numbers of iterations %
 show that two iterations are usually sufficient to obtain an improvement %
 (see SM).

\section{Experiments}
\label{sec:experiments}
We extensively evaluated the proposed solvers on a large variety of synthetic and real data to test their robustness to noise, outliers, and scene properties, and to assess their performance inside state-of-the-art RANSAC-frameworks~\cite{barath2017graph, PoseLib}. 
We compare our novel solvers with the homotopy continuation \sfhc solver ~\cite{Hruby_cvpr2022} and the \sft baseline minimal solver for the (5,5,3) 
configuration.

\begin{figure}
    \centering

\begin{tikzpicture} 

        \begin{axis}[%
        hide axis, xmin=0,xmax=0,ymin=0,ymax=0,
        legend style={draw=white!15!white, 
        line width = 1pt,
        legend  columns =7, %
        /tikz/every even column/.append style={column sep=0.05cm},
        font=\scriptsize
        },
        legend image post style={xscale=1}
        ]
        
        \addlegendimage{Seaborn2}
        \addlegendentry{\texttt{5p(E)}};
        \addlegendimage{Seaborn4}
        \addlegendentry{\texttt{4p(M)}};
        \addlegendimage{Seaborn5}
        \addlegendentry{\texttt{4p(M$\pm \delta$)}};
        \addlegendimage{Seaborn3}
        \addlegendentry{\texttt{4p(A)}}; 
        \addlegendimage{black!30,dash pattern=on 2pt off 1pt on 2pt off 1pt}
        \addlegendentry{w/o \texttt{ENM}};
        \addlegendimage{black!30,dash pattern=on 1pt off 0.5pt on 1pt off 0.5pt}
        \addlegendentry{w/ \texttt{ENM}};
        \end{axis}
    \end{tikzpicture}
    \includegraphics[trim={1.5cm 0 1.5cm 2cm},clip, width=0.85\columnwidth]{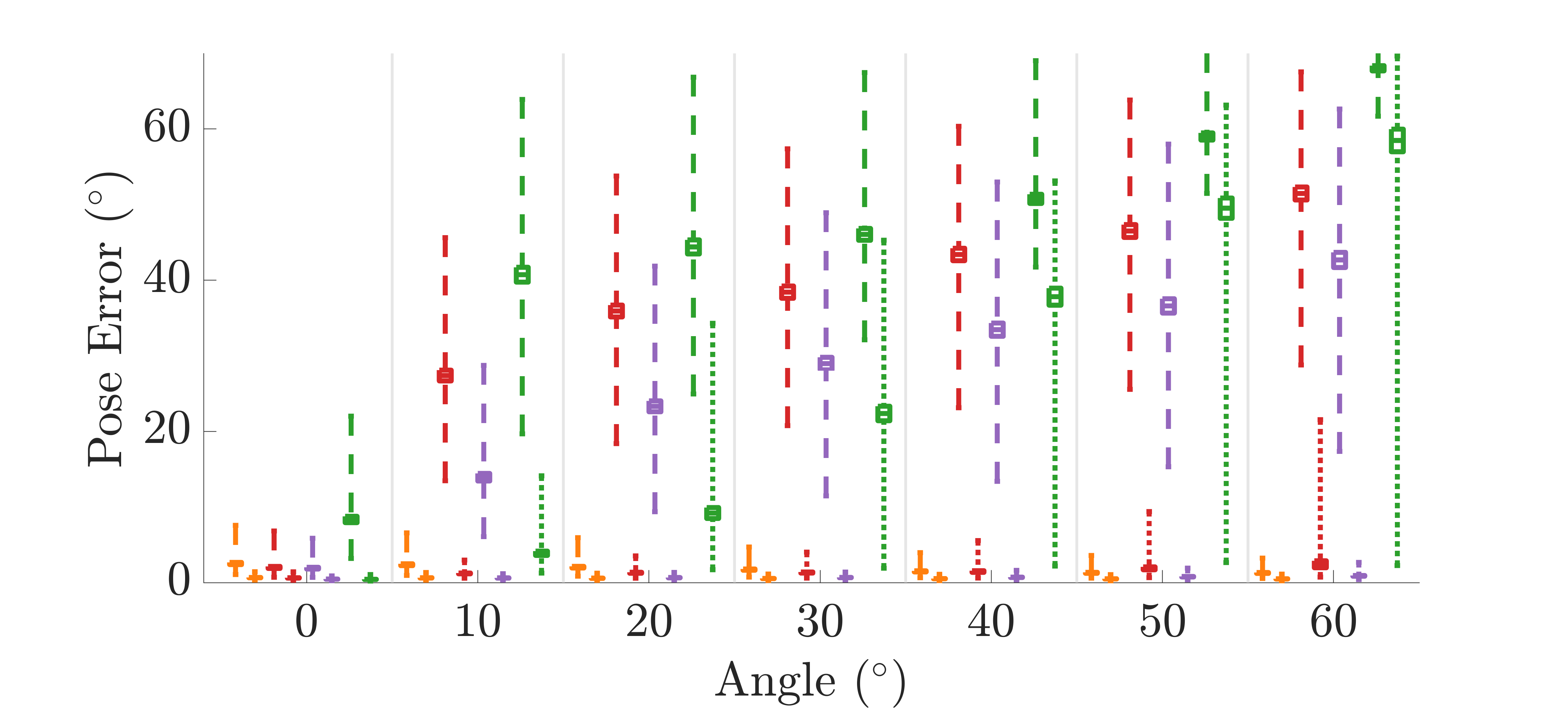}
    \caption{Results of a synthetic experiment measuring the accuracy of two-view variants of our solvers depending on the  angle between the principal axes of the cameras.}
    \label{fig:synth_angle}
\end{figure}

\PAR{Experimental setup.} To obtain feature correspondences, we use SuperPoint~\cite{detone2018superpoint} features with the  LightGlue~\cite{lindenberger2023lightglue} matcher. 
We extract at most 2048 features per image. 
We perform matching for all three image pairs and keep only those matches that were consistently matched across all three views. 
We perform evaluation within two RANSAC frameworks: PoseLib~\cite{PoseLib} and GC-RANSAC~\cite{barath2017graph}. For the \texttt{5pt} solver, %
we use~\cite{Nister-5pt-PAMI-2004} and for the \sp solver, %
we use~\cite{lambda-twist}. 
In PoseLib, we perform LO~\cite{Chum-2003} using LM optimization. 
In GC-RANSAC, we perform LO using non-minimal solvers~\cite{Nister-5pt-PAMI-2004,dlspnp} for fitting models to larger-than-minimal samples. 
We evaluated different shifts for our~$\delta$-based solvers using a validation scene (for the ablation study, see SM) and selected $\delta = 0.08*\texttt{\small(longest triangle dim.)}$. 
The choice of $\delta$ is not critical, \ie, different values perform similarly. %

\PAR{Evaluation measures.} Inspired by~\cite{IMC2020}, we define the %
pose error %
as $\text{max}\left(0.5 (\M R_{err}^{12} + \M R_{err}^{13}), 0.5 (\V t_{err}^{12} + \V t_{err}^{13})\right)$, where $\M R_{err}^{ij}$ and $\V t_{err}^{ij}$ are the angular errors of rotation and translation for pair $ij$ in degrees~\cite{IMC2020}. We also report AUC values~\cite{IMC2020} at different thresholds for the pose error. 
We include 
results for an alternative pose error definition which includes $\M R_{err}^{23}$ and $\V t_{err}^{23}$ %
in SM.

 \begin{figure}[t!]
    \centering
    \includegraphics[width=0.41\columnwidth]{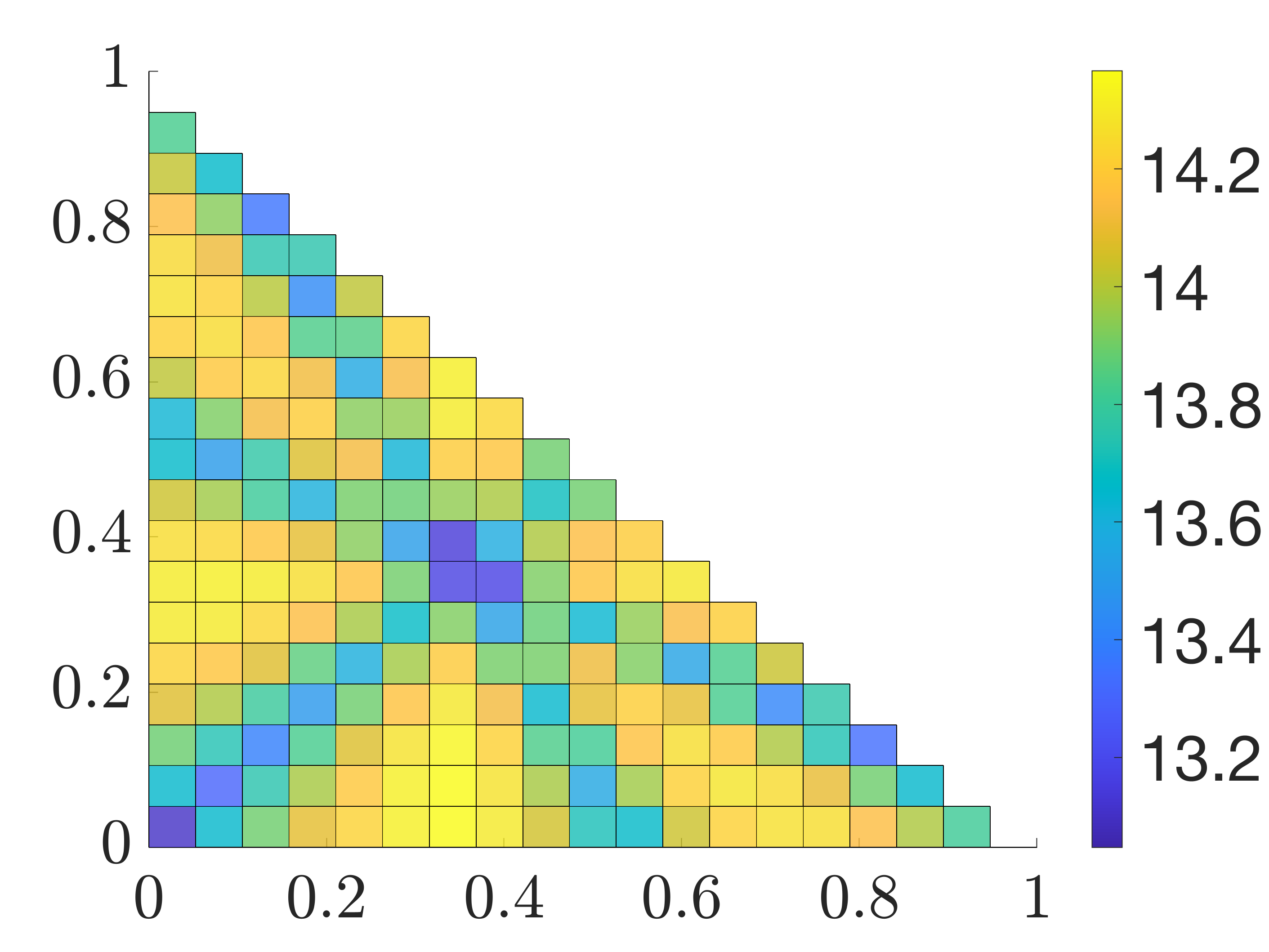}
    \includegraphics[width=0.41\columnwidth]{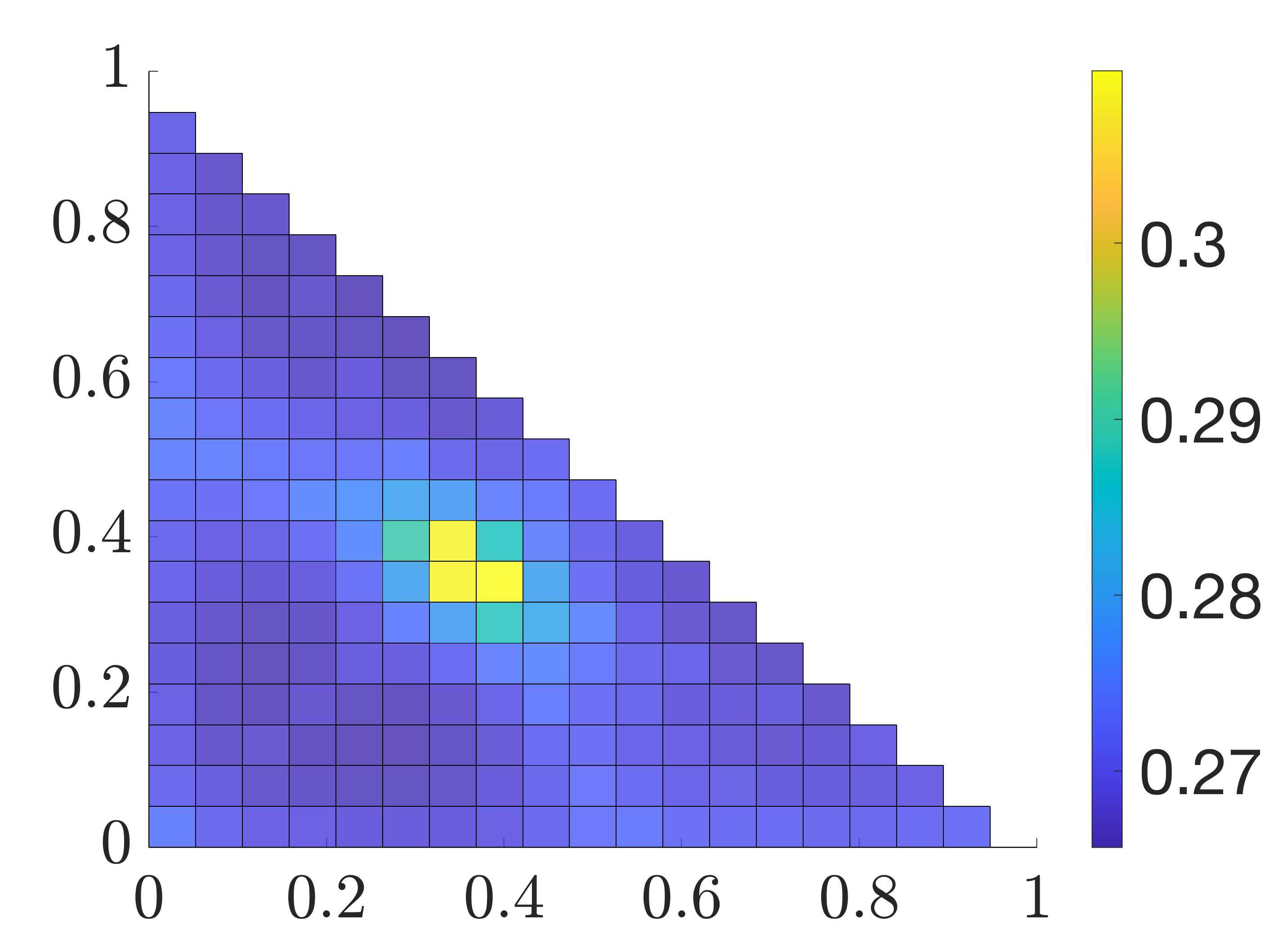}

\caption{
Distribution of the (left) rotation error (0.3373, 0.3349); and (right) percentage of inliers gathered (0.3266, 0.3434), as a function of the barycentric coordinates of the triangle in the second image \wrt the mean point of the corresponding triangle in the first image
on 465k four-tuples of correspondences from \textit{St.~Peter's Square} scene from the PhotoTourism dataset~\cite{IMC2020}. 
We fit a 2D Gaussian distribution to the results and report the mean in brackets.
}
\label{fig:mean_tests}
\end{figure}

\PAR{Approximate camera geometry.} The first experiments aim to support our idea of 
estimating approximate geometry in the first two views.
The accuracy of the
used approximations depends on a large number of variables, including the depths of the points \wrt the cameras, the angle under which the points are observed, 
the type of motion,  \etc. 
A detailed analysis of all these factors, \eg, through synthetic experiments, is beyond the scope of this paper. 

An approximation error introduced by the para-perspective projection, \ie, the affine geometry, 
and its effect on absolute camera pose determination is 
studied in the literature~\cite{Zhang2014,Horaud97}. Here, we thus only study the effect of the proposed approximation on relative pose estimation for real-world data and interesting synthetic scenarios. In the following experiments, all solvers are tested on two views without RANSAC. The two-view variants of our solvers are denoted as $\texttt{4p(A)}$, $\texttt{4p(M)}$, and \texttt{4p(M$\pm \delta$)}.

In the first set of experiments, we consider several interesting synthetic setups. 
Fig.~\ref{fig:synth_angle} presents the results for increasing the angle between the principal axes 
of the two cameras. 
We generate  3D points in a 2000x2000x100 cube uniformly at random. 
The distances of the camera centers are random between 1000 and 1200. The cameras are looking towards the scene and their principal axes form a specified angle. To simulate realistic scenarios, we add 1px noise to the image correspondences and 20\% outliers.\footnote{Outliers are only used in the ENM part to simulate that for approximate geometry outliers might be contained in the non-minimal samples.}
It can be seen that the results for both approximate solvers deteriorate with increasing angle, with $\texttt{4p(A)}$ being significantly less accurate than $\texttt{4p(M)}$ and \texttt{4p(M$\pm \delta$)}. However, ENM can significantly decrease errors, with the $\texttt{4p(M$\pm \delta$)+ENM}$ solver returning almost identical results to $\texttt{5pt+ENM}$. 
Results for %
varying distances of the cameras from the scene, depth of the scene, and image noise are in SM. 

Similar observations as on synthetic data can be made on real-world data. Tab.~\ref{tab:4ptM_vs_5pt} shows results 
on 
different 
scenes from the PhotoTourism dataset~\cite{IMC2020}. 
 As can be seen, the approximate solvers are not as accurate as the \texttt{5pt} solver when used without ENM.
The gap between the  $\texttt{4p(M)}$/\texttt{4p(M$\pm \delta$)} solvers and the \texttt{5pt} solver is for some scenes larger %
(%
\textit{Sacre Coeur}, 
\textit{Trevi Fountain}), and for some noticeably smaller (\textit{Reichstag}, \textit{St. Peter's Square}), showing that the performance of our solvers is scene-dependent. The $\texttt{4p(A)}$ solver returns large errors. 
The accuracy of the $\texttt{4p(M)}$/\texttt{4p(M$\pm \delta$)}  solvers is not very far from the \texttt{5pt} solver, with the  \texttt{4p(M$\pm \delta$)} solver even slightly outperforming the \texttt{5pt} solver on some scenes (\textit{Reichstag}, \textit{Taj Mahal}). ENM helps to increase the precision of all solvers, especially $\texttt{4p(A)}$.
 Still, the $\texttt{4p(A)+ENM}$ solver %
 returns quite high errors.  
However, in practice, all solvers are used inside RANSAC, where the current best pose is locally optimized on inliers. LO compensates for less accurate pose estimates and, as we show on real experiments for three %
as well as two views (see SM), it can suppress the errors of the $\texttt{4p(A)}$/\sfaf-based solvers, making them comparable to the $\texttt{4p(M)}$/\sftm-based solvers.
For the three view scenario, the performance of all of our solvers can be further improved by %
pose refinement (+R) and filtering (+F). 
%
%
%

%

%
%
%
%
%
%
%
%
%
%
%
%
%
%
%
%
%
%
%
%
%

%
%
%
%
%
%
%
%
%
%
%
%
%
%
%
%
%
%
%
%
%

\begin{table}  \resizebox{0.95\columnwidth}{!}{\begin{tabular}{c | c c c c || c c c c}
    \multicolumn{1}{c}{~} & \multicolumn{4}{c}{MED ($^\circ$)} & \multicolumn{4}{c}{$20^{th}$ perc. ($^\circ$)} \\
    \cmidrule{2-9}
    \multicolumn{1}{c}{~} & \multicolumn{8}{c}{w/o \texttt{ENM}} \\
    \midrule
    Scene & \texttt{5p(E)} & \texttt{4p(M)} & \texttt{4p(M$\pm \delta$)} & \texttt{4p(A)} & \texttt{5p(E)} & \texttt{4p(M)} & \texttt{4p(M$\pm \delta$)} & \texttt{4p(A)}\\
    \midrule
    \textit{Brandenburg Gate} & 16.14 & 19.02 & 15.69 & 75.32 & \phantom{1}7.22 & \phantom{1}9.83 & \phantom{1}7.46 & 45.99 \\
\textit{Buckingham Palace} & 18.87 & 21.12 & 17.74 & 66.19 & \phantom{1}8.04 & 10.01 & \phantom{1}8.23 & 40.59 \\
\textit{Colosseum Exterior} & 17.40 & 23.50 & 18.93 & 66.67 & \phantom{1}6.44 & 11.36 & \phantom{1}8.34 & 40.82 \\
\textit{Grand Place Brussels} & 18.40 & 20.28 & 16.78 & 69.31 & \phantom{1}8.04 & 10.09 & \phantom{1}8.07 & 44.45 \\
\textit{Notre Dame Front Facade} & 15.82 & 23.97 & 19.61 & 71.60 & \phantom{1}6.40 & 12.27 & \phantom{1}9.01 & 48.39 \\
\textit{Palace of Westminster} & 16.04 & 17.62 & 14.72 & 69.38 & \phantom{1}4.99 & \phantom{1}7.98 & \phantom{1}6.40 & 46.49 \\
\textit{Pantheon Exterior} & 22.65 & 25.71 & 21.31 & 61.33 & 10.65 & 14.09 & 10.94 & 38.12 \\
\textit{Reichstag} & 12.15 & 12.99 & 10.27 & 82.85 & \phantom{1}5.15 & \phantom{1}6.16 & \phantom{1}4.57 & 60.12 \\
\textit{Sacre Coeur} & 11.81 & 17.61 & 14.31 & 74.21 & \phantom{1}3.69 & \phantom{1}7.95 & \phantom{1}5.87 & 43.72 \\
\textit{St. Peter's Square} & 17.85 & 18.75 & 15.48 & 72.75 & \phantom{1}8.51 & \phantom{1}9.71 & \phantom{1}7.66 & 44.60 \\
\textit{Taj Mahal} & \phantom{1}9.95 & 11.20 & \phantom{1}8.56 & 82.74 & \phantom{1}3.77 & \phantom{1}4.90 & \phantom{1}3.63 & 64.27 \\
\textit{Temple Nara Japan} & 18.57 & 21.53 & 17.01 & 64.97 & \phantom{1}7.60 & 10.63 & \phantom{1}8.16 & 32.21 \\
\textit{Trevi Fountain} & 20.93 & 27.77 & 22.89 & 50.12 & \phantom{1}8.10 & 13.89 & 10.78 & 31.32 \\
\midrule
    \multicolumn{1}{c}{~} & \multicolumn{8}{c}{w/ \texttt{ENM}} \\
    \midrule
    Scene & \texttt{5p(E)} & \texttt{4p(M)} & \texttt{4p(M$\pm \delta$)} & \texttt{4p(A)} & \texttt{5p(E)} & \texttt{4p(M)} & \texttt{4p(M$\pm \delta$)} & \texttt{4p(A)}\\
    \midrule
    \textit{Brandenburg Gate} & 14.53 & 17.05 & 14.38 & 28.40 & \phantom{1}6.74 & \phantom{1}8.86 & \phantom{1}6.93 & 16.96 \\
\textit{Buckingham Palace} & 16.71 & 18.76 & 16.07 & 35.07 & \phantom{1}7.35 & \phantom{1}8.96 & \phantom{1}7.46 & 19.72 \\
\textit{Colosseum Exterior} & 16.21 & 21.46 & 17.70 & 34.90 & \phantom{1}6.35 & 10.71 & \phantom{1}8.04 & 21.98 \\
\textit{Grand Place Brussels} & 16.21 & 17.96 & 15.23 & 31.53 & \phantom{1}7.27 & \phantom{1}8.85 & \phantom{1}7.24 & 18.40 \\
\textit{Notre Dame Front Facade} & 15.05 & 22.08 & 18.46 & 34.49 & \phantom{1}6.30 & 11.65 & \phantom{1}8.71 & 22.08 \\
\textit{Palace of Westminster} & 13.83 & 14.95 & 12.74 & 29.06 & \phantom{1}4.68 & \phantom{1}7.13 & \phantom{1}5.84 & 16.41 \\
\textit{Pantheon Exterior} & 21.19 & 24.03 & 20.30 & 37.21 & 10.32 & 13.28 & 10.55 & 24.42 \\
\textit{Reichstag} & \phantom{1}9.32 & 10.03 & \phantom{1}8.36 & 17.69 & \phantom{1}4.40 & \phantom{1}5.06 & \phantom{1}3.97 & \phantom{1}9.39 \\
\textit{Sacre Coeur} & 10.71 & 15.19 & 12.64 & 24.92 & \phantom{1}3.61 & \phantom{1}7.23 & \phantom{1}5.49 & 13.41 \\
\textit{St. Peter's Square} & 16.01 & 16.94 & 14.35 & 28.60 & \phantom{1}7.75 & \phantom{1}8.63 & \phantom{1}7.03 & 15.68 \\
\textit{Taj Mahal} & \phantom{1}8.32 & \phantom{1}9.32 & \phantom{1}7.43 & 15.85 & \phantom{1}3.34 & \phantom{1}4.10 & \phantom{1}3.18 & \phantom{1}7.72 \\
\textit{Temple Nara Japan} & 16.24 & 19.38 & 15.74 & 29.97 & \phantom{1}6.83 & \phantom{1}9.19 & \phantom{1}7.32 & 16.65 \\
\textit{Trevi Fountain} & 20.22 & 26.42 & 22.09 & 39.67 & \phantom{1}8.02 & 13.43 & 10.53 & 26.62 \\
\bottomrule
    \end{tabular}}
    \caption{Accuracy of two-view solvers on PhotoTourism scenes.}
    \label{tab:4ptM_vs_5pt}
\end{table}

In the last set of experiments, we study the accuracy of the mean point correspondence.
We sample 100 four-tuples of point correspondences
consistent with the ground truth relative pose, \ie, inliers, for each image pair in scenes from the PhotoTourism dataset~\cite{IMC2020}. 
We use the first three correspondences to define the triangles in both images. 
Then, we establish correspondences between the mean of the triangle in one image and various points in the triangle in the second image. 
We express points in the second triangle via their barycentric coordinates and uniformly sample $19 \times 19$ barycentric coordinates $(a,b)\in [0,1]^2$, such that $a+b \leq 1$. 
Fig.~\ref{fig:mean_tests} shows the results 
for the rotation error and the percentage of inliers consistent with the pose obtained with the $\texttt{4p(M)}$ solver on the \textit{St.~Peter's Square} scene.  
The optimum of the metrics is reached around the mean point of the triangles (the mean values of 2D Gaussians fitted to the results
are very close to the mean point $(0.\bar{3}, 0.\bar{3})$ of the triangles).
A detailed description, the translation and the symmetric epipolar errors, and the results for more scenes are in SM. For all 
 tested
 scenes, we observed a similar behavior.

\begin{figure*}
    \centering
    \resizebox{1.0\linewidth}{!}{
\begin{tikzpicture} 

        \begin{axis}[%
        hide axis, xmin=0,xmax=0,ymin=0,ymax=0,
        legend style={draw=white!15!white, 
        line width = 1pt,
        legend  columns =9, %
        /tikz/every even column/.append style={column sep=0.5cm},
        }
        ]
        
        \addlegendimage{Seaborn1}        \addlegendentry{\sfhc~\cite{Hruby_cvpr2022}};
        \addlegendimage{Seaborn2}
        \addlegendentry{\sft};
        \addlegendimage{Seaborn3}
        \addlegendentry{\sfafRC};
        \addlegendimage{Seaborn4}
        \addlegendentry{\sftmRC}; 
        \addlegendimage{Seaborn5}
        \addlegendentry{\sftmdRC};
        \addlegendimage{black!30}
        \addlegendentry{w/o \texttt{ENM}};
        \addlegendimage{black!30,dash pattern=on 2pt off 1pt on 2pt off 1pt}
        \addlegendentry{w/ \texttt{ENM}};
        
        \end{axis}
    \end{tikzpicture}}

    \begin{subfigure}{0.24\linewidth}
    \includegraphics[width=\linewidth]{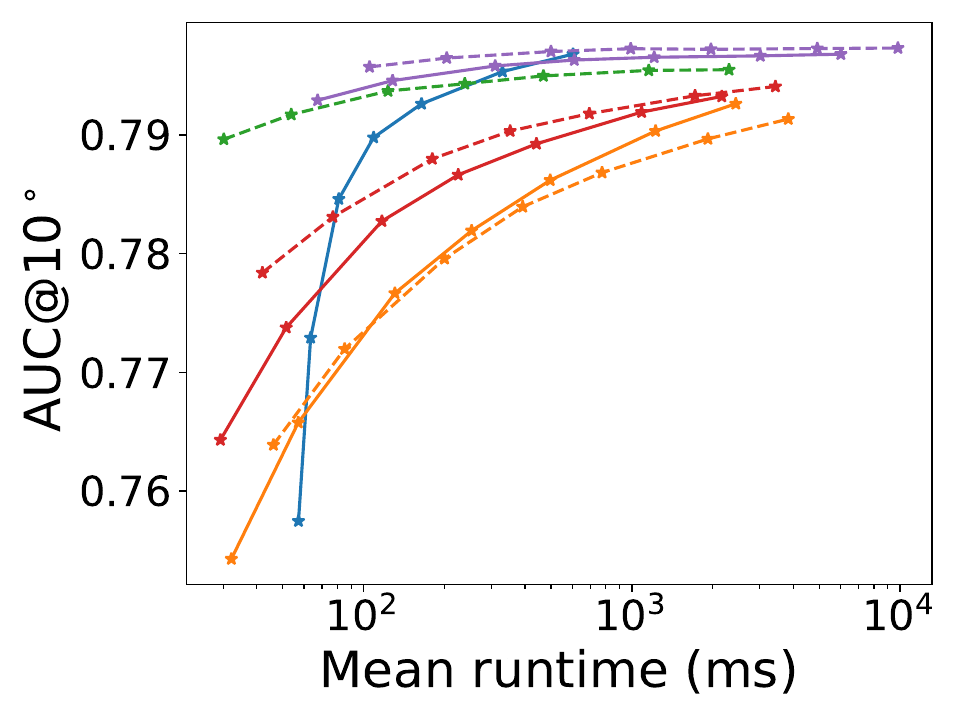}
    \caption{Phototourism~\cite{IMC2020}}
    \end{subfigure} \hfill
    \begin{subfigure}{0.24\linewidth}
    \includegraphics[width=\linewidth]{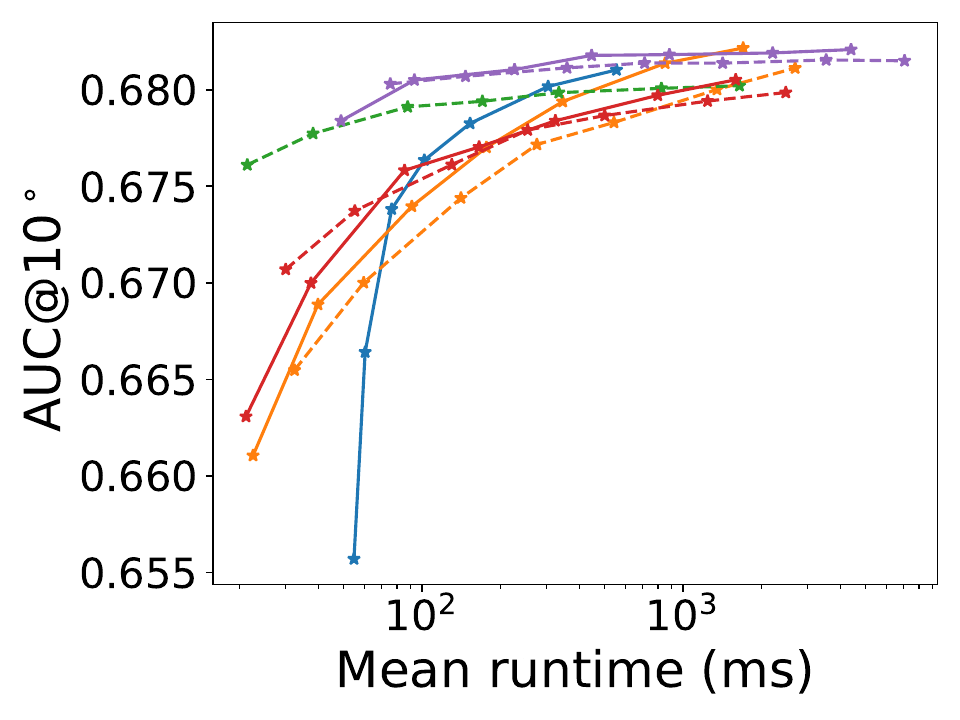}
    \caption{Cambridge Landmarks~\cite{kendall2015cambridge}}
    \end{subfigure} \hfill    
    \begin{subfigure}{0.24\linewidth}
    \includegraphics[width=\linewidth]{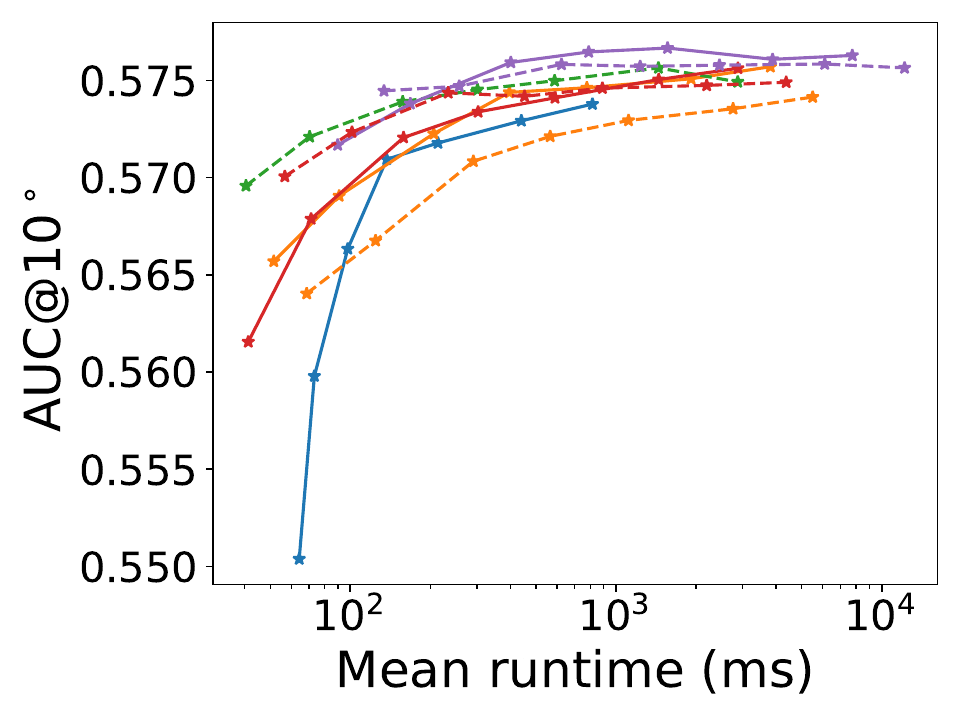}
    \caption{Aachen Day-Night v1.1~\cite{zhang2021aachen}}
    \end{subfigure} \hfill
    \begin{subfigure}{0.24\linewidth}
    \includegraphics[width=\linewidth]{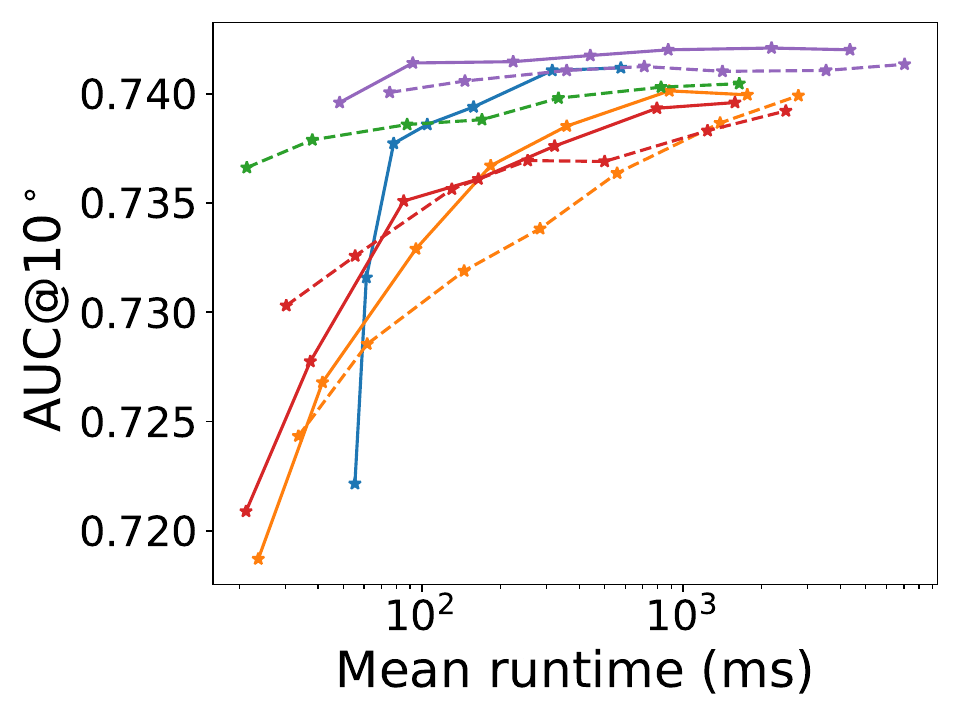}
    \caption{\textit{St. Mary's Church}~\cite{kendall2015cambridge}}
    \label{fig:poselib_graphs_st_mary}
    \end{subfigure}
       
    \caption{Speed-accuracy trade-off for (a) all scenes from PhotoTourism~\cite{IMC2020} except \textit{St.\ Peter's Square}, (b) 5 scenes from  Cambridge Landmarks~\cite{kendall2015cambridge}, (c) Aachen Day-Night v1.1~\cite{zhang2021aachen}, and (d) \textit{St.\ Mary's Church} scene from the Cambridge Landmarks dataset~\cite{kendall2015cambridge}. 
    We report the AUC@10$^\circ$ of the pose error and vary the number of PoseLib RANSAC iterations (100, 200, 500, 1000, 2000, 5000, 10,000) with a 5 px epipolar threshold. Runtimes are averaged over all image triplets.}%
    \label{fig:poselib_graphs}
\end{figure*}

\PAR{Experiments on real data.} 
We test the solvers on all scenes from the PhotoTourism dataset~\cite{snavely2006photo, IMC2020} which provide ground truth poses and intrinsics via a COLMAP~\cite{Schoenberger2016CVPR} reconstruction. 
In the results, we do not include the \textit{St.~Peter's Square} scene that we used for the validation of $\delta$ and the number of refinement iterations (see SM). We also include results for the Cambridge Landmarks dataset~\cite{kendall2015cambridge} (except the Street scene, which is commonly not used due to issues with its ground truth) and Aachen Day-Night v1.1~\cite{zhang2021aachen}. 
For PhotoTourism and Aachen, we use the images in their original resolution. For Cambridge Landmarks, we resize
them
so that the larger side is 800~px. 
For each scene, we sample 5,000 random image triplets with at least 10 matches obtained with~\cite{detone2018superpoint,lindenberger2023lightglue}
with at least %
$10\%$ overlap~\cite{IMC2020}.

\begin{table*}[t]
    \centering
    \resizebox{1.0\linewidth}{!}{

\begin{tabular}{ l | c c c c c c | c c  c c c  c | c c c c c c}
    \multicolumn{1}{c}{~} & \multicolumn{6}{c}{PhotoTourism~\cite{IMC2020}} & \multicolumn{6}{c}{Cambridge Landmarks~\cite{kendall2015cambridge}} & \multicolumn{6}{c}{Aachen Day-Night v1.1~\cite{zhang2021aachen}} \\
    \toprule
    Estimator & AVG $(^\circ)$ $\downarrow$ & MED $(^\circ)$ $\downarrow$ & AUC@5 $\uparrow$ & @10 $\uparrow$ & @20 $\uparrow$ & Runtime $\downarrow$& 
    AVG $(^\circ)$ $\downarrow$ & MED $(^\circ)$ $\downarrow$ & AUC@5 $\uparrow$ & @10 $\uparrow$ & @20 $\uparrow$ & Runtime $\downarrow$& 
    AVG $(^\circ)$ $\downarrow$ & MED $(^\circ)$ $\downarrow$ & AUC@5 $\uparrow$ & @10 $\uparrow$ & @20 $\uparrow$ & Runtime $\downarrow$\\
    \midrule

\sfhc~\cite{Hruby_cvpr2022} & \phantom{1}6.37 & \phantom{1}1.62 & 62.41 & 73.53 & 81.94 & \phantom{1}64.10 & \phantom{1}8.73 & \phantom{1}2.67 & 47.99 & 64.58 & 76.65 & 60.11 & 12.43 & \phantom{1}3.55 & 42.28 & 55.20 & 66.55 & \phantom{1}67.26 \\
\midrule\sft & \phantom{1}5.40 & \phantom{1}1.60 & 62.93 & 74.30 & 83.00 & \phantom{1}33.77 & \phantom{1}7.46 & \phantom{1}2.65 & 48.36 & 65.33 & 77.75 & 24.04 & 10.52 & \phantom{1}3.38 & 43.43 & 56.73 & 68.28 & \phantom{1}53.34 \\
\sftENM & \phantom{1}5.23 & \phantom{1}1.56 & 63.77 & 75.03 & 83.50 & \phantom{1}48.79 & \phantom{1}7.20 & \phantom{1}2.63 & 48.74 & 65.68 & 78.04 & 34.82 & 10.59 & \phantom{1}3.41 & 43.32 & 56.64 & 68.06 & \phantom{1}71.61 \\
\sftRC & \phantom{1}5.10 & \phantom{1}1.50 & 64.90 & 75.90 & 84.11 & \phantom{1}38.12 & \phantom{1}7.15 & \phantom{1}2.58 & 49.42 & 66.35 & 78.55 & 31.82 & 10.37 & \phantom{1}3.36 & 43.73 & 57.15 & 68.71 & \phantom{1}54.34 \\
\sftRCENM & \phantom{1}4.99 & \phantom{1}1.50 & 65.18 & 76.20 & 84.30 & \phantom{1}56.52 & \phantom{1}6.98 & \phantom{1}2.58 & 49.38 & 66.31 & 78.57 & 47.07 & \underline{10.20} & \phantom{1}3.34 & 43.75 & 57.11 & 68.63 & \phantom{1}76.86 \\
\midrule \sfaf & 37.48 & 25.16 & 22.91 & 28.59 & 35.16 & \phantom{1}16.58 & 41.23 & 21.73 & 22.75 & 31.47 & 38.98 & 13.33 & 36.40 & 19.53 & 24.41 & 31.73 & 39.24 & \phantom{1}32.04 \\
\sfafENM & \phantom{1}4.97 & \phantom{1}1.48 & 65.49 & 76.47 & 84.53 & \phantom{1}40.45 & \phantom{1}7.02 & \phantom{1}2.59 & 49.35 & 66.32 & 78.53 & 28.35 & 10.35 & \phantom{1}3.38 & 43.62 & 56.99 & 68.55 & \phantom{1}62.44 \\
\sfafR & 33.45 & 19.95 & 25.86 & 32.22 & 39.16 & \phantom{1}\underline{16.32} & 36.98 & 11.59 & 27.03 & 36.70 & 44.48 & \underline{12.48} & 34.84 & 18.22 & 25.27 & 32.77 & 40.29 & \phantom{1}\underline{29.56} \\
\sfafRC & 35.65 & 23.47 & 23.99 & 29.97 & 36.64 & \phantom{1}\textbf{11.10} & 38.34 & 14.47 & 25.89 & 35.19 & 42.80 & \phantom{1}\textbf{9.35} & 36.77 & 20.76 & 23.94 & 31.23 & 38.58 & \phantom{1}\textbf{19.75} \\
\sfafRCENM & \phantom{1}4.98 & \phantom{1}1.48 & 65.62 & 76.56 & 84.55 & \phantom{1}32.37 & \phantom{1}6.99 & \phantom{1}2.57 & 49.62 & 66.55 & 78.72 & 23.43 & 10.43 & \phantom{1}3.38 & 43.57 & 56.95 & 68.56 & \phantom{1}42.36 \\
\midrule\sftm & \phantom{1}6.02 & \phantom{1}1.71 & 60.86 & 72.56 & 81.73 & \phantom{1}35.18 & \phantom{1}8.36 & \phantom{1}2.74 & 47.15 & 64.06 & 76.59 & 25.15 & 12.09 & \phantom{1}3.62 & 41.78 & 55.06 & 66.67 & \phantom{1}54.89 \\
\sftmENM & \phantom{1}5.28 & \phantom{1}1.58 & 63.61 & 74.96 & 83.46 & \phantom{1}48.83 & \phantom{1}7.21 & \phantom{1}2.63 & 48.69 & 65.57 & 77.95 & 34.93 & 10.67 & \phantom{1}3.43 & 43.31 & 56.61 & 68.08 & \phantom{1}70.87 \\
\sftmR & \phantom{1}5.49 & \phantom{1}1.56 & 63.66 & 74.88 & 83.37 & \phantom{1}41.52 & \phantom{1}7.69 & \phantom{1}2.63 & 48.69 & 65.67 & 77.99 & 30.76 & 11.17 & \phantom{1}3.46 & 43.28 & 56.39 & 67.88 & \phantom{1}60.40 \\
\sftmRC & \phantom{1}5.48 & \phantom{1}1.54 & 63.99 & 75.14 & 83.51 & \phantom{1}30.88 & \phantom{1}7.75 & \phantom{1}2.62 & 48.77 & 65.71 & 77.98 & 22.72 & 11.16 & \phantom{1}3.48 & 43.04 & 56.28 & 67.73 & \phantom{1}42.38 \\
\sftmRCENM & \phantom{1}5.01 & \phantom{1}1.50 & 65.18 & 76.16 & 84.28 & \phantom{1}44.51 & \phantom{1}6.99 & \phantom{1}2.58 & 49.43 & 66.31 & 78.56 & 32.48 & \underline{10.25} & \phantom{1}3.39 & 43.74 & 57.21 & 68.80 & \phantom{1}58.08 \\
\midrule\sftmd & \phantom{1}5.55 & \phantom{1}1.68 & 61.86 & 73.73 & 82.78 & \phantom{1}83.70 & \phantom{1}7.65 & \phantom{1}2.67 & 48.16 & 65.25 & 77.76 & 59.23 & 11.13 & \phantom{1}3.47 & 42.74 & 56.11 & 67.74 & 125.02 \\
\sftmdENM & \phantom{1}4.99 & \phantom{1}1.56 & 64.02 & 75.51 & 84.01 & 125.66 & \phantom{1}\underline{6.88} & \phantom{1}2.60 & 49.11 & 66.13 & 78.51 & 89.05 & 10.26 & \phantom{1}3.35 & 43.66 & 56.99 & 68.55 & 175.54 \\
\sftmdR & \phantom{1}\underline{4.86} & \phantom{1}\underline{1.47} & \underline{65.97} & \underline{77.00} & \underline{85.01} & 100.61 & \phantom{1}7.10 & \phantom{1}\underline{2.56} & \underline{49.79} & \underline{66.89} & \underline{79.08} & 73.94 & 10.37 & \phantom{1}\underline{3.34} & \underline{43.93} & \underline{57.26} & \underline{68.82} & 138.53 \\
\sftmdRC & \phantom{1}4.92 & \phantom{1}1.47 & 65.85 & 76.87 & 84.90 & \phantom{1}71.73 & \phantom{1}7.19 & \phantom{1}2.56 & 49.71 & 66.75 & 78.95 & 52.84 & 10.52 & \phantom{1}3.36 & 43.79 & 57.17 & 68.73 & \phantom{1}92.89 \\
\sftmdRCENM & \phantom{1}\textbf{4.66} & \phantom{1}\textbf{1.46} & \textbf{66.12} & \textbf{77.14} & \textbf{85.15} & 112.60 & \phantom{1}\textbf{6.65} & \phantom{1}\textbf{2.55} & \textbf{49.87} & \textbf{66.95} & \textbf{79.20} & 81.98 & \phantom{1}\textbf{9.96} & \phantom{1}\textbf{3.32} & \textbf{44.02} & \textbf{57.42} & \textbf{68.97} & 139.19 \\
\midrule
\end{tabular}}
    \caption{Results for different solvers implemented in the PoseLib framework~\cite{PoseLib} on 12 scenes from PhotoTourism~\cite{IMC2020}, 5 scenes from Cambridge Landmarks~\cite{kendall2015cambridge} and Aachen Day-Night v1.1~\cite{zhang2021aachen}. We mark the \textbf{best} and \underline{second best} results. Runtimes are reported in ms for the whole RANSAC with early termination (0.9999 confidence, minimum 100 iterations) and epipolar threshold set to 5 px.
    }
    \label{tab:poselib_both}
\end{table*}

\begin{table*}[t]
    \centering
    \resizebox{1.0\linewidth}{!}{

\begin{tabular}{ l | c c c c c c | c c  c c c  c | c c c c c c}
    \multicolumn{1}{c}{~} & \multicolumn{6}{c}{PhotoTourism~\cite{IMC2020}} & \multicolumn{6}{c}{Cambridge Landmarks~\cite{kendall2015cambridge}} & \multicolumn{6}{c}{Aachen Day-Night v1.1~\cite{zhang2021aachen}} \\
    \toprule
    Estimator & AVG $(^\circ)$ $\downarrow$ & MED $(^\circ)$ $\downarrow$ & AUC@5 $\uparrow$ & @10 $\uparrow$ & @20 $\uparrow$ & Runtime $\downarrow$& 
    AVG $(^\circ)$ $\downarrow$ & MED $(^\circ)$ $\downarrow$ & AUC@5 $\uparrow$ & @10 $\uparrow$ & @20 $\uparrow$ & Runtime $\downarrow$& 
    AVG $(^\circ)$ $\downarrow$ & MED $(^\circ)$ $\downarrow$ & AUC@5 $\uparrow$ & @10 $\uparrow$ & @20 $\uparrow$ & Runtime $\downarrow$\\
    \midrule

\sfhc~\cite{Hruby_cvpr2022} & \phantom{1}6.37 & \phantom{1}1.62 & 62.41 & 73.53 & 81.94 & \phantom{1}64.10 &
\phantom{1}8.73 & \phantom{1}2.67 & 47.99 & 64.58 & 76.65 & 60.11 & 12.43 & \phantom{1}3.55 & 42.28 & 55.20 & 66.55 & \phantom{1}67.26 \\
\sft & \phantom{1}5.40 & \phantom{1}1.60 & 62.93 & 74.30 & 83.00 & \phantom{1}\cellcolor{third}{33.77} & \phantom{1}7.46 & \phantom{1}2.65 & 48.36 & 65.33 & 77.75 & \cellcolor{third}{24.04} & 10.52 & \phantom{1}\cellcolor{third}{3.38} & 43.43 & 56.73 & 68.28 & \phantom{1}\cellcolor{third}{53.34} \\
\sftRCENM & \phantom{1}\cellcolor{third}{4.99} & \phantom{1}\cellcolor{third}{1.50} & \cellcolor{third}{65.18} & \cellcolor{third}{76.20} & \cellcolor{third}{84.30} & \phantom{1}56.52 & \phantom{1}\cellcolor{second}{6.98} & \phantom{1}\cellcolor{third}{2.58} & 49.38 & \cellcolor{third}{66.31} & \cellcolor{third}{78.57} & 47.07 & \cellcolor{second}{10.20} & \phantom{1}\cellcolor{second}{3.34} & \cellcolor{second}{43.75} & \cellcolor{third}{57.11} & \cellcolor{third}{68.63} & \phantom{1}76.86 \\
\sfaf & 37.48 & 25.16 & 22.91 & 28.59 & 35.16 & \phantom{1}\cellcolor{best}{16.58} & 41.23 & 21.73 & 22.75 & 31.47 & 38.98 & \cellcolor{best}{13.33} & 36.40 & 19.53 & 24.41 & 31.73 & 39.24 & \phantom{1}\cellcolor{best}{32.04} \\
\sfafRCENM & \phantom{1}\cellcolor{second}{4.98} & \phantom{1}\cellcolor{second}{1.48} & \cellcolor{second}{65.62} & \cellcolor{second}{76.56} & \cellcolor{second}{84.55} & \phantom{1}\cellcolor{second}{32.37} & \phantom{1}\cellcolor{third}{6.99} & \phantom{1}\cellcolor{second}{2.57} & \cellcolor{second}{49.62} & \cellcolor{second}{66.55} & \cellcolor{second}{78.72} & \cellcolor{second}{23.43} & 10.43 & \phantom{1}\cellcolor{third}{3.38} & 43.57 & 56.95 & 68.56 & \phantom{1}\cellcolor{second}{42.36} \\
\sftm & \phantom{1}6.02 & \phantom{1}1.71 & 60.86 & 72.56 & 81.73 & \phantom{1}35.18 & \phantom{1}8.36 & \phantom{1}2.74 & 47.15 & 64.06 & 76.59 & 25.15 & 12.09 & \phantom{1}3.62 & 41.78 & 55.06 & 66.67 & \phantom{1}54.89 \\
\sftmRCENM & \phantom{1}5.01 & \phantom{1}\cellcolor{third}{1.50} & \cellcolor{third}{65.18} & 76.16 & 84.28 & \phantom{1}44.51 & \phantom{1}\cellcolor{third}{6.99} & \phantom{1}\cellcolor{third}{2.58} & \cellcolor{third}{49.43} & \cellcolor{third}{66.31} & 78.56 & 32.48 & \cellcolor{third}{10.25} & \phantom{1}3.39 & \cellcolor{third}{43.74} & \cellcolor{second}{57.21} & \cellcolor{second}{68.80} & \phantom{1}58.08 \\
\sftmd & \phantom{1}5.55 & \phantom{1}1.68 & 61.86 & 73.73 & 82.78 & \phantom{1}83.70 & \phantom{1}7.65 & \phantom{1}2.67 & 48.16 & 65.25 & 77.76 & 59.23 & 11.13 & \phantom{1}3.47 & 42.74 & 56.11 & 67.74 & 125.02 \\
\sftmdRCENM & \phantom{1}\cellcolor{best}{4.66} & \phantom{1}\cellcolor{best}{1.46} & \cellcolor{best}{66.12} & \cellcolor{best}{77.14} & \cellcolor{best}{85.15} & 112.60 & \phantom{1}\cellcolor{best}{6.65} & \phantom{1}\cellcolor{best}{2.55} & \cellcolor{best}{49.87} & \cellcolor{best}{66.95} & \cellcolor{best}{79.20} & 81.98 & \phantom{1}\cellcolor{best}{9.96} & \phantom{1}\cellcolor{best}{3.32} & \cellcolor{best}{44.02} & \cellcolor{best}{57.42} & \cellcolor{best}{68.97} & 139.19 \\
\midrule
\end{tabular}}

\caption{Results for different solvers implemented in the PoseLib framework~\cite{PoseLib} on 12 scenes from PhotoTourism~\cite{IMC2020}, 5 scenes from Cambridge Landmarks~\cite{kendall2015cambridge} and Aachen Day-Night v1.1~\cite{zhang2021aachen}. We mark the \textbf{best} and \underline{second best} results. Runtimes are reported in ms for the whole RANSAC with early termination (0.9999 confidence, minimum 100 iterations) and epipolar threshold set to 5 px.
    }
    \label{tab:poselib_both}
\end{table*}

Tab.~\ref{tab:poselib_both} shows the results 
for PoseLib RANSAC with early termination and 5px epipolar threshold. %
Results for GC-RANSAC are in SM. 
With the suggested modifications (see Sec.~\ref{sec:making_solvers_practical}) all 
proposed 
solvers outperform the state-of-the-art %
\sfhc solver~\cite{Hruby_cvpr2022} in terms of pose accuracy with many 
of the variants also achieving faster runtimes. 

Using filtering (\texttt{+F}) improves the run-time of RANSAC at the cost of a decrease in pose accuracy. 
Still, our \sftmRC solver outperforms \sfhc in terms of both accuracy and run-time. 
The \sftmd and \sftmdR solvers clearly improve upon the \sftm solvers, albeit at an increased run-time.
Still the \sftmdRC solver provides, in general, the best speed-accuracy trade-off.
The results also show that early non-minimal refinement (\texttt{+ENM}) is necessary for the \sfaf solver to perform well. 
For the other approximate solvers, this form of refinement also leads to improved accuracy at some cost to the runtime. 
{Assuming that the \sft solver samples five points from which at least four have projections in all three cameras, the suggested 
pose refinement (+R) and filtering (+F) can also be applied to the baseline \sft by using the fourth (potentially the fifth) point in the third view.
The suggested modifications improve the performance of the \sft solver, with the \sftRCENM solver having an accuracy similar to that of the \sftmRCENM solver, however at slower run-times, 
Moreover, it is less 
accurate than the proposed \sftmdRCENM solver.
}

We also investigate the speed-accuracy trade-off of the solvers by running PoseLib RANSAC~\cite{PoseLib} for a set of fixed numbers of iterations. 
Runtimes are reported for 1 core of a 2 GHz Intel Xeon Gold 6338 CPU. 
As shown in Fig.~\ref{fig:poselib_graphs}, %
the proposed \sftmdRC solver consistently provides the best speed-accuracy trade-off both with and without \texttt{ENM}. \sfafRCENM provides a similar performance, %
typically beating \sfhc~\cite{Hruby_cvpr2022}. However, as shown in Fig.~\ref{fig:poselib_graphs_st_mary} for the \textit{St.~Mary's Church} scene~\cite{kendall2015cambridge}, %
it may perform worse for some specific scenes.
%
%
%
%

%
{
The SM provides detailed ablation studies on the significance of each individual modification (\texttt{+R}/\texttt{+C}/\texttt{+ENM}).} 
%
%
%
%

%
%
%

%
%
%
%
%

%
%
%
%
%
%
%

%
%
%
%

%
\PAR{Limitations.} 
As discussed, 
the accuracy of the proposed approximate solvers is scene dependent.\footnote{This weakness also applies to~\cite{Hruby_cvpr2022}, since the scene needs to be similar enough to the training scenes for the MLP-based classifier to work well.} However, after using the propsed ENM refitting, the scene-dependency is negligible. 
Especially for the \sftmdENM solver, we have not noticed performance drops for some specific scene geometries or camera configurations. 
While the two-view variants of our solvers $\texttt{4p(M)}$/\texttt{4p(M$\pm \delta$)} provide very similar results to the $\texttt{5pt}$ solver in two views, in this case they do not outperform the $\texttt{5pt}$ solver. In the three-view scenario, the better performance is achieved thanks to the proposed pose refinement (+R) and filtering (+F).

The proposed ENM refitting cannot be applied to the \sfhc solver~\cite{Hruby_cvpr2022}, since this solver estimates the pose of all three cameras together. It is theoretically applicable to the solver proposed in~\cite{DBLP:journals/ijcv/NisterS06}. However, here ENM would have needed to be run on hundreds of candidate poses corresponding to tens-to-hundreds of sampled epipoles. This is because for one fixed point on the curve of epipoles,
the error of the sampled epipole is not bounded since the true epipole can lie anywhere on the curve, even outside the image. Thus, ENM might avoid local minima present in the second step of the original solver, but would significantly slow down the solver in the first step, making it much slower than our \sftmENM / \sftmdENM solvers.

\section{Conclusion}
We consider the highly challenging %
problem of relative pose estimation of three calibrated %
cameras from four correspondences. 
We propose  novel 
approaches that solve the problem by utilizing solvers based on approximate geometry. The best performing solver uses a simple, yet novel strategy by using mean coordinates of three input points and points in their vicinity as an  approximate $5^{th}$ correspondences.
Extensive experiments show that our solvers achieve state-of-the-art performance on a large variety of real scenes. 
At the same time, our solvers %
are simple to implement, especially compared to the current state-of-the-art~\cite{Hruby_cvpr2022}. 

\PAR{Acknowledgements.}
This work was funded by 
the Czech Science Foundation (GAČR) JUNIOR STAR Grant No.~22-23183M (supporting C.T., Y.D., and Z.K.), 
the Grant Agency of the Czech Technical University in Prague grant no.~SGS23/173/OHK3/3T/13 (supporting C.T.), 
the EU NextGenerationEU through the Recovery and Resilience Plan for Slovakia under the project ''InnovAIte Slovakia, Illuminating Pathways for AI-Driven Breakthroughs" No.~09I02-03-V01-00029, 
the TERAIS project -- a Horizon-Widera-2021 program of the European Union under the Grant agreement number 101079338 (supporting V.K. and Z.B.H.), 
and the EU Horizon 2020 project RICAIP grant agreement No.~857306 (supporting T.S.). 
Part of the research results was obtained using the computational resources procured in the national project National competence centre for high performance computing project code: 311070AKF2, funded by European Regional Development Fund, EU Structural Funds Informatization of society, Operational Program Integrated Infrastructure.

\section*{Supplementary Material}
This supplementary material provides additional details and experimental results promised in the main paper: 
Sec.~\ref{sec:estimating} provides the proof on the bound of the epipolar error of the mean point correspondence mentioned in Sec.~3.1 of the main paper, additional synthetic and noise experiments that were discussed in Sec.~4 of the main paper, accuracy-speed trade-off results for two-view approximate geometry inside RANSAC (see Sec.~4 of the main paper), and additional details and plots on more scenes for the mean point correspondence accuracy (see Sec.~4 of the main paper). 
Sec.~\ref{sec:exp:ablations} contains ablation studies to validate our choices regarding the modifications discussed in Sec.~3.2 of the main paper.
Sec.~\ref{sec:exp} provides results for the accuracy-speed trade-off for individual scenes from the PhotoTourism~\cite{IMC2020} and Cambridge Landmarks~\cite{kendall2015cambridge} datasets, evaluations of the three-view solvers for different thresholds inside RANSAC (see Sec.~4 of the main paper), runtimes of the proposed and state-of-the-art solvers for the 4p3v problem, semi-synthetic experiments for increasing outliers ratios on a scene from PhotoTourism~\cite{IMC2020}, details and results on an alternative evaluation measure (see Sec.~4 of the main paper), and results with GC-RANSAC (see Sec.~4 of the main paper).
\makeatletter
\def\blfootnote{\gdef\@thefnmark{}\@footnotetext}
\makeatother
\blfootnote{$^\ast$ Equal contribution}

\section{ Approximate camera geometry}
In this section, in addition to the experiments presented in Sec. 4  of the main paper (paragraph ``Approximate camera geometry"), we present additional experiments and results to support our idea of estimating approximate geometry in the first two views. We start 
with the proof on the bound of the epipolar error of the mean-point correspondence used in the proposed \sftm-based solvers (see Sec 3.1 of the main paper). 
Next, Sec.~\ref{sec:independentc} discusses why the mean-point correspondence provides an additional constraint (compared to the original point correspondences) that can be used to estimate the essential matrix. 
Then, to further assess the accuracy of the two-view variants of our approximate solvers (outside of RANSAC), \ie, \texttt{4p(A)}, \texttt{4p(M)},  and \texttt{4p(M$\pm \delta$)}, we design two additional synthetic experiments (see Sec.~\ref{sec:exp:synth}). 
The goal of these synthetic experiments is to study how the accuracy of approximate solutions varies with varying properties of the scene and the cameras.
Moreover, Sec.~\ref{sec:exp:synth} provides a synthetic noise experiment on 
data extracted from a real scene from the ETH3D dataset~\cite{Schops_2017_CVPR}, outside of RANSAC for the three-view solvers. Lastly, Sec.~\ref{sec:twoview} contains a speed-accuracy evaluation of the two-view variants of the proposed solvers, inside Poselib RANSAC, and in Sec.~\ref{sec:bary} we study the accuracy of the mean point correspondence (as in Fig. 3 of the main paper).

\label{sec:estimating}

\begin{figure}[t]
     \centering
     \includegraphics[width=1\columnwidth]{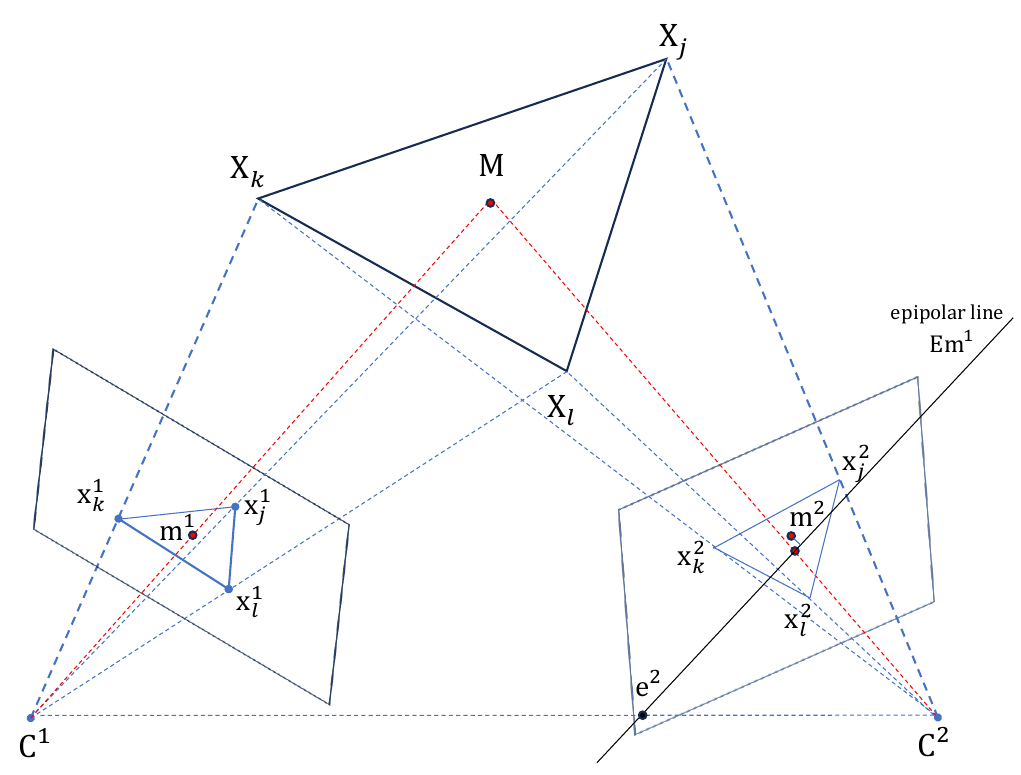}
          \caption{Illustration of the geometric configuration considered in the proof of Lemma~\ref{lemma:inter}.}
     \label{fig:illustration}
 \end{figure}

\subsection{Proof of the bounds on the epipolar error} 
\label{sec:proof}
While the mean point correspondence $\V m^1 \leftrightarrow \V m^2$  used in the \sftm-based solvers can provide a good approximation of a correct correspondence, such a correspondence can be noisy. 
Note that all state-of-the-art 4p3v solvers (including 
\sfhc~\cite{Hruby_cvpr2022} and 
the solver from~\cite{DBLP:journals/ijcv/NisterS06}) rely on certain approximations without establishing theoretical proofs to quantify their accuracy. 
In the \sfhc solver~\cite{Hruby_cvpr2022}, the failures that appear quite often are usually the results of tracking a geometrically incorrect solution inside the homotopy continuation method.\footnote{The solver is tracking only one from 272 solutions of the relaxed version of the 4p3v problem and this solution does not need to be a geometrically correct one.} Thus, this solution can be arbitrarily far from the correct solution. The solver from~\cite{DBLP:journals/ijcv/NisterS06}  requires sampling epipoles from a $10^{th}$-degree curve on which the true epipole must lie. For any selected  point on the curve of epipoles, the error of the sampled epipole is not bounded, since the true epipole can lie anywhere on the unbounded curve. The curve is unbounded since the epipole can be located arbitrarily far away from the image center based on the relative pose of the two cameras (up to infinity for sideways motion). 

\begin{figure*}[htbp]
    \centering

    \begin{tikzpicture} 
        \begin{axis}[%
        hide axis, xmin=0,xmax=0,ymin=0,ymax=0,
        legend style={draw=white!15!white, 
        line width = 1pt,
        legend columns =7, 
        /tikz/every even column/.append style={column sep=0.05cm}
        },
        legend image post style={xscale=1}
        ]
        \addlegendimage{Seaborn2}
        \addlegendentry{\texttt{5p(E)}~\cite{Nister-5pt-PAMI-2004}};
        \addlegendimage{Seaborn4}
        \addlegendentry{\texttt{4p(M)}};
        \addlegendimage{Seaborn5}
        \addlegendentry{\texttt{4p(M$\pm \delta$)}};
        \addlegendimage{Seaborn3}
        \addlegendentry{\texttt{4p(A)}};
        \addlegendimage{}
        \addlegendentry{};
        \addlegendimage{black!30,dash pattern=on 2pt off 1pt on 2pt off 1pt}
        \addlegendentry{w/o \texttt{ENM}};
        \addlegendimage{black!30,dash pattern=on 1pt off 0.5pt on 1pt off 0.5pt}
        \addlegendentry{w/ \texttt{ENM}};
        \end{axis}
    \end{tikzpicture}

    \begin{subfigure}[b]{0.49\textwidth}
        \centering
        \includegraphics[trim={1.5cm 0 1.5cm 2cm},clip, width=\textwidth]{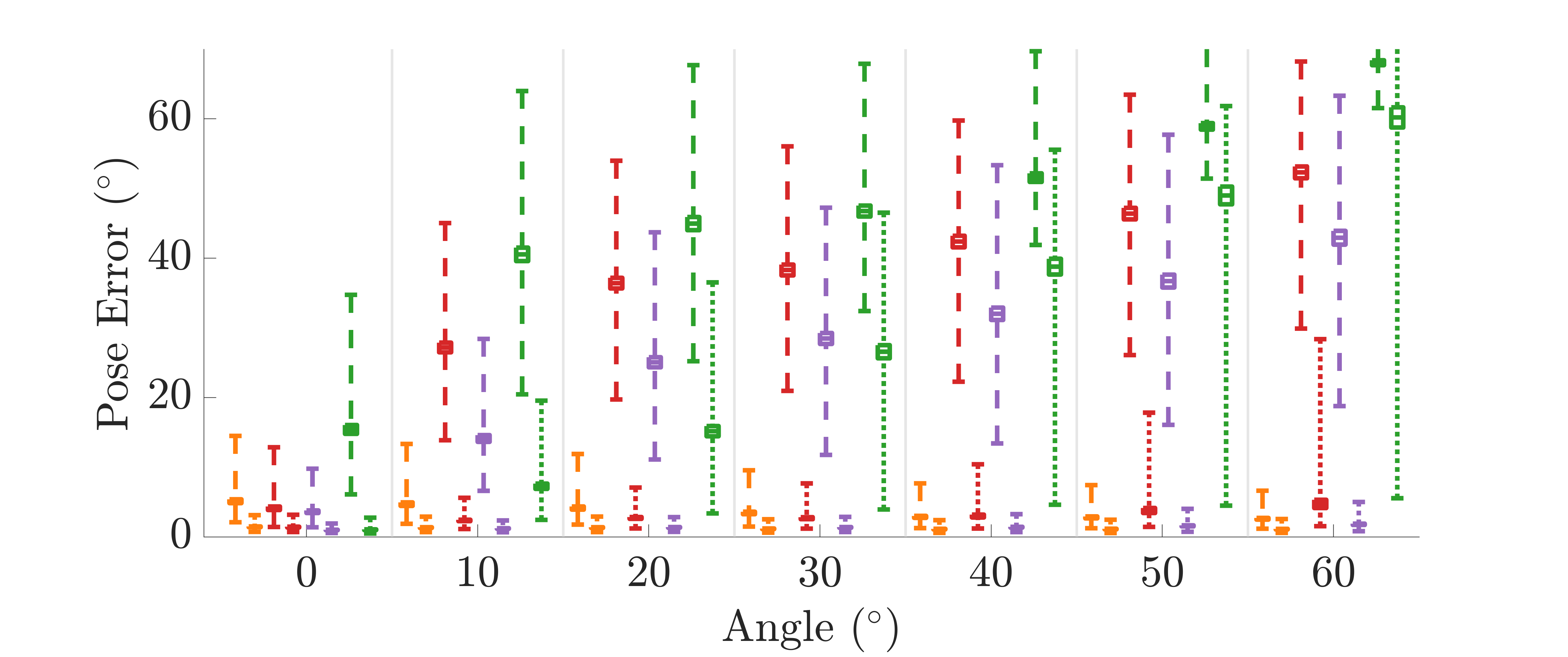}
        \caption{}
        \label{fig:subfig102}
    \end{subfigure}
    \hfill
    \begin{subfigure}[b]{0.49\textwidth}
        \centering
        \includegraphics[trim={1.5cm 0 1.5cm 2cm},clip, width=\textwidth]{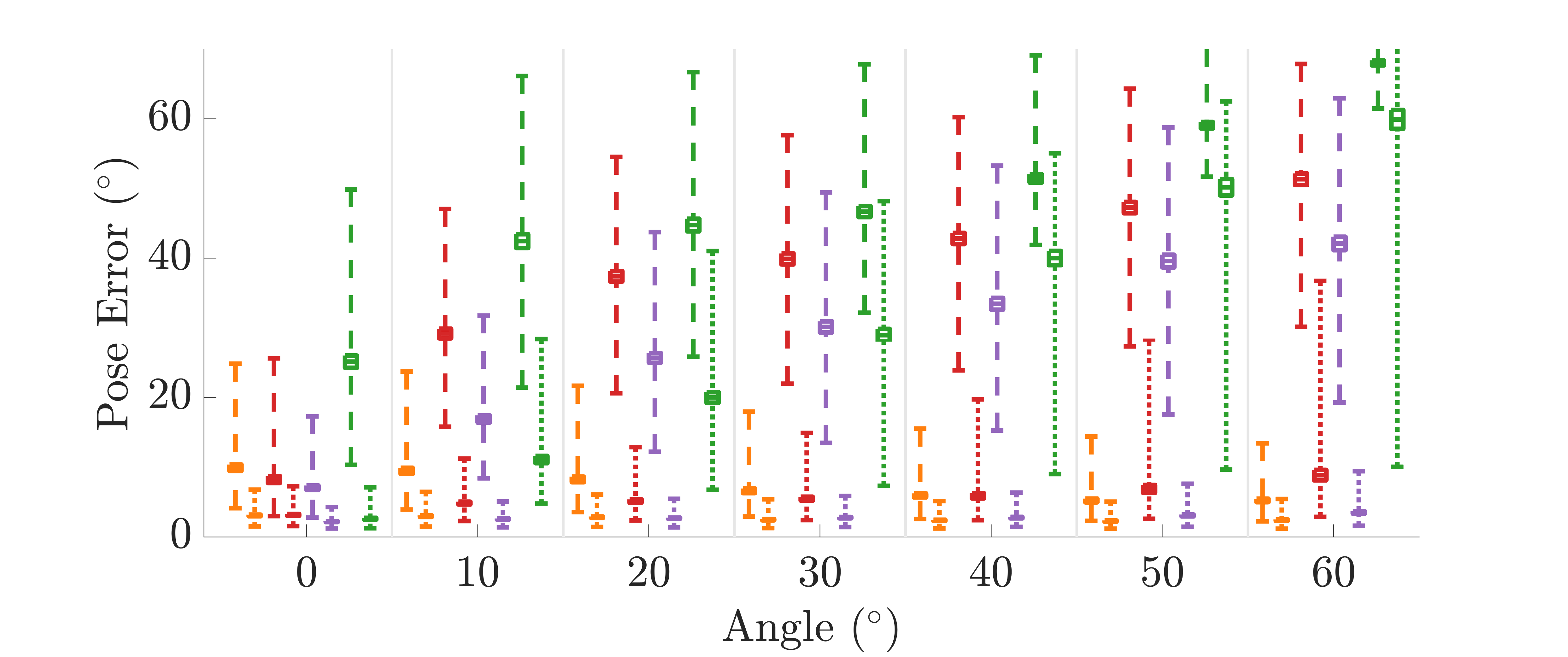}
        \caption{}
        \label{fig:subfig202}
    \end{subfigure}
    
    \begin{subfigure}[b]{0.49\textwidth}
        \centering
        \includegraphics[trim={1.5cm 0 1.5cm 2cm},clip, width=\textwidth]{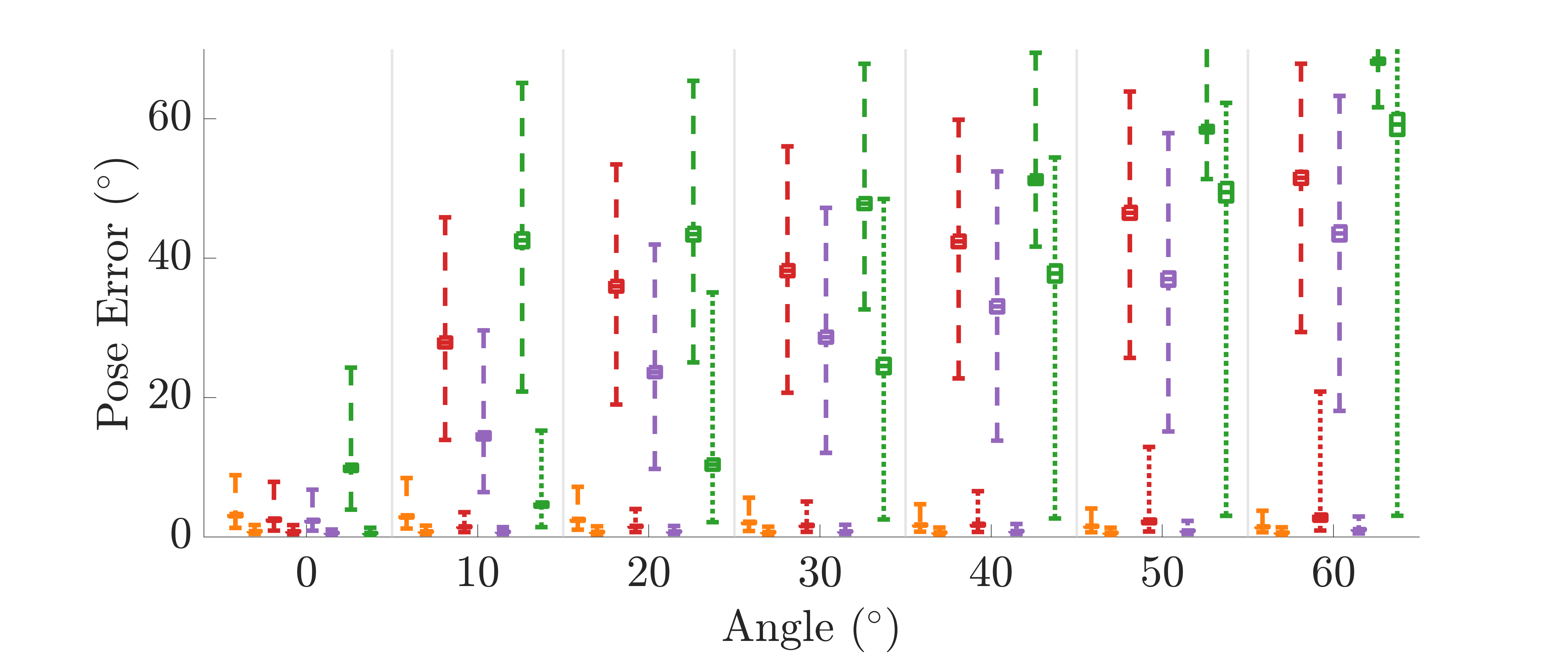}
        \caption{}
        \label{fig:subfig302}
    \end{subfigure}
    \hfill
    \begin{subfigure}[b]{0.49\textwidth}
        \centering
        \includegraphics[trim={1.5cm 0 1.5cm 2cm},clip, width=\textwidth]{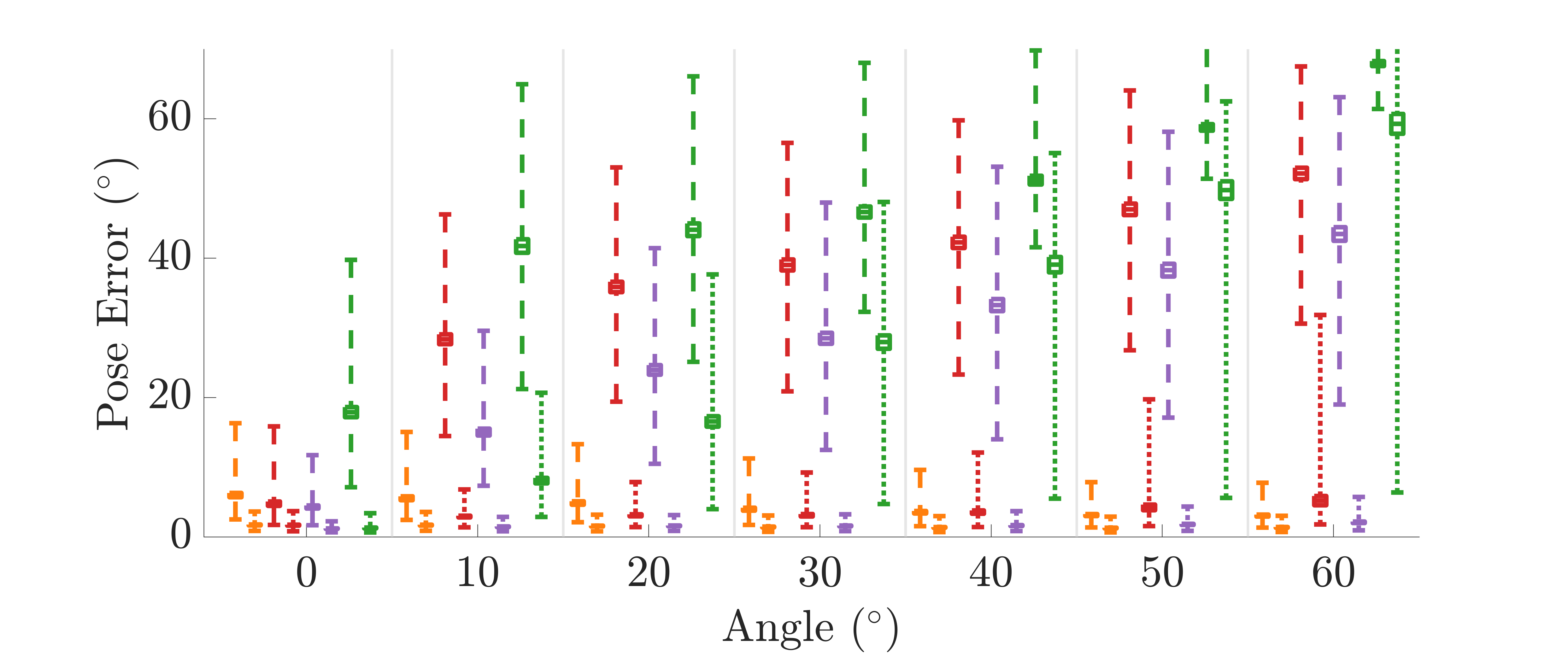} 
        \caption{}
        \label{fig:subfig402}
    \end{subfigure}
    \caption{{Results from a synthetic experiment evaluating the accuracy of two-view variants of our solvers as a function of the angle between the principal axes of the cameras are presented. The top row, comprising Subfigures (a) and (b), shows results for Gaussian noise with standard deviations of 2px and 4px, respectively. The bottom row, consisting of Subfigures (c) and (d), presents results for uniform noise with 2px and 4px deviations, respectively. The outlier ratio is set to 20\% in all cases. From the solutions for each solver and sample we select the one with the lowest error.}}
    \label{fig:synth_angle_4px}
\end{figure*}

\begin{figure*}[htbp]
    \centering

    \begin{tikzpicture} 
        \begin{axis}[%
        hide axis, xmin=0,xmax=0,ymin=0,ymax=0,
        legend style={draw=white!15!white, 
        line width = 1pt,
        legend columns =7, 
        /tikz/every even column/.append style={column sep=0.05cm}
        },
        legend image post style={xscale=1}
        ]
        \addlegendimage{Seaborn2}
        \addlegendentry{\texttt{5p(E)}~\cite{Nister-5pt-PAMI-2004}};
        \addlegendimage{Seaborn4}
        \addlegendentry{\texttt{4p(M)}};
        \addlegendimage{Seaborn5}
        \addlegendentry{\texttt{4p(M$\pm \delta$)}};
        \addlegendimage{Seaborn3}
        \addlegendentry{\texttt{4p(A)}};
        \addlegendimage{}
        \addlegendentry{};
        \addlegendimage{black!30,dash pattern=on 2pt off 1pt on 2pt off 1pt}
        \addlegendentry{w/o \texttt{ENM}};
        \addlegendimage{black!30,dash pattern=on 1pt off 0.5pt on 1pt off 0.5pt}
        \addlegendentry{w/ \texttt{ENM}};
        \end{axis}
    \end{tikzpicture}

    \begin{subfigure}[b]{0.49\textwidth}
        \centering
        \includegraphics[trim={1.5cm 0 1.5cm 2cm},clip, width=\textwidth]{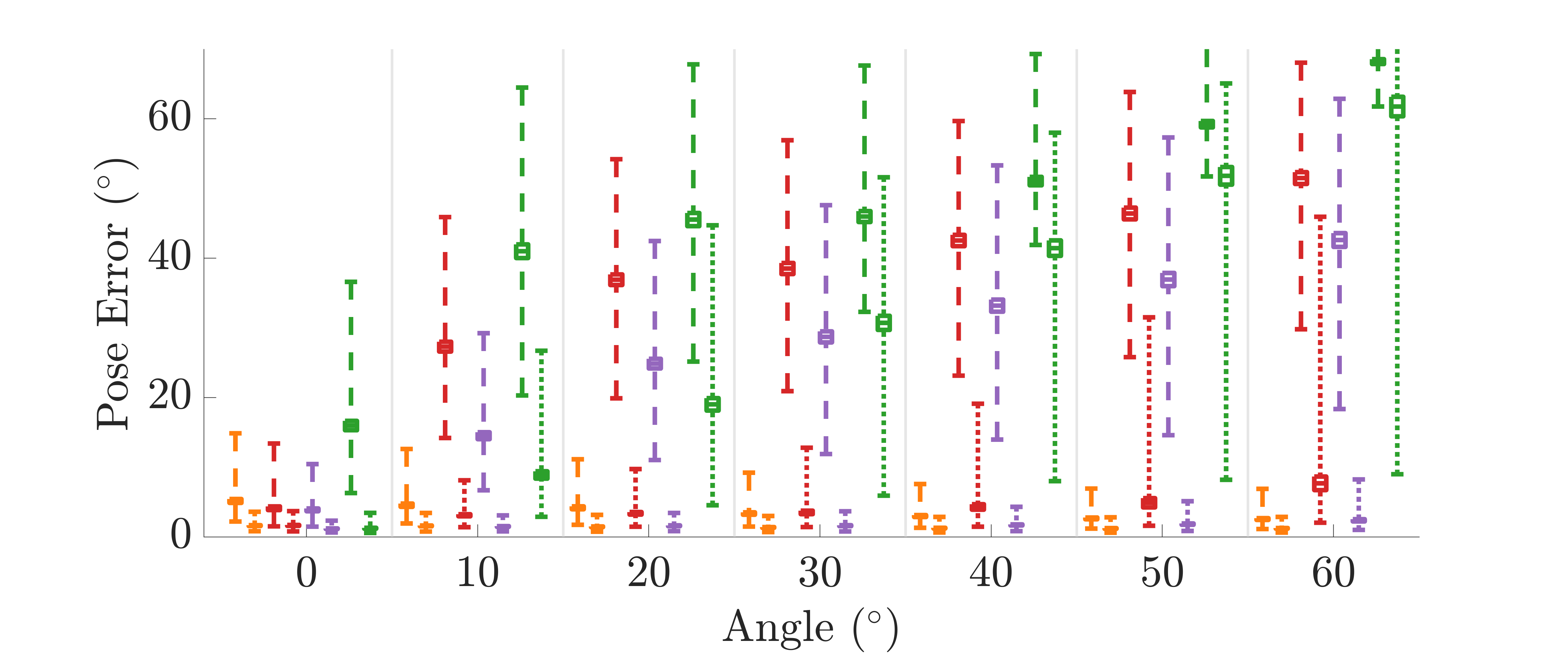}
        \caption{}
        \label{fig:subfig104}
    \end{subfigure}
    \hfill
    \begin{subfigure}[b]{0.49\textwidth}
        \centering
        \includegraphics[trim={1.5cm 0 1.5cm 2cm},clip, width=\textwidth]{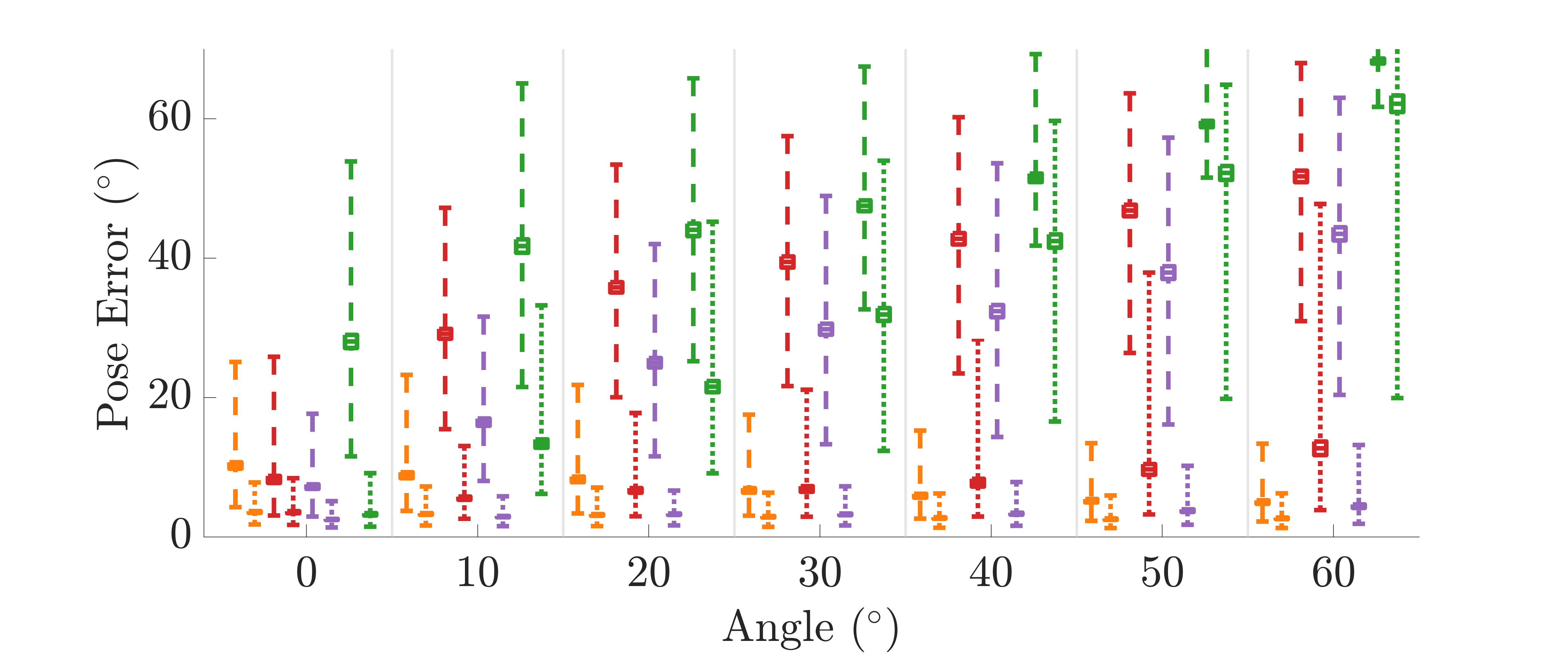}
        \caption{}
        \label{fig:subfig204}
    \end{subfigure}

    \caption{{Results from a synthetic experiment evaluating the accuracy of two-view variants of our solvers as a function of the angle between the principal axes of the cameras are presented. 
    We present results for Gaussian noise with standard deviations of 2px and 4px, respectively. The outlier ratio is set to 40\% in all cases. From the solutions for each solver and sample we select the one with the lowest error.}}
    \label{fig:synth_angle_04out}
\end{figure*}

\begin{figure*}[t!]
    \centering

\begin{tikzpicture} 

        \begin{axis}[%
        hide axis, xmin=0,xmax=0,ymin=0,ymax=0,
        legend style={draw=white!15!white, 
        line width = 1pt,
        legend  columns =7, 
        /tikz/every even column/.append style={column sep=0.2cm},
        }
        ]
        
        \addlegendimage{Seaborn2}
        \addlegendentry{\texttt{5p(E)}~\cite{Nister-5pt-PAMI-2004}};
        \addlegendimage{Seaborn4}        \addlegendentry{\texttt{4p(M)}};
        \addlegendimage{Seaborn5}        \addlegendentry{\texttt{4p(M$\pm\delta$)}};
        \addlegendimage{Seaborn3}
        \addlegendentry{\texttt{4p(A)}};
        \addlegendimage{}
        \addlegendentry{};
        \addlegendimage{black!30,dash pattern=on 2pt off 1pt on 2pt off 1pt}
        \addlegendentry{w/o \texttt{ENM}};
        \addlegendimage{black!30,dash pattern=on 1pt off 0.5pt on 1pt off 0.5pt}
        \addlegendentry{w/ \texttt{ENM}};
        \end{axis}
    \end{tikzpicture}
   
    \begin{subfigure}{0.49\textwidth}
    \includegraphics[trim={2.5cm 0 2.5cm 0},clip, width=\textwidth]{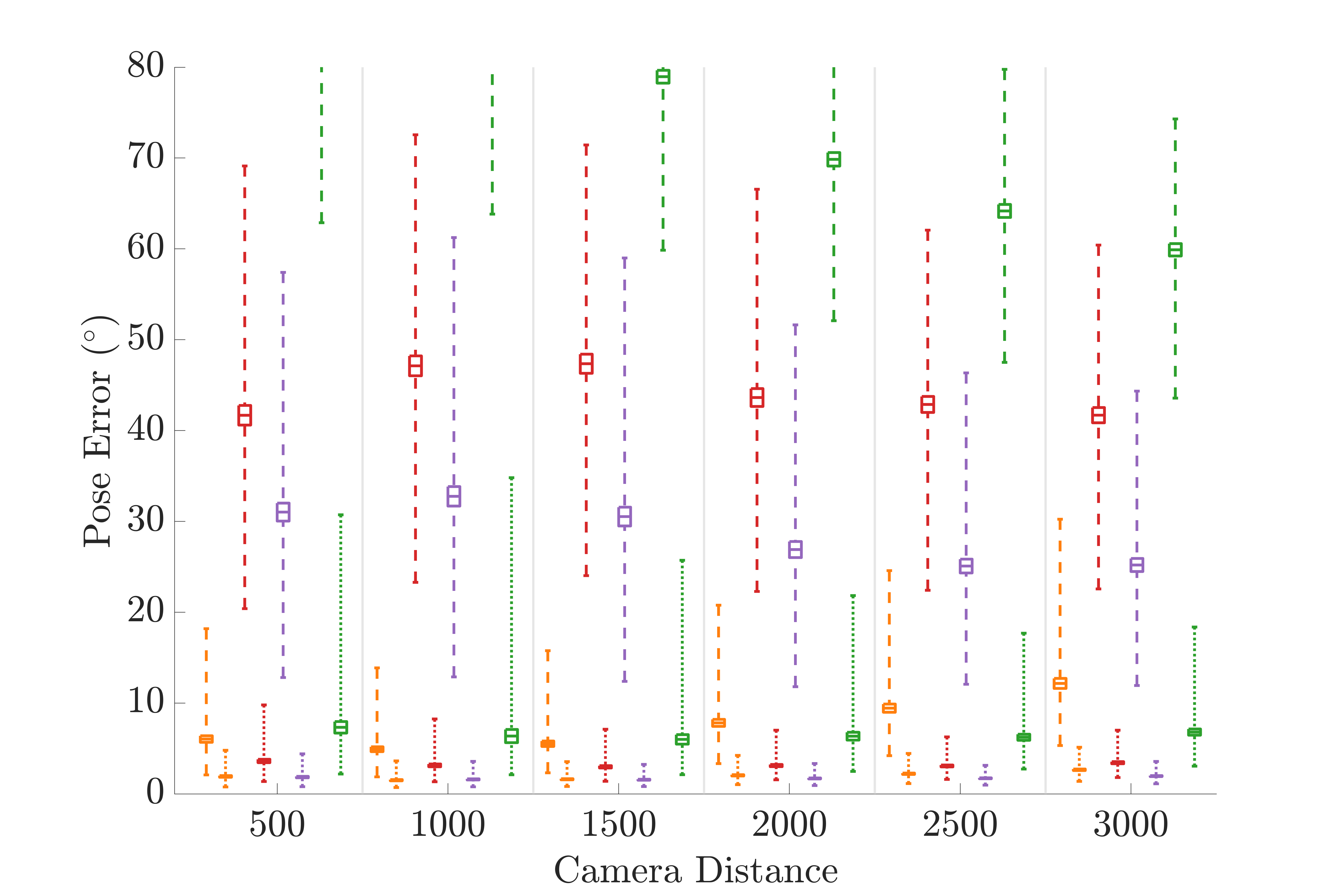}
    \caption{}
    \label{fig:synth_dist}
    \end{subfigure}
    \hfill    
    \begin{subfigure}{0.49\textwidth}
    \includegraphics[trim={2.5cm 0 2.5cm 0},clip, width=\textwidth]{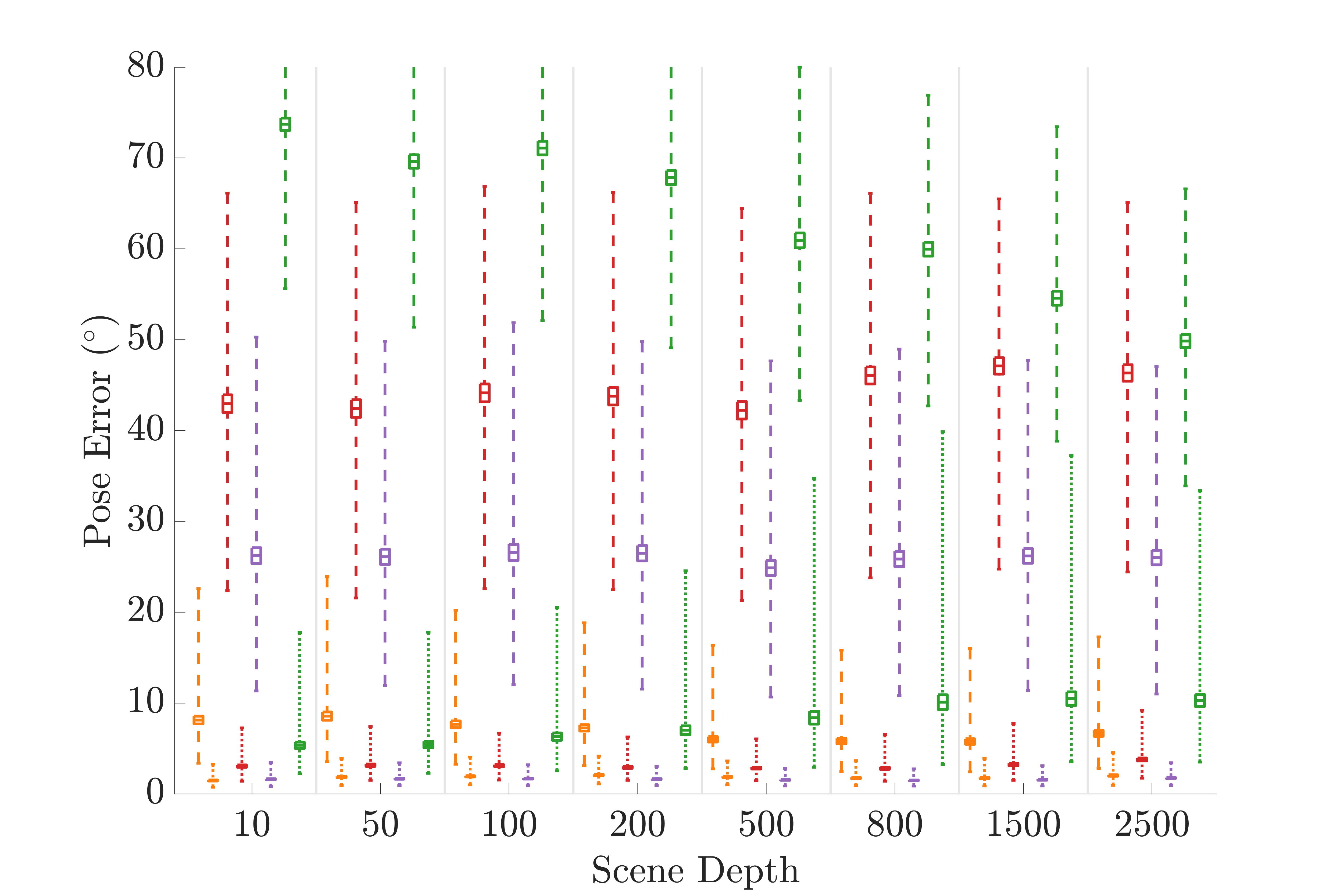}
    \caption{}
    \label{fig:synth_depth}
    \end{subfigure}  
    
    \caption{Synthetic experiments for two-view solvers, measuring the pose error under varying camera distance from the scene and scene depths, with added 1px noise. In \textbf{(a)}, points are uniformly sampled within a 2000x2000x100 cube, and projected to cameras positioned at distances (in units) from the scene as indicated by the x-axis, looking towards the scene. In \textbf{(b)}, the depth of the scene is varied, with points sampled inside a 2000x2000x\textit{depth} cube as specified by the x-axis. Cameras are randomly placed at distances between 1000 and 1200 units from the scene, looking towards the scene. On the y-axis of both figures is the pose error measured as $\max\left(\M R_{err}, \M t_{err} \right)$. From the solutions for each solver and sample we select the one with the lowest error. The results are displayed by boxplots which shows the 25\% to 75\% quantiles as a box with a horizontal line at the median.}
    \label{fig:synth_depth_distance}
\end{figure*}

In contrast, 
the error of our mean point correspondence is bounded. 
The $\V m^1 \leftrightarrow \V m^2$ correspondence can be seen as a correspondence of points that are projections of the mean point of three 3D points. 
In this case, the error of both projections $\V m^1$ and $\V m^2$ can be computed as an error that is introduced by approximating the perspective projection using the para-perspective projection. The approximation error introduced by the para-perspective projection is studied in the literature and can be found \eg. in~\cite{Zhang2014, Horaud97}.

 However, we can also look at the $\V m^1 \leftrightarrow \V m^2$ correspondence from a different point of view. We can consider this correspondence as a correspondence in which we fix a point in one view, \eg, $\V m^1$, and generate a corresponding point in the second view.
 In this case, as mentioned in the main paper, it can be proven that the epipolar error of the mean point correspondence $\V m^1 \leftrightarrow\V m^2$ is bounded by the maximum distance of the mean point $\V m^2$ from the vertices of the triangle $\left\{\V x_i^2,\V x_j^2,\V x_k^2\right\}$.
Here we provide a simple proof.

\begin{lemma}
Let us assume two cameras with camera centers $\V C^1$ and $\V C^2$ that observe 3D points $X_i, X_j,$ and $X_k$ (see Figure~\ref{fig:illustration} for an illustration). Let $\left\{\V x_i^1,\V x_j^1,\V x_k^1\right\}$ and $\left\{\V x_i^2,\V x_j^2,\V x_k^2\right\}$ be the projections of these 3D points in camera 1 and camera 2, respectively. Let $\V m^1$ be the mean point of the points $\left\{\V x_i^1,\V x_j^1,\V x_k^1\right\}$ and let $\V E$  be the essential matrix between these two cameras, \ie, a matrix that satisfies $\V x_l^{2^\top} \V E \V x_l^1 = 0,\; l \in \left\{i,j,k\right\}$. Then the epipolar line $\V E \V m^1$ passes through the triangle $\left\{\V x_i^2,\V x_j^2,\V x_k^2\right\}$.
\label{lemma:inter}
\end{lemma}

\begin{proof}
The camera center $\V C^1$ and the 3D points $\V X_i, \V X_j,$ and $\V X_k$ form a tetrahedron $T^1$ (see Figure~\ref{fig:illustration}). The projections $\left\{\V x_i^1,\V x_j^1,\V x_k^1\right\}$ in the first camera lie at the edges of this tetrahedron $T^1$.
The ray from the camera center $\V C^1$ through the mean point $\V m^1$ thus lies inside the tetrahedron $T^1$ and intersects the plane defined by 3D points $\V X_i, \V X_j,$ and $\V X_k$ in a point $\V M$ that lies inside the triangle defined by $\left\{\V X_i,\V X_j,\V X_k\right\}$. 

The camera center $\V C^2$ and the 3D points $\V X_i, \V X_j,$ and $\V X_k$ form a tetrahedron $T^2$.
Again, the projections $\left\{\V x_i^2,\V x_j^2,\V x_k^2\right\}$ 
lie at the edges of the tetrahedron $T^2$.
The ray passing through the camera center $\V C^2$ and the 3D point $\V M$ lies inside the tetrahedron $T^2$ and thus intersects the image plane of the second camera at a point that lies inside the triangle defined by the points $\left\{\V x_i^2,\V x_j^2,\V x_k^2\right\}$. 
By construction, the projection of $\V M$ into the second camera lies on the epipolar line $\V E \V m^1$. 
Therefore, the epipolar line $\V E \V m^1$ which is a line connecting this point and the epipole $\V e^2$, passes through the triangle $\left\{\V x_i^2,\V x_j^2,\V x_k^2\right\}$. 
\end{proof}

It follows from Lemma~\ref{lemma:inter} that since the epipolar line $\V E \V m^1$ passes through the triangle $\left\{\V x_i^2,\V x_j^2,\V x_k^2\right\}$, the maximum distance of the mean point $\V m^2$ to the epipolar line $\V E \V m^1$ is equal to the maximum distance of $\V m^2$ to the vertices of the triangle.

\subsection{Mean-point constraint}
\label{sec:independentc}
As already mentioned in the main paper, under the assumption of a para-perspective projection, \ie, of affine geometry, the mean point of three 3D points is projected to the mean points of the 3D points' projections in both images~\cite{Zhang2014}.
Thus,
the mean point correspondence $\V m^1 \leftrightarrow \V m^2$ does not add a new constraint if used to estimate an affine camera.  This can be easily shown. In the case of affine cameras, the essential matrix $\M E_A$ has the form
\begin{equation}
 \M E_A  =  \begin{bmatrix}
0 & 0 & a\\
0 & 0 & b \\
c & d & f
\end{bmatrix} \enspace .
\end{equation}
Thus the epipolar constraint for the mean point correspondence $\V m^1 \leftrightarrow \V m^2$ with the homogeneous coordinates $\V m^1 = (\V x_i^1/3 + \V x_j^1/3 + \V x_k^1/3$), and $\V m^2 = (\V x_i^2/3 + \V x_j^2/3 + \V x_k^2/3$):
\begin{equation}
    (\V m^{2})^{\top}  \mathtt{E}_A \V m^{1} = 0 
\end{equation}
can be written as a linear combination of the epipolar constraint for the three input points, \ie  $\V (x_l^{2})^{\top}  \mathtt{E}_A \V x_l^{1} = 0,\, l \in \left\{i,j,k\right\}$.

This is not the case for perspective cameras. For perspective cameras, \ie, when estimating the full essential 
matrix $\M E$, the mean point correspondence introduces an additional constraint. In this case the epipolar constraint  
\begin{equation}
    (\V m^{2})^{\top}  \mathtt{E} \V m^{1} = 0 \enspace, 
    \label{eq:epipolarE}
\end{equation}
 after expansion, contains terms $x_a^1x_b^2,  a \neq b, \, a,b \in \left\{i,j,k\right\}$. 
Thus, the epipolar constraint~\eqref{eq:epipolarE} is not a linear combination of the individual epipolar constraints $ (\V x_l^{2})^{\top}  \mathtt{E} \V x_l^{1} = 0, \, l =i,j,k$.
Therefore, the mean point correspondence provides an independent constraint when used to estimate the epipolar geometry of perspective cameras.

\subsection{Synthetic Experiments}
\label{sec:exp:synth}
The
error of the relative poses estimated with the proposed approximate 
\sftm-based and \sfaf-based solvers
depends on many aspects, \eg, the baseline and the view angles of the cameras \wrt the three points used to compute the mean point correspondence, the depth of these points, the size and shape of the triangles defined by the three points, the type of motion of the cameras, the depth of the scene and the distance of cameras from the scene, the level of noise in the correspondences, \etc 
Isolating the impact of the individual aspects, \eg, through experiments on synthetic data, is highly non-trivial (\eg, how to generate realistic synthetic scenarios that allow conclusions to generalize to real-world scenarios) and analysing the co-dependencies between different aspects on the overall performance seems to need a paper on its own. Moreover, the effect of approximation introduced by using para-perspective projection was already studied in the literature~\cite{Zhang2014,Horaud97}.

In the main paper, we thus presented mostly results on real-world scenes, without trying to isolate individual factors (see Figure~3 and Table~1 in the main paper).  However, we also tested interesting camera and scene setups using synthetically generated data.

We extend the synthetic experiment for increasing angles between the projection rays of the cameras in the main paper (Fig.~2 of the main paper), by investigating the impact of increased noise, alternative noise models, and higher outlier ratio. Similar to the experiment of increasing angles, we evaluate the two-view solvers (outside of RANSAC) in two additional interesting scenarios. 
The goal is to study the effect of the proposed approximations on the relative pose estimation under varying properties of the scene and the cameras. 

Additionally, we test the performance (outside of RANSAC) of our proposed approximate solvers and the state-of-the-art solvers for the 4p3v problem \wrt increasing image noise added to ground-truth correspondences extracted from a scene from the ETH3D dataset~\cite{Schops_2017_CVPR}.
As an evaluation metric for the two-view geometry, we use the pose error measured as $\max\left(\M R_{err}, \M t_{err} \right)$~\cite{IMC2020}.

\PAR{Increasing angle between principal axes of cameras.}
In Fig.~\ref{fig:synth_angle_4px}, we present results analogous to Fig.~2 of the main paper, but for $\sigma=2$px and $\sigma=4$px noise levels. We also include results for the uniform noise model (noise is evenly distributed in range $[-\sigma, \sigma]$). As in Fig.~2 of the main paper, the synthetic data contain 20\% outliers. As in the results presented in Fig.~2 of the main paper, the accuracy of both approximate solvers decreases as the angle increases, where $\texttt{4p(A)}$ demonstrates notably lower accuracy than $\texttt{4p(M)}$ and \texttt{4p(M$\pm \delta$)}. However, ENM singnificantly improves the accuracy, with $\texttt{4p(M$\pm \delta$)+ENM}$ achieving the same or slightly better accuracy than $\texttt{5pt+ENM}$ for angle $\leq 30^\circ$. In Fig.~\ref{fig:synth_angle_04out} we present results for Gaussian noise and higher outlier ratio, in particular 40\% outliers. The results are consistent with those of Fig.~\ref{fig:synth_angle_4px}, though the increased outlier ratio leads to a higher error when using ENM.

\begin{figure*}[t!]
    \centering
    \resizebox{1.0\linewidth}{!}{
\begin{tikzpicture} 

        \begin{axis}[%
        hide axis, xmin=0,xmax=0,ymin=0,ymax=0,
        legend style={draw=white!15!white, 
        line width = 1pt,
        legend  columns =9, 
        /tikz/every even column/.append style={column sep=0.5cm},
        }
        ]
        
        \addlegendimage{Seaborn1}        \addlegendentry{\sfhc~\cite{Hruby_cvpr2022}};
        \addlegendimage{Seaborn2}
        \addlegendentry{\sft};
        \addlegendimage{Seaborn4}
        \addlegendentry{\sftm}; 
        \addlegendimage{Seaborn5}
        \addlegendentry{\sftmd};
        \addlegendimage{Seaborn3}
        \addlegendentry{\sfaf};
        
        \addlegendimage{black!30,dash pattern=on 2pt off 1pt on 2pt off 1pt}
        \addlegendentry{\texttt{+R} w/o \texttt{ENM}};
        \addlegendimage{black!30,dash pattern=on 1pt off 0.5pt on 1pt off 0.5pt}
        \addlegendentry{\texttt{+R} w/ \texttt{ENM}};
        
        \end{axis}
    \end{tikzpicture}}
    \includegraphics[width=0.85\textwidth]{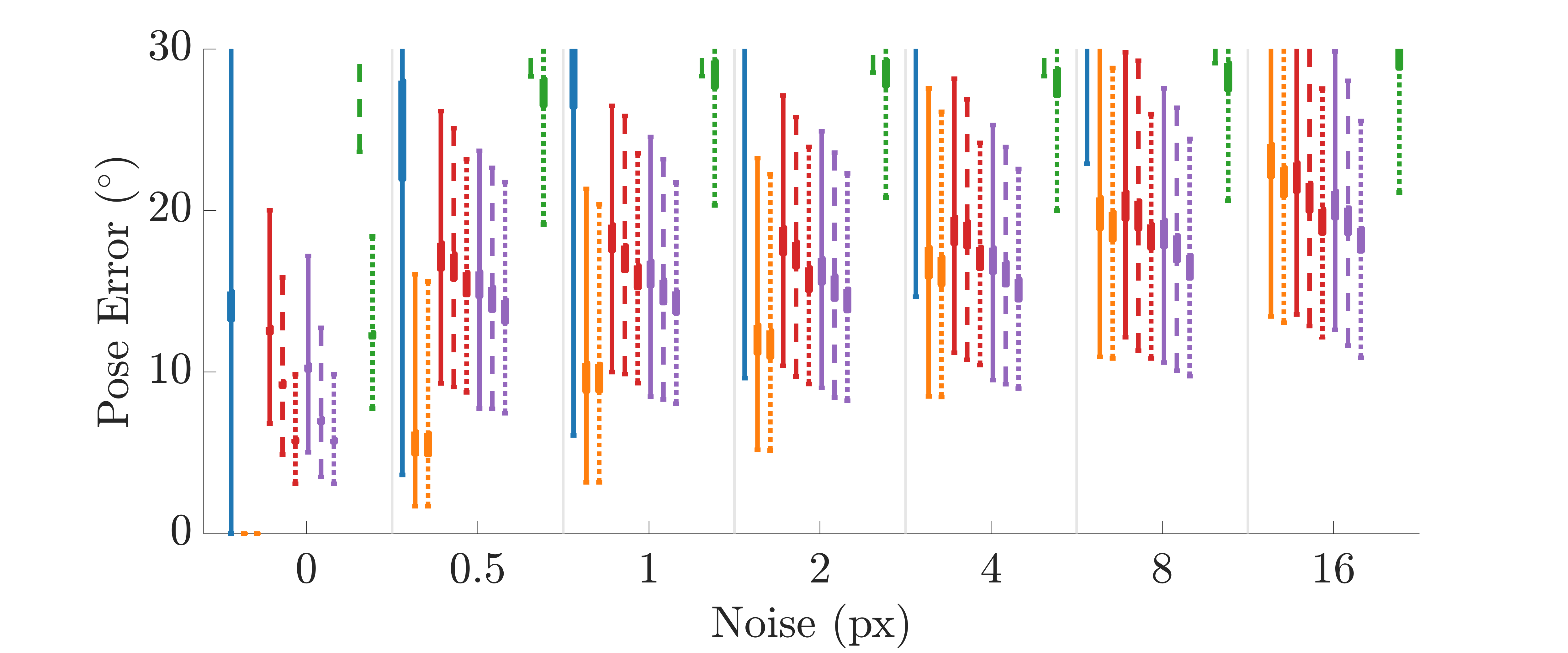}
    \caption{{Noise experiment showing 
    the pose error measured as $\text{max}\left(0.5 (\M R_{err}^{12} + \M R_{err}^{13}), 0.5 (\V t_{err}^{12} + \V t_{err}^{13})\right)$, as a function of the noise scale in pixels. Here, $\M R_{ij}$ and $\M t_{ij}$ are the relative rotation and translation of the $i^{\text{th}}$ and $j^{\text{th}}$ views, respectively. From the solutions for each solver and sample we select the one with the lowest error. Note that the errors observed for the pure \texttt{4p(A)} solver (without \texttt{ENM} and without refinement) are outside of the error range shown in this plot.}}
    \label{fig:noise}
\end{figure*}

\PAR{Increasing distance of cameras to the 3D scene.}
It is known that the quality of the affine approximation, \ie, the approximation of the perspective projection using the para-perspective projection, depends on the distance of the points from the camera~\cite{Zhang2014}. 
Thus in the first experiment, we evaluate the performance of all solvers w.r.t. increasing distance of the cameras to the 3D scene.

We perform this experiment on 10k synthetically generated instances. 
For each of the 10k instances, we uniformly sample 3D points inside a 2000x2000x100 unit cube, and the camera centers are placed at random points with the same distance from the scene. 
The distances tested (in units) are $\{ 500, 1000, 1500, 2000, 2500, 3000\}$. The cameras are generated such that they look towards the scene. We add 1px noise to the projected points. Note that the scene is generated such that the projections of the points cover a large portion of the image for cameras at all distances.

Fig.~\ref{fig:synth_dist} shows the results of this experiment represented
by the boxplot function, which shows values between the $25\%$ and $75\%$ quantiles as a box with a horizontal line at the median.
As expected, the errors of the \texttt{4p(A)} solver decrease as the distance of the cameras from the scene increases, since affine geometry can be better satisfied with larger distances from the scene. Without considering \texttt{ENM}, \texttt{5p(E)} is the best performing solver. 
The errors of the \texttt{5p(E)} solver are increasing with increasing distance of the cameras from the scene (due to the fact that fixed image noise is generating larger errors for points that are farther from cameras).  This effect is less visible for the proposed \texttt{4p(M)} and \texttt{4p(M$\pm \delta$)} solvers
since for these solvers the error is originally more dominated by the error in the mean point correspondence.
When considering \texttt{ENM}, \texttt{5p(E)}, \texttt{4p(M)}, and \texttt{4p(M$\pm \delta$)} solvers perform similarly, with \texttt{4p(M$\pm \delta$)} being the most accurate for distances $\geq 1500$. The \texttt{4p(A)} solver is also greatly improved when using \texttt{ENM}, reaching similar or even better accuracy than \texttt{5p(E)} (w/o \texttt{ENM}), for distances $\geq 1000$.

\PAR{Increasing depth of the 3D scene.}
In Fig.~\ref{fig:synth_depth}, instead of increasing the distance of the cameras to the 3D scene, we place the cameras randomly at distances between 1000 and 1200 units away from the scene, looking towards the scene, while changing the depth of the scene. In particular, the 3D points are generated uniformly at random inside a 2000x2000x\textit{depth} unit cube, where the depth of the scene is specified by the values on the x-axis. The tested depths are $\{ 10, 50, 100, 200, 500, 800, 1500, 2500\}$. We add 1px noise to the projected points. Without using \texttt{ENM}, the pose errors of \texttt{4p(A)} are visibly decreasing as the depth of the scenes increases. 
This is to be expected since increasing the scene depth increases the chance of sampling four points that are more consistent with the para-perspective / affine camera model (since the points are more likely to be farther away from the cameras). 
The remaining solvers, that is, \texttt{5p(E)}, \texttt{4p(M)}, and \texttt{4p(M$\pm \delta$)}, are not significantly affected by the changes in the depth of the scene. When using \texttt{ENM}, all tested solvers are improved significantly, with \texttt{4p(M$\pm \delta$)} being the most accurate in terms of pose error. 
When using \texttt{ENM}, the errors in the estimated poses increase with increasing scene depths, which is particularly visible for the \texttt{4p(A)} solver. 
This behavior is due to the fact that for points farther away from the camera, the same amount of image noise (1px) has a larger impact, thus leading to the non-minimal samples being more affected by noise. 
Still, using \texttt{ENM} clearly leads to significantly smaller errors for all solvers. 

\PAR{Noise experiments.}
\label{sec:exp:noise}
We test the performance of our solvers and the state-of-the-art solvers \wrt increasing image noise. 
We used the SfM model of the botanical garden scene (randomly selected from all scenes) from the ETH3D dataset~\cite{Schops_2017_CVPR} to obtain instances of 5 points in three views by identifying images in the scene that share 3D points. 
Perfect noise-free image correspondences are generated by projecting the 3D points into the images. 
We then add increasing amounts of normally distributed noise to these correspondences. 
We generated more than 1k instances. 
Note that the \sfhc solver was trained on the ETH3D dataset~\cite{Schops_2017_CVPR}.

The results for increasing noise in the image points are shown in Fig.~\ref{fig:noise}. 
The figure shows boxplots of %
the pose errors measured in the same way as in our experiments in the main paper (\cf Sec.~4 in the main paper), \ie, as
$\text{max}\left(0.5 (\M R_{err}^{12} + \M R_{err}^{13}), 0.5 (\V t_{err}^{12} + \V t_{err}^{13})\right)$.\footnote{Here $\M R_{err}^{ij}$  is the error of the estimated relative rotation between cameras $i$ and $j$, computed as the angle in the axis-angle representation of $\M R_{ij}^{-1}\M R_{ij}^{\M{GT}}$, and $\V t_{err}^{ij}$ is the error of the estimated translation computed as the angle between the two unit vectors corresponding to the translations~\cite{IMC2020}.} 
{The errors are zoomed into an interesting interval and are shown as functions of varying Gaussian noise from $0$px to $16$px.}

Due to the
approximations used in our proposed 
\sftm-based and \sfaf-based solvers, 
these solvers exhibit non-zero errors for zero noise. However, at noise levels 
{$\geq 4$px},  
our 
\sftmdR solvers return comparable or even better (w/ \texttt{ENM}) results than the \sft solver. 
For noise 
{$\geq 8$px},  
{also the \sftmR solver with \texttt{ENM} returns  slightly more accurate poses than the \sft solver with \texttt{ENM}.}
In general, the effect of increasing image noise is
less visible for approximate \sftm-based and \sfaf-based solvers. In this case, the error of the approximation is dominating the error introduced by the noise in the image correspondences. While for \sfaf-based solvers the approximation error is dominant at all noise levels, for \sftm-based solvers,  at noise 
{$\geq 4$px}, the error introduced by the approximate mean point correspondence (and points in their vicinity in $\delta$-based solvers) is suppressed by the error introduced by noise in the remaining point correspondences.\footnote{This can be seen from the comparable pose accuracy of the \texttt{5pt+P3P} and \texttt{4p3v(M)} solvers for 
{$\geq 8$px} noise, and the comparable accuracy of the \texttt{5pt+P3P} and \texttt{4p3v(M$\pm\delta$)} solvers for 
{$\geq 4$px} noise.} 
Note that, although the pose errors for 
the \texttt{4p3V(A)+R+ENM} solver
are higher than those of the rest of the solvers,
as shown in our real experiments, this solver still returns reasonably low errors 
to provide local optimization (LO) within RANSAC with a good initialization in real-world settings.
Further, note that the \sft solver samples one more point (real correspondence) in the first two cameras, and these points are affected only by the considered noise.
This shows that the mean point correspondence used in the \sftm-based solver is a good approximation to a real correspondence. 
The recent state-of-the-art \sfhc solver~\cite{Hruby_cvpr2022} is failing in about 50\% of the instances for noiseless data, even though the solver was trained on the ETH3D dataset. Thus, the median errors are significantly larger than the median errors of the remaining solvers. 

\begin{figure}[!t]
    \centering
    \resizebox{0.8\linewidth}{!}{
\begin{tikzpicture} 
        \begin{axis}[%
        hide axis, xmin=0,xmax=0,ymin=0,ymax=0,
        legend style={draw=white!15!white, 
        line width = 1pt,
        legend  columns =4, 
        /tikz/every even column/.append style={column sep=0.2cm},
        }
        ]
        
        \addlegendimage{Seaborn2}
        \addlegendentry{\texttt{5p(E)}~\cite{Nister-5pt-PAMI-2004}};
        \addlegendimage{Seaborn4}        \addlegendentry{\texttt{4p(M)}};
        \addlegendimage{Seaborn5}        \addlegendentry{\texttt{4p(M$\pm\delta$)}};
        \addlegendimage{Seaborn3}
        \addlegendentry{\texttt{4p(A)}};
        \addlegendimage{white}
        \addlegendentry{~}

        \addlegendimage{black!30}
        \addlegendentry{w/o \texttt{ENM}};
        \addlegendimage{black!30,dash pattern=on 2pt off 1pt on 2pt off 1pt}
        \addlegendentry{w/ \texttt{ENM}};
        \end{axis}
    \end{tikzpicture}}
    
    \includegraphics[width=0.8\linewidth]{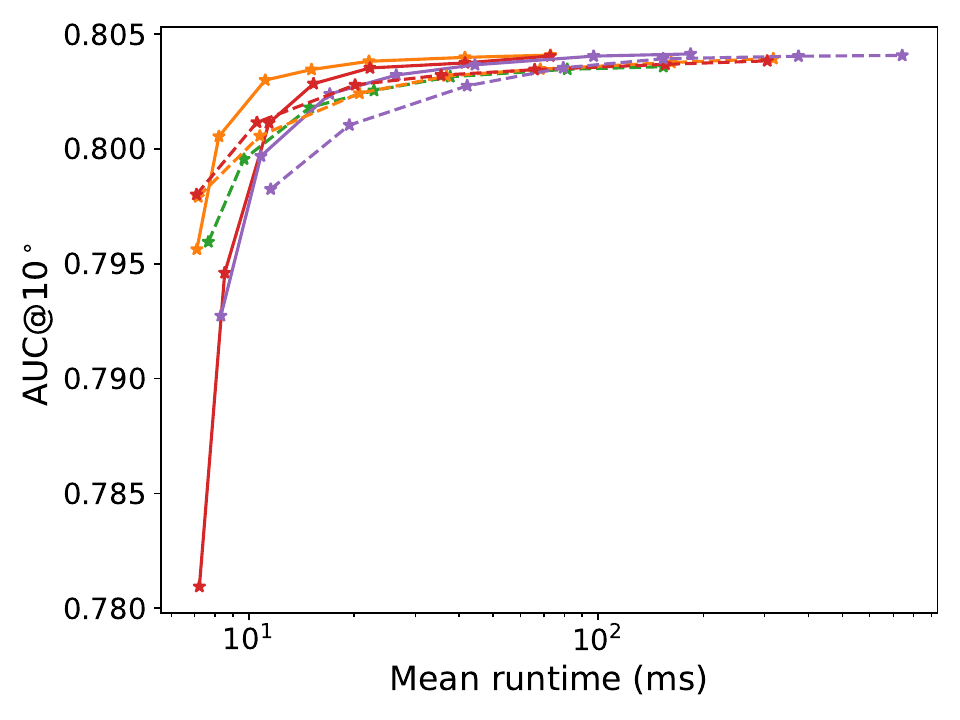}
    
    \includegraphics[width=0.8\linewidth]{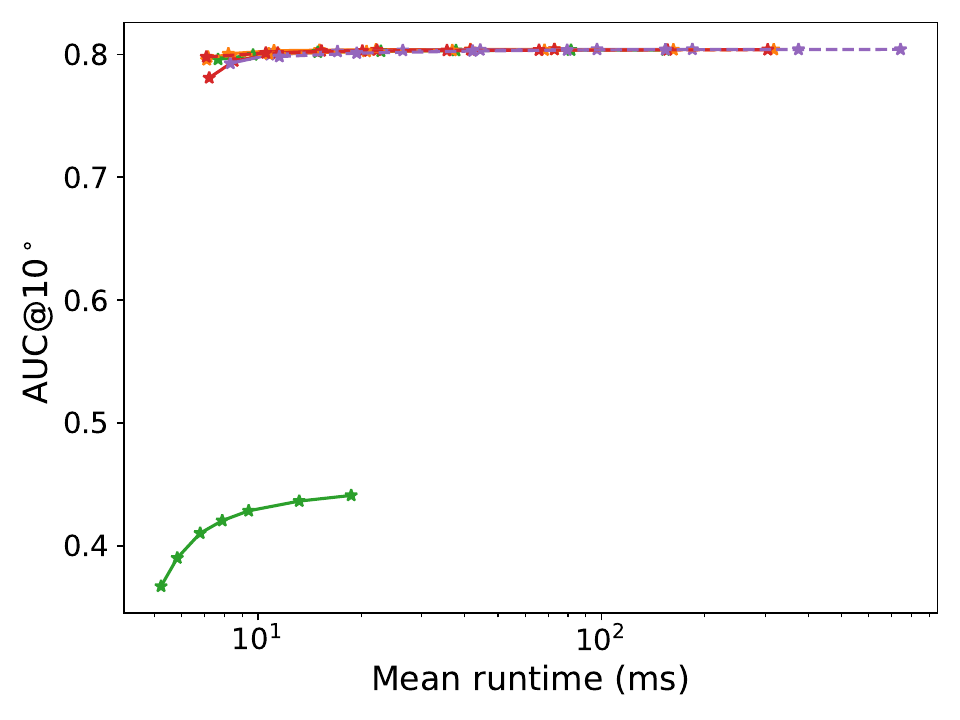}
    \caption{Speed-accuracy evaluation of various solvers for two view relative pose estimation, evaluated using PoseLib~\cite{PoseLib} on 12 scenes from the Phototourism dataset~\cite{IMC2020} (excluding \textit{St.~Peter's Square}). We report the AUC@10$^\circ$ of the pose error and vary the number of Poselib RANSAC iterations ($\{10, 20, 50, 100, 200, 500, 1000\}$) for a fixed 5 px epipolar threshold. The top plot is a zoomed-in version of the bottom plot. }
    \label{fig:twoview}
\end{figure}

\subsection{Two-view approximate solutions inside RANSAC}
\label{sec:twoview}
In the next experiment, we evaluate the discussed two-view solvers, \ie, the \texttt{5p(E)}, \texttt{4p(M)}, \texttt{4p(M$\pm \delta$)}, and \texttt{4p(A)} solvers,  inside RANSAC as well.
This experiment indicates how the proposed approximate solvers would have behaved if used as two-view solvers. Note that in the two-view case, the proposed filtering (\texttt{+F}) and refinement (\texttt{+R}) using the $4^{th}$ point correspondence in the third view are not applicable.

Fig.~\ref{fig:twoview} shows the speed-accuracy evaluation of the solvers for the problem of two view relative pose estimation evaluated using PoseLib RANSAC~\cite{PoseLib} on 12 scenes from the Phototourism dataset~\cite{IMC2020} (excluding \textit{St.~Peter's Square}) with pairwise point correspondences obtained using~\cite{detone2018superpoint, lindenberger2023lightglue}. The statistic reported is the AUC@10$^\circ$ of the pose error for a varied number of RANSAC iterations ($\{10, 20, 50, 100, 200, 500, 1000\}$) and a fixed epipolar threshold of 5px. The upper figure is a zoom-in of the lower figure to an interesting interval, where differences between the solvers are more visible. Although all proposed solvers (except the 
\texttt{4p(A)} solver without \texttt{ENM}) have a performance comparable to that of the state-of-the-art two view \texttt{5pt} solver, the \texttt{5pt} solver is the best performing one for the two-view scenario. This result is not surprising, given the well-known good performance (in terms of speed and accuracy) of the \texttt{5pt} solver and not very high outlier contamination of the data (for which sampling one point less would have potentially had a more visible effect). It also indicates that the proposed modifications, \ie,  the filtering \texttt{+F} the and refinement \texttt{+R}, for the three-view scenario are important and are making the proposed 4p3v approximate solvers practical and more precise than the \sft sovler. 

\begin{figure*}[t!]
    \centering
    \begin{tikzpicture}
    \node[anchor=south west, inner sep=0] (image) at (0,0) {\includegraphics[width=0.24\textwidth]{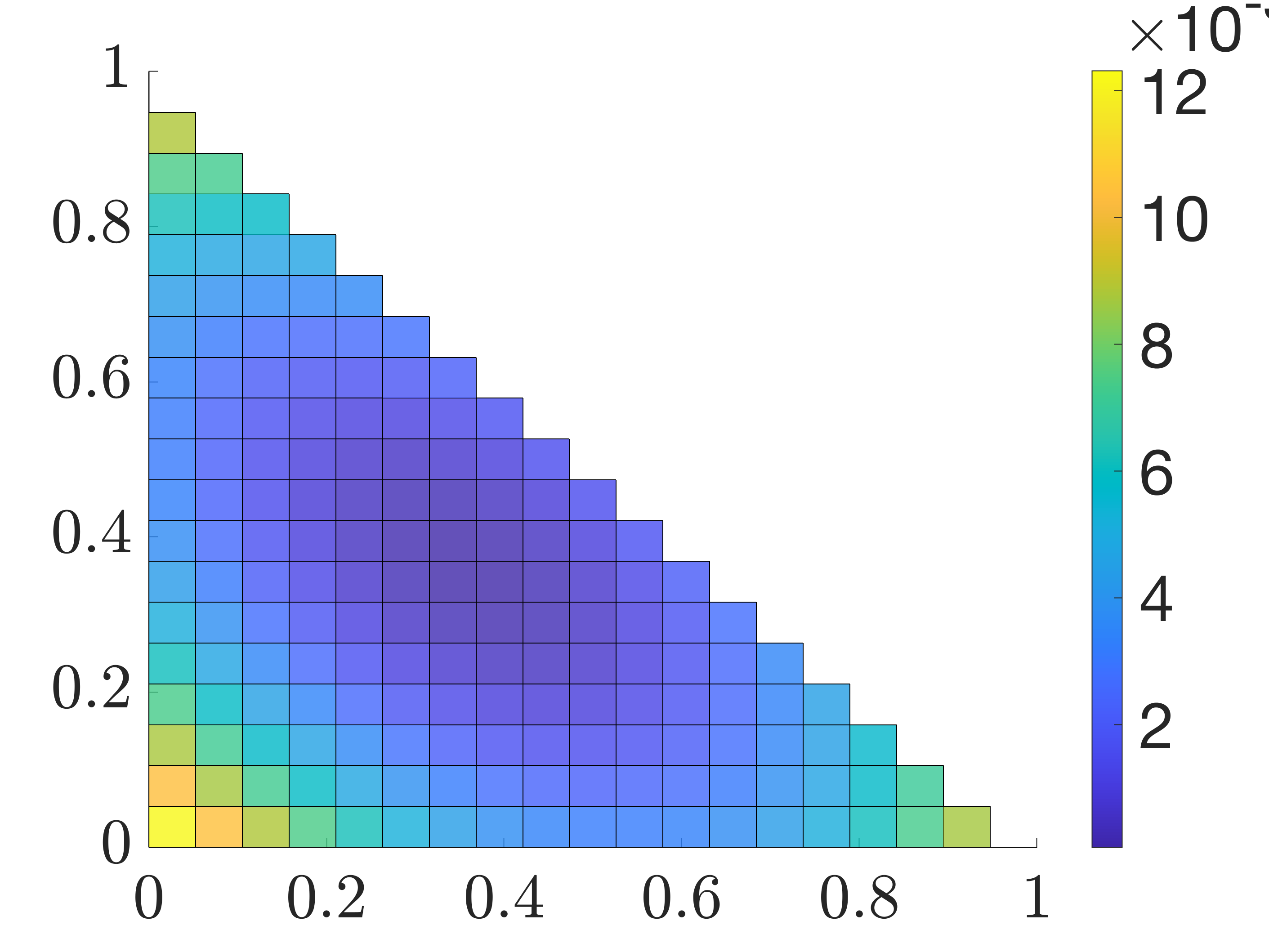}};
        \begin{scope}[shift={(image.south west)}, x={(image.south east)}, y={(image.north west)}]
            \fill[white] (0.85, 0.93) rectangle (1, 1); 
            \node[anchor=center] at (0.93, .97) {\tiny $\times 10^{\scalebox{0.5}[1.0]{\( - \)}3}$};
        \end{scope}
    \end{tikzpicture}   
    \includegraphics[width=0.24\textwidth]{figures/bary_4p/st_peters_bary_rotation.png}
    \includegraphics[width=0.24\textwidth]{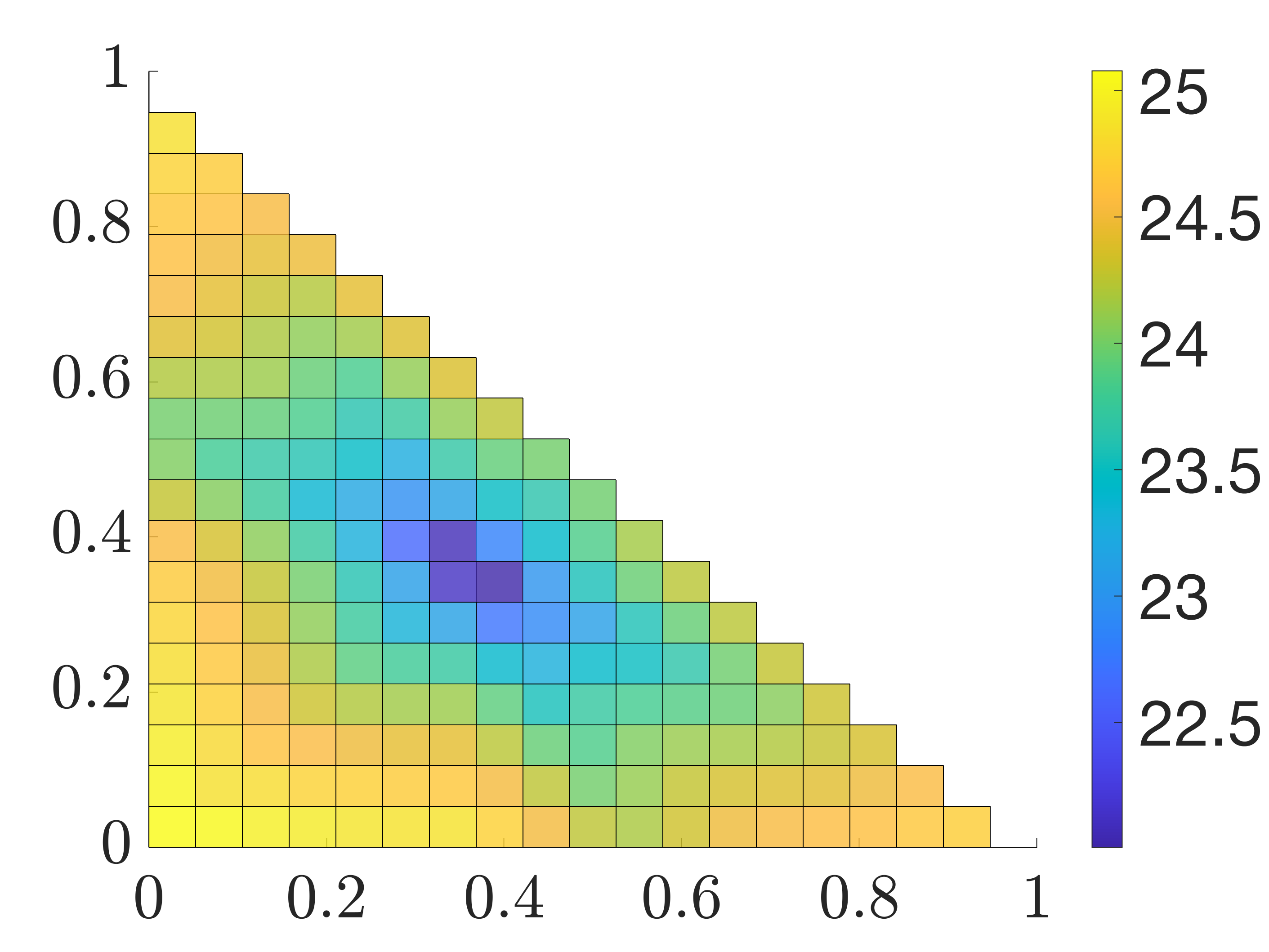}
    \includegraphics[width=0.24\textwidth]{figures/bary_4p/st_peters_bary_inlier.png}

    \hfill
    
    \begin{tikzpicture}
    \node[anchor=south west, inner sep=0] (image) at (0,0) {\includegraphics[width=0.24\textwidth]{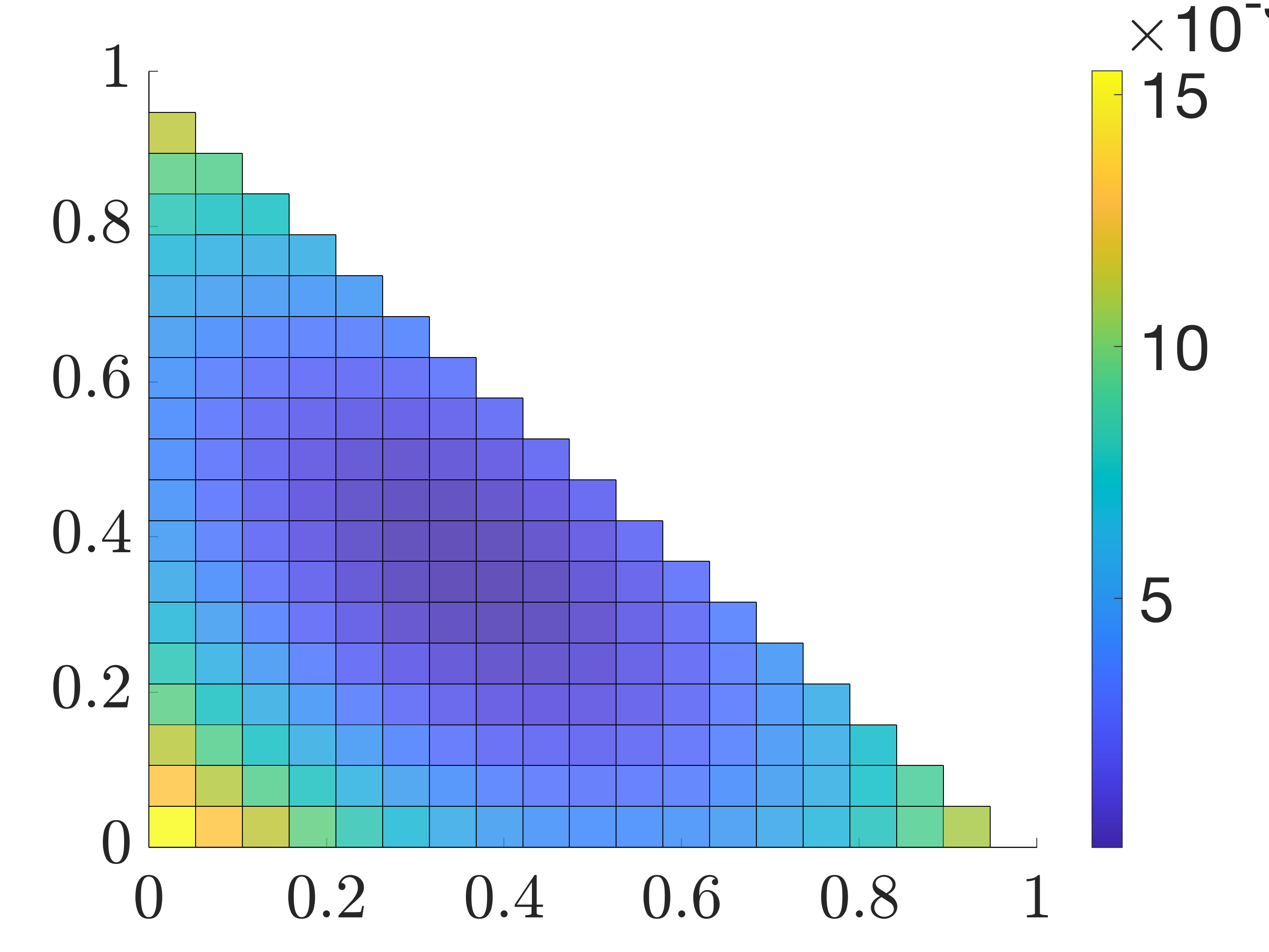}};
        \begin{scope}[shift={(image.south west)}, x={(image.south east)}, y={(image.north west)}]
            \fill[white] (0.85, 0.93) rectangle (1, 1); 
            \node[anchor=center] at (0.93, .97) {\tiny $\times 10^{\scalebox{0.5}[1.0]{\( - \)}3}$};
        \end{scope}
    \end{tikzpicture}   
    \includegraphics[width=0.24\textwidth]{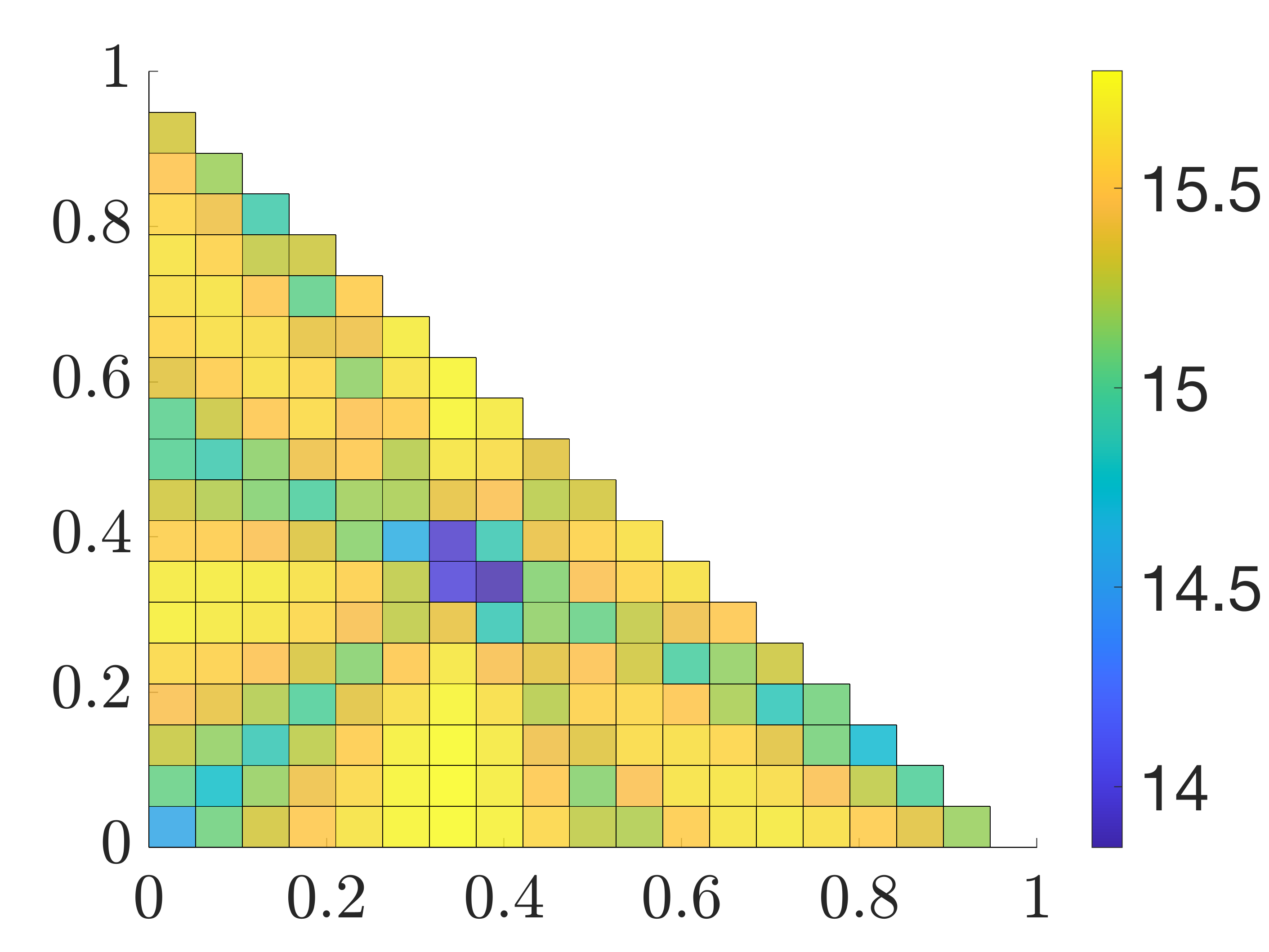}
    \includegraphics[width=0.24\textwidth]{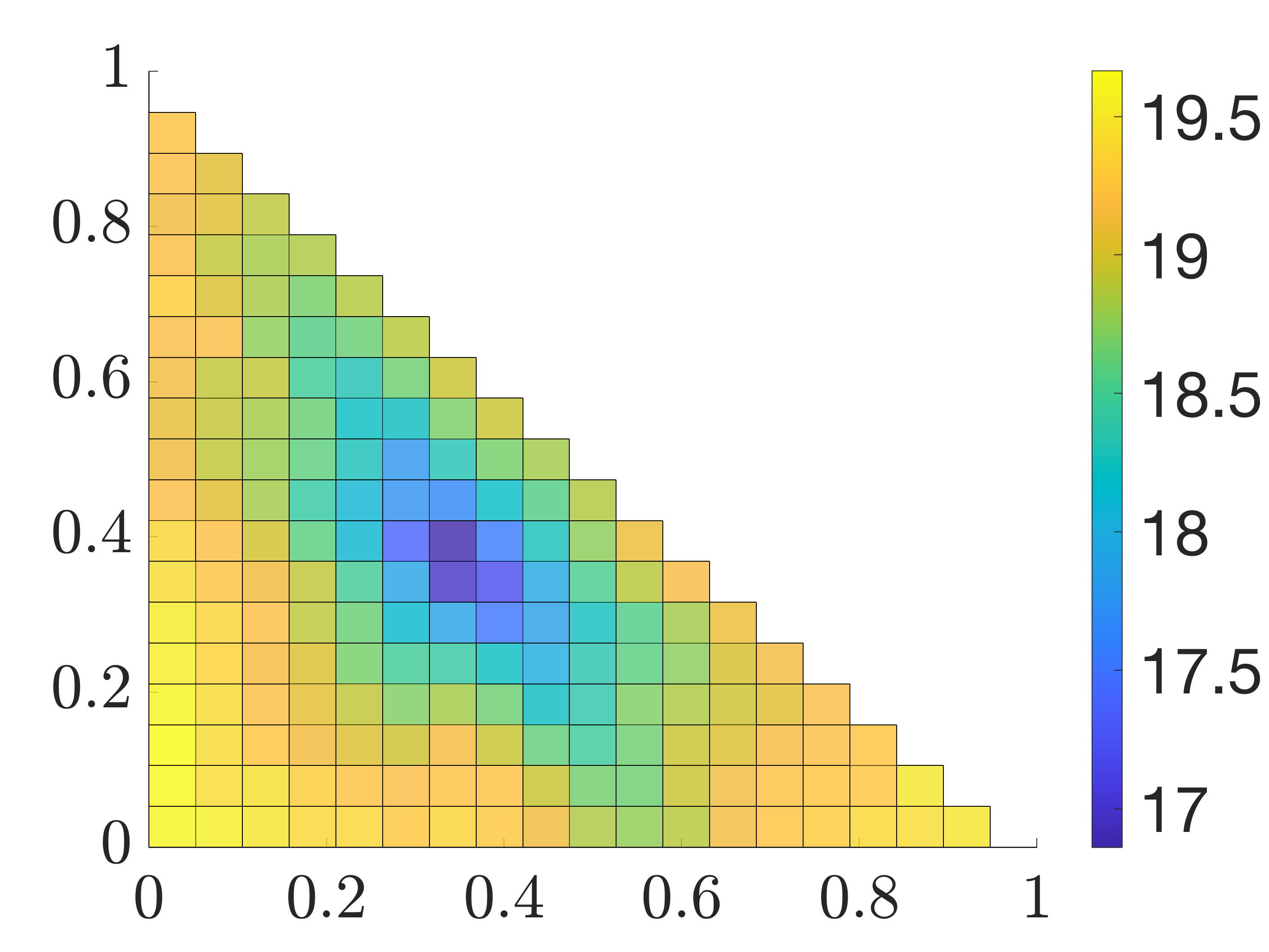}
    \includegraphics[width=0.24\textwidth]{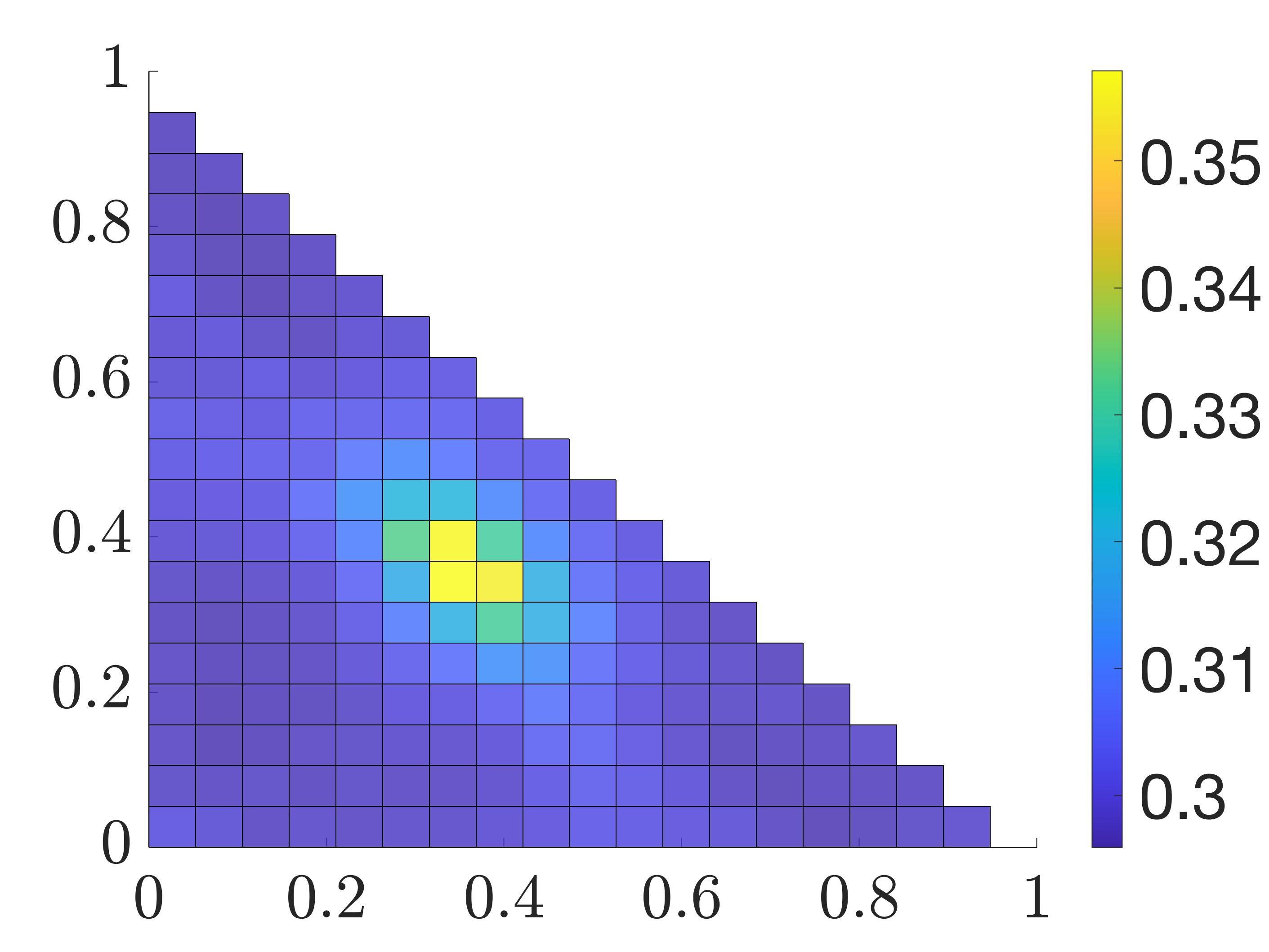}

    \hfill

    \begin{tikzpicture}
    \node[anchor=south west, inner sep=0] (image) at (0,0) {\includegraphics[width=0.24\textwidth]{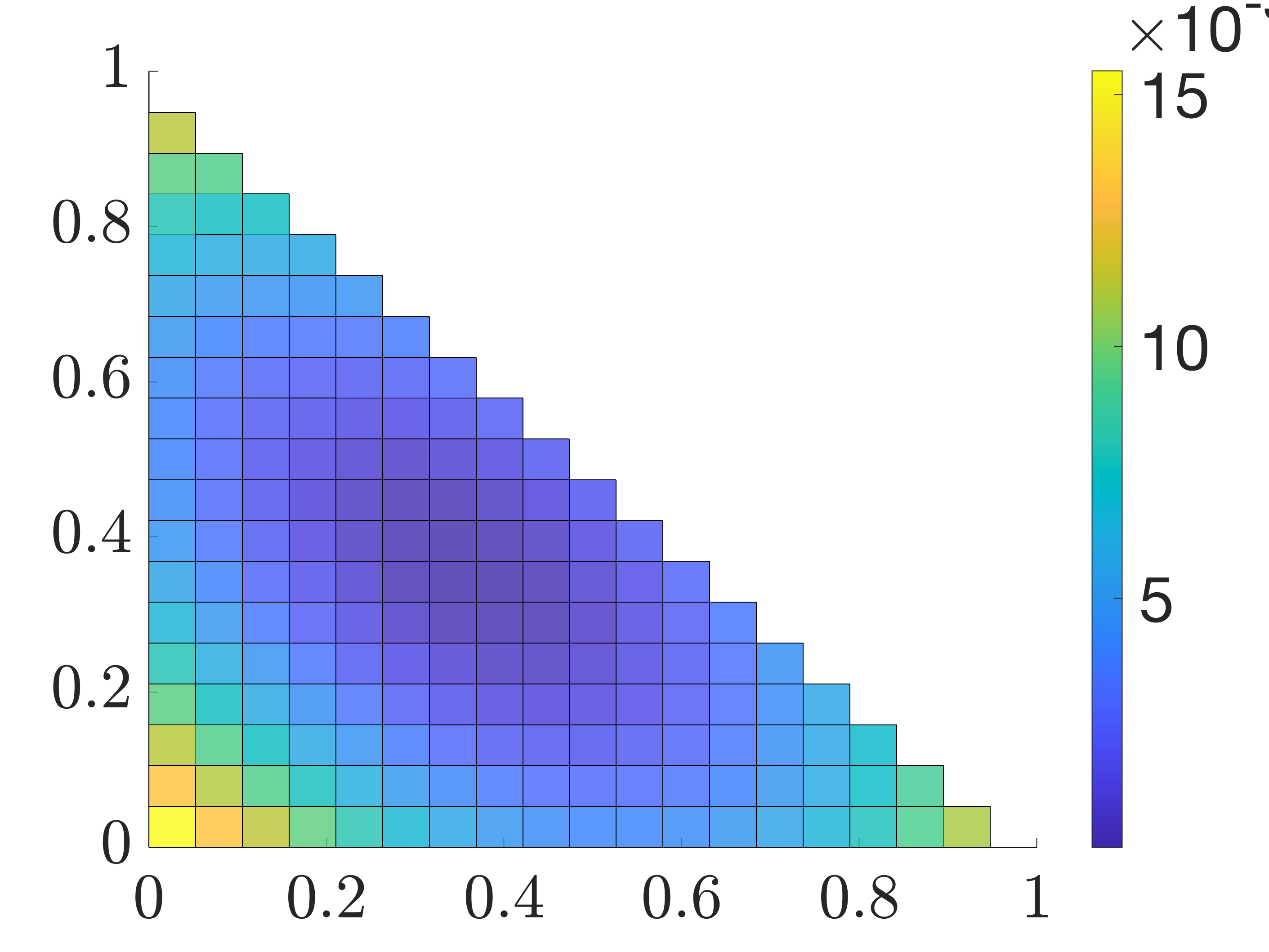}};
        \begin{scope}[shift={(image.south west)}, x={(image.south east)}, y={(image.north west)}]
            \fill[white] (0.85, 0.93) rectangle (1, 1); 
            \node[anchor=center] at (0.93, .97) {\tiny $\times 10^{\scalebox{0.5}[1.0]{\( - \)}3}$};
        \end{scope}
    \end{tikzpicture}  
    \includegraphics[width=0.24\textwidth]{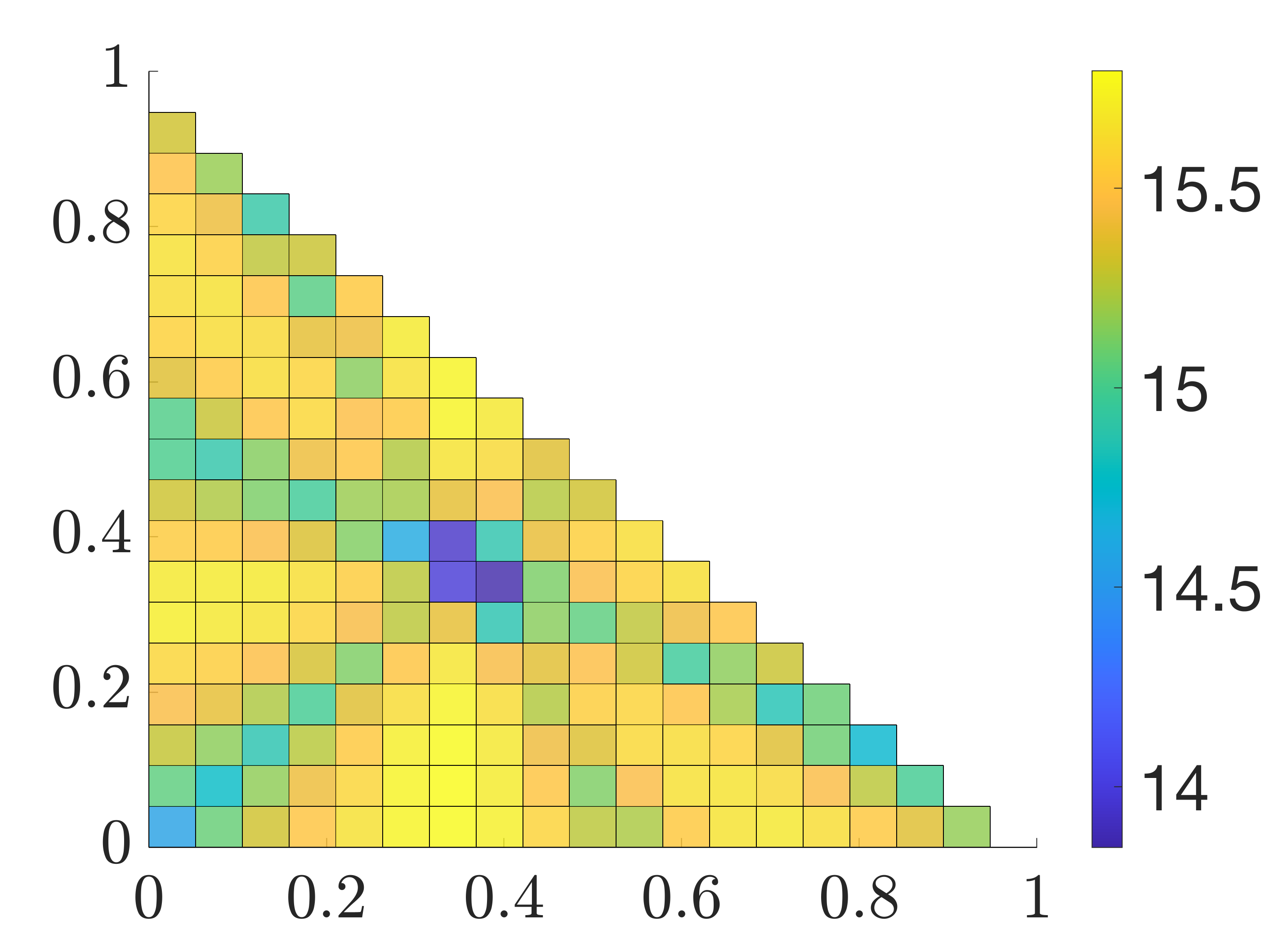}
    \includegraphics[width=0.24\textwidth]{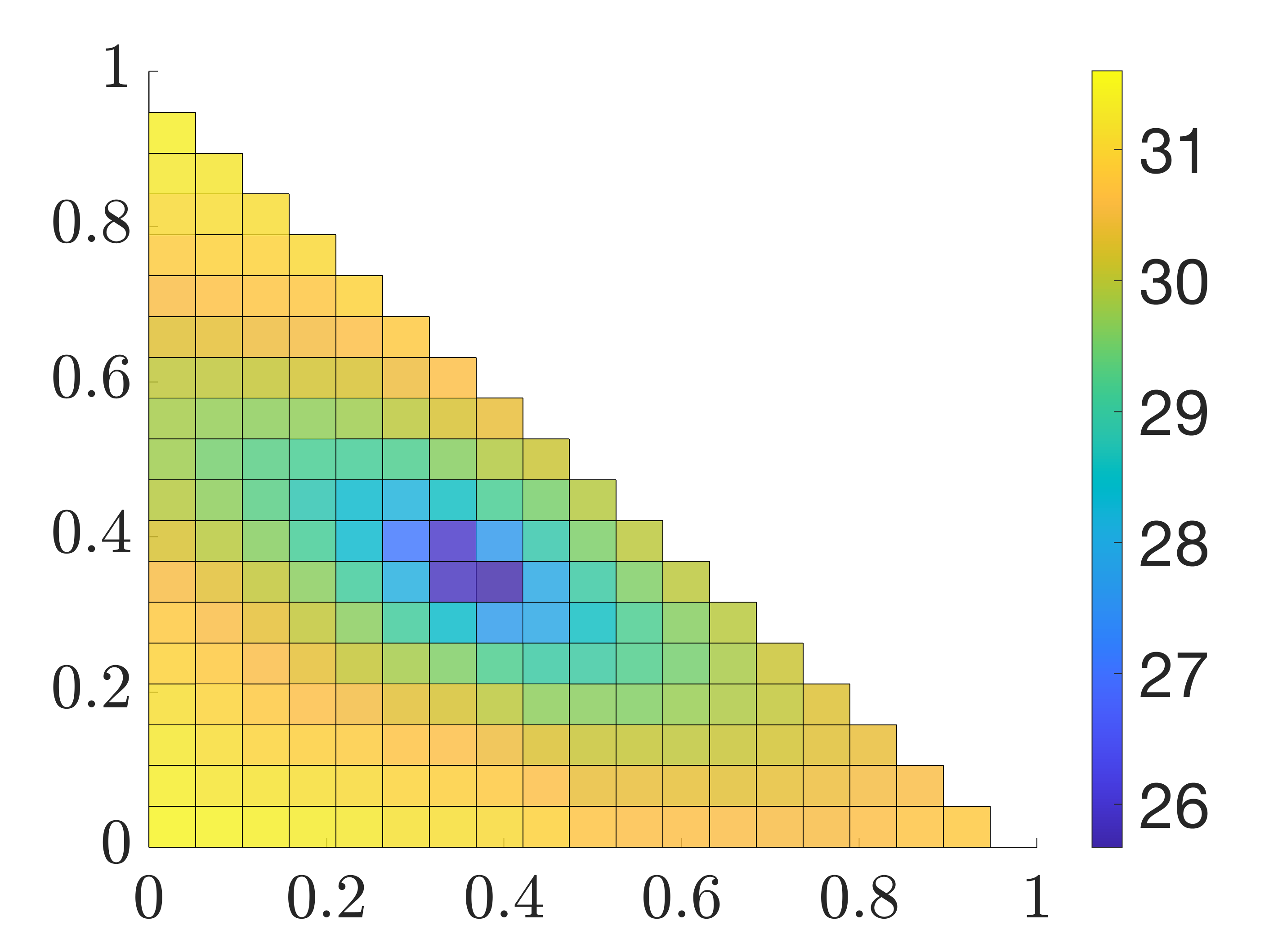}
    \includegraphics[width=0.24\textwidth]{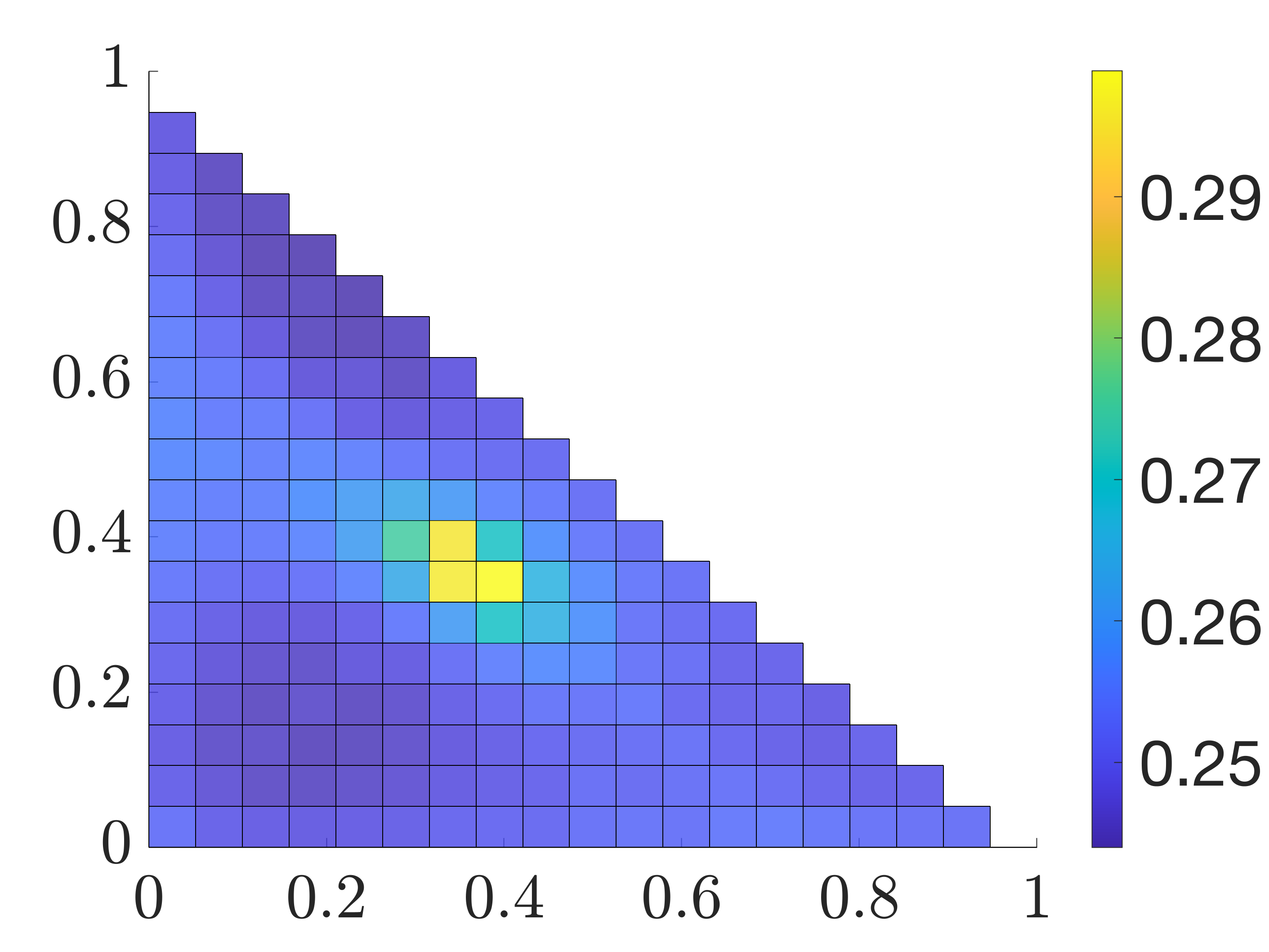}

\caption{Left to right: Distribution of the average symmetric epipolar error (top: 0.3337, 0.3327) (middle: 0.3319,    0.3308) (bottom: 0.3355, 0.3290); rotation error (top: 0.3373, 0.3349) (middle: 0.3373,    0.3347) (bottom: 0.3261, 0.3496); translation error (top: 0.3336, 0.3417) (middle: 0.3325,    0.3382) (bottom: 0.3213, 0.3515); and percentage of inliers gathered (top: 0.3266, 0.3434) (middle: 0.3377,    0.3354) (bottom: 0.3198, 0.3552), as a function of the barycentric coordinates of the triangle in the second image \wrt the mean point of the corresponding triangle in the first image 
on 485k four-tuples of correspondences from scenes (top) \textit{St.~Peter's Square}, (middle) \textit{Sacre Coeur}, and (bottom) \textit{Temple Nara Japan} from the PhotoTourism dataset~\cite{IMC2020}. For each metric, we fit a 2D Gaussian distribution and report the mean of the distribution in brackets.}
\label{fig:mean_stats_pt}
\end{figure*}

\subsection{ Accuracy of the mean point correspondence} 
\label{sec:bary}

Fig.~3 in the main paper showed results obtained by establishing correspondences between the mean of the triangle in one image and various points in the triangle in the second image. 
We expressed points in the second triangle via their barycentric coordinates and uniformly sample $19 \times 19$ barycentric coordinates $(a,b)\in [0,1]^2$, such that $a+b \leq 1$ (ensuring 
points inside the triangle). 
The 
3rd coordinate is given as $c = 1-a-b$. 
For each correspondence, we measured the symmetric epipolar error \wrt the ground truth 
pose, 
translation and rotation errors, and the percentage of inliers.
Fig.~3 in the main paper showed the rotation error and percentage of inliers, as observed 
for the \textit{St.~Peter's Square} scene from the PhotoTourism dataset~\cite{IMC2020}. 
Here,  Fig.~\ref{fig:mean_stats_pt} shows the same statistics, including translation and symmetric epipolar errors, for the \textit{St.~Peter's Square} scene already used in the main paper (Fig.~\ref{fig:mean_stats_pt} (top row)), and two more scenes from the PhotoTourism dataset: 
\textit{Sacre Coeur} (Fig.~\ref{fig:mean_stats_pt} (middle row)), and \textit{Temple Nara Japan} (Fig.~\ref{fig:mean_stats_pt} (bottom row)).

As with Fig.~3 in the main paper, 
to suppress the effect of discrete sampling, 
for each metric, we fit a 2D Gaussian distribution and report the mean value (in barycentric coordinates) as numbers in brackets in the caption of the figure. 
As can be seen, the same conclusion can be drawn from Fig.~\ref{fig:mean_stats_pt} as from Fig.~3 in the main paper: 
The optima of the studied metrics are reached very close to the mean point of the triangles, which has barycentric coordinates $(0.\bar{3}, 0.\bar{3})$. 
This validates our approach of using the mean point correspondence as an approximate correspondence in our \sftm-based solvers. 

\begin{table}[!t]
    \centering
    \resizebox{\linewidth}{!}{
    \begin{tabular}{ l | l | c c | c c c}
\toprule
Estimator & \multicolumn{1}{|c|}{$\delta$} & AVG $(^\circ)$ $\downarrow$ & MED $(^\circ)$ $\downarrow$ & AUC@5 $\uparrow$ & @10 $\uparrow$ & @20 $\uparrow$ \\
\midrule

\multirow{10}{*}{\sftmd}
 & 0.2 & 4.03 & 2.09 & 57.81 & 73.48 & 84.67 \\
 & 0.1 & 3.99 & 2.03 & 58.29 & 73.71 & 84.78 \\
 & 0.09 & 3.99 & 2.02 & 58.52 & 73.95 & 84.93 \\
 & 0.08 & \textbf{3.95} & 2.04 & \underline{58.66} & \underline{74.11} & \textbf{85.04} \\
 & 0.07 & 4.02 & 2.03 & 58.63 & 74.01 & \underline{84.98} \\
 & 0.06 & 4.01 & \underline{2.01} & 58.62 & 73.92 & 84.90 \\
 & 0.05 & \underline{3.98} & \textbf{1.98} & \textbf{58.94} & \textbf{74.19} & \textbf{85.04} \\
 & 0.01 & 4.01 & 2.07 & 57.90 & 73.60 & 84.75 \\
 & 0.005 & 4.12 & 2.07 & 57.54 & 73.23 & 84.55 \\
 & 0.001 & 4.28 & 2.14 & 56.65 & 72.51 & 84.08 \\
\midrule
\multirow{10}{*}{\sftmdR}
 & 0.2 & 3.75 & \underline{1.87} & 61.36 & \underline{75.83} & \underline{86.07} \\
 & 0.1 & 3.75 & \underline{1.87} & 61.33 & 75.78 & 86.03 \\
 & 0.09 & 3.74 & \underline{1.87} & 61.39 & 75.81 & 86.02 \\
 & 0.08 & \textbf{3.71} & \underline{1.87} & \textbf{61.45} & \textbf{75.90} & \textbf{86.16} \\
 & 0.07 & 3.76 & \underline{1.87} & \underline{61.41} & 75.82 & 86.06 \\
 & 0.06 & 3.79 & \underline{1.87} & 61.40 & 75.81 & 86.04 \\
 & 0.05 & 3.75 & \textbf{1.86} & 61.38 & 75.79 & 86.03 \\
 & 0.01 & 3.75 & \underline{1.87} & 61.25 & 75.70 & 85.95 \\
 & 0.005 & \underline{3.73} & \underline{1.87} & 61.22 & 75.75 & 86.03 \\
 & 0.001 & 3.86 & 1.90 & 60.82 & 75.46 & 85.80 \\
\midrule
\multirow{10}{*}{\sftmdRC}
 & 0.2 & 3.78 & \underline{1.88} & 61.12 & 75.70 & 85.96 \\
 & 0.1 & 3.77 & \textbf{1.87} & 61.16 & 75.68 & 85.95 \\
 & 0.09 & 3.76 & \textbf{1.87} & 61.24 & 75.70 & 85.92 \\
 & 0.08 & \textbf{3.73} & \textbf{1.87} & \textbf{61.30} & \textbf{75.78} & \textbf{86.07} \\
 & 0.07 & 3.77 & \textbf{1.87} & \textbf{61.30} & 75.71 & \underline{85.99} \\
 & 0.06 & 3.80 & \textbf{1.87} & 61.27 & \underline{75.75} & \underline{85.99} \\
 & 0.05 & \underline{3.75} & \textbf{1.87} & \underline{61.28} & 75.73 & 85.98 \\
 & 0.01 & 3.78 & \underline{1.88} & 61.11 & 75.62 & 85.89 \\
 & 0.005 & 3.76 & 1.89 & 61.09 & 75.65 & 85.95 \\
 & 0.001 & 3.90 & 1.91 & 60.53 & 75.26 & 85.67 \\
\midrule
    \end{tabular}}
    \caption{Evaluation of the effects of the scale of the $\delta$ shift on the  \textit{St.~Peter's Square} scene from PhotoTourism~\cite{IMC2020}.}
    \label{tab:delta_ablation}
\end{table}

\section{Ablation studies}
\label{sec:exp:ablations}

This section contains ablation studies to validate our choices in modifications discussed in Sec. 3.2 of the main paper.

\PAR{Validation of $\delta$.} We tested our $\delta$-based solvers for different values of $\delta$ and measured their performance. In general, there is no common value of the $\delta$ shift that leads to the best results on all datasets. This is expected since the precision of the mean-point correspondence depends on many different factors, \eg, the viewing angles of the cameras, the type of the motion, the depth and spatial distributions of the 3D points, \etc 
We set the value for $\delta$ 
by evaluating their effects on the \textit{St.~Peter's Square} scene from the PhotoTourism dataset~\cite{IMC2020},  which we used for validation only and did not include it in the other results for PhotoTourism~\cite{IMC2020} in the paper.  
Tab.~\ref{tab:delta_ablation} shows how the different settings of the scale of the $\delta$ shift affect the accuracy of the $\delta$-based solvers. 
Based on these experiments, we use $\delta = 0.08$ as it typically provides the best or the second best 
results for all variants of the \sftmd-based solvers.
However, note that \sftmd-based solvers achieve a very similar accuracy even with different settings of $\delta$.
Thus, we can conclude that the choice of $\delta$ is not critical.
In some scenarios, the choice of the optimal $\delta$ parameter may be  more scene-dependent and could potentially be set using learning-based approaches.

\begin{figure}
    \centering
    \includegraphics[width=0.8\linewidth]{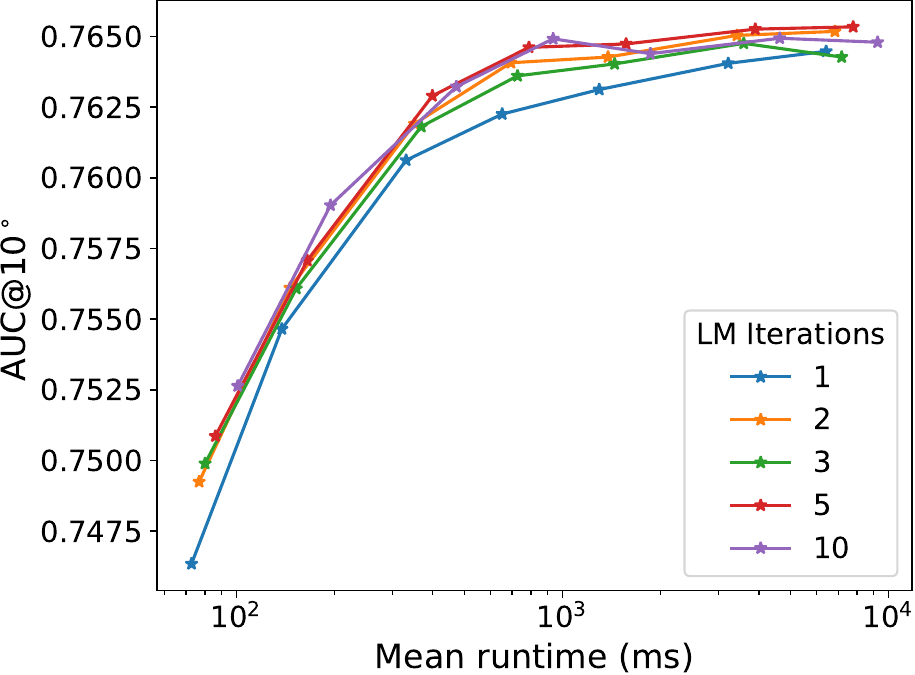}
    \caption{Evaluation of the effects of the number of inner refinement (\texttt{+R}) iterations within the \sftmdRC solver on the \textit{St.~Peter's Square} scene from the PhotoTourism~\cite{IMC2020} dataset. Shown is the speed-accuracy evaluation, where the curves are obtained by varying the number of Poselib RANSAC iterations.}
    \label{fig:r_validation}
\end{figure}

\begin{figure}[!h]
    \centering
    \resizebox{0.95\linewidth}{!}{
    \begin{tikzpicture} 
        \begin{axis}[%
        hide axis, xmin=0,xmax=0,ymin=0,ymax=0,
        legend style={draw=white!15!white, 
        line width = 1pt,
        legend  columns =3, 
        /tikz/every even column/.append style={column sep=0.2cm}
        }
        ]
        
        \addlegendimage{Seaborn4}        \addlegendentry{\sftm};
        \addlegendimage{Seaborn7}
        \addlegendentry{\sftmd};
        \addlegendimage{Seaborn9}
        \addlegendentry{\sftmdR};
        \addlegendimage{Seaborn5}
        \addlegendentry{\sftmdRC}; 
        \addlegendimage{Seaborn6}
        \addlegendentry{\sftmdRCENM}; 
        \addlegendimage{Seaborn3}
        \addlegendentry{\sfafRCENM}; 

        \addlegendimage{black!30,dash pattern=on 1pt off 0.5pt on 1pt off 0.5pt}
        \addlegendentry{$t = 3\text{px}$};
        \addlegendimage{black!30}
        \addlegendentry{$t = 5\text{px}$};
        \addlegendimage{black!30,dash pattern=on 2pt off 1pt on 2pt off 1pt}
        \addlegendentry{$t = 10\text{px}$};      
        \end{axis}
    \end{tikzpicture}}

    \includegraphics[width=0.8\linewidth]{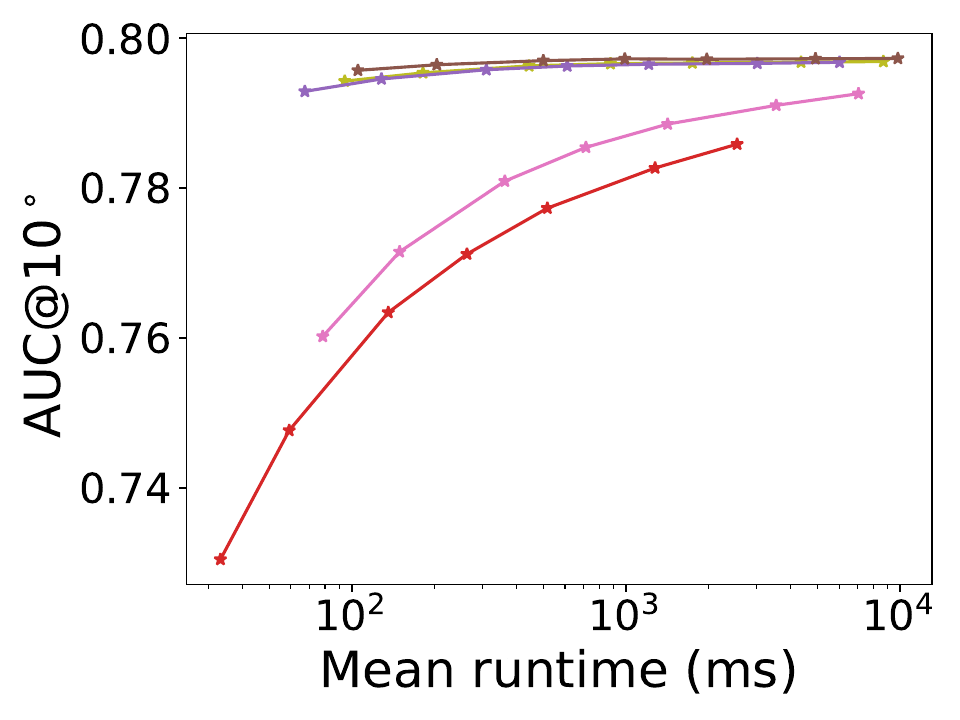}
    \includegraphics[width=0.8\linewidth]{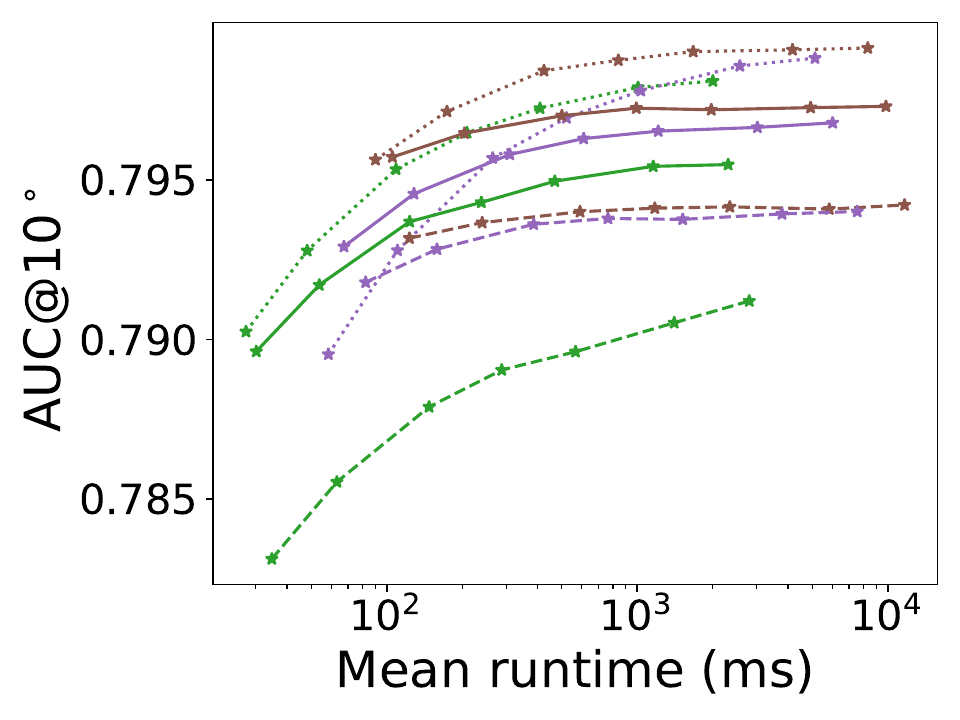}
    \caption{
    {Speed-accuracy trade-off on 12 scenes of Phototourism~\cite{IMC2020}. We show the impact of 
    (\textbf{Top:}) 
     different modifications presented in Sec.~3.2 in the main paper on the performance of the solvers
    and (\textbf{Bottom:}) the maximum epipolar threshold used in RANSAC  on the performance of the three best-performing methods.}
    }
    \label{fig:poselib_ablation_threshold}
\end{figure}

\PAR{Refinement validation.} We also perform validation of the total number of LM steps in the refinement (\texttt{+R}). 
Again, we used the \textit{St.~Peter's Square} scene from the PhotoTourism dataset~\cite{IMC2020} for validation. 
The results of this experiment are shown in Fig.~\ref{fig:r_validation}. We chose the value of 2 for other experiments as it provides good speed-accuracy trade-off across a range of RANSAC iterations. However, we note that other settings provide very similar performance. 

\begin{figure*}
    \centering
    \resizebox{\linewidth}{!}{
    \begin{tikzpicture} 
        \begin{axis}[%
        hide axis, xmin=0,xmax=0,ymin=0,ymax=0,
        legend style={draw=white!15!white, 
        line width = 1pt,
        legend  columns =5, 
        /tikz/every even column/.append style={column sep=0.2cm},
        }
        ]
        
        \addlegendimage{Seaborn4}        \addlegendentry{\sftm};
        \addlegendimage{Seaborn7}
        \addlegendentry{\sftmd};
        \addlegendimage{Seaborn9}
        \addlegendentry{\sftmdR};
        \addlegendimage{Seaborn5}
        \addlegendentry{\sftmdRC}; 
        \addlegendimage{Seaborn6}
        \addlegendentry{\sftmdRCENM}; 
        \end{axis}
    \end{tikzpicture}}

    \centering

    \begin{subfigure}{0.48\linewidth}
    \includegraphics[width=\linewidth]{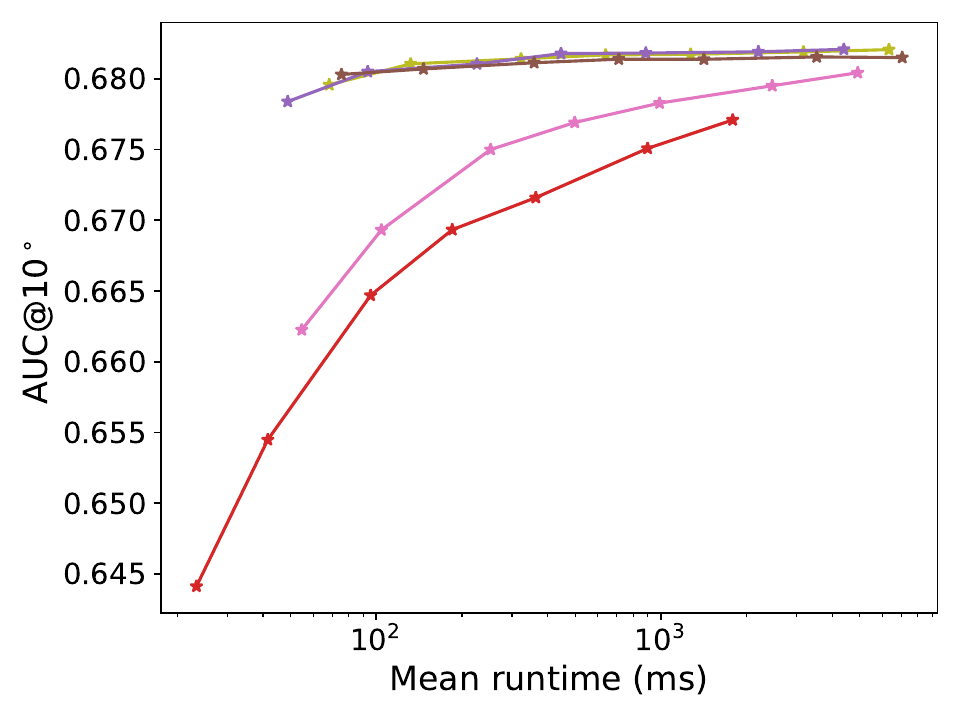}       \caption{Cambridge~\cite{kendall2015cambridge}}
    \end{subfigure}
    \begin{subfigure}{0.48\linewidth}   
    \includegraphics[width=\linewidth]{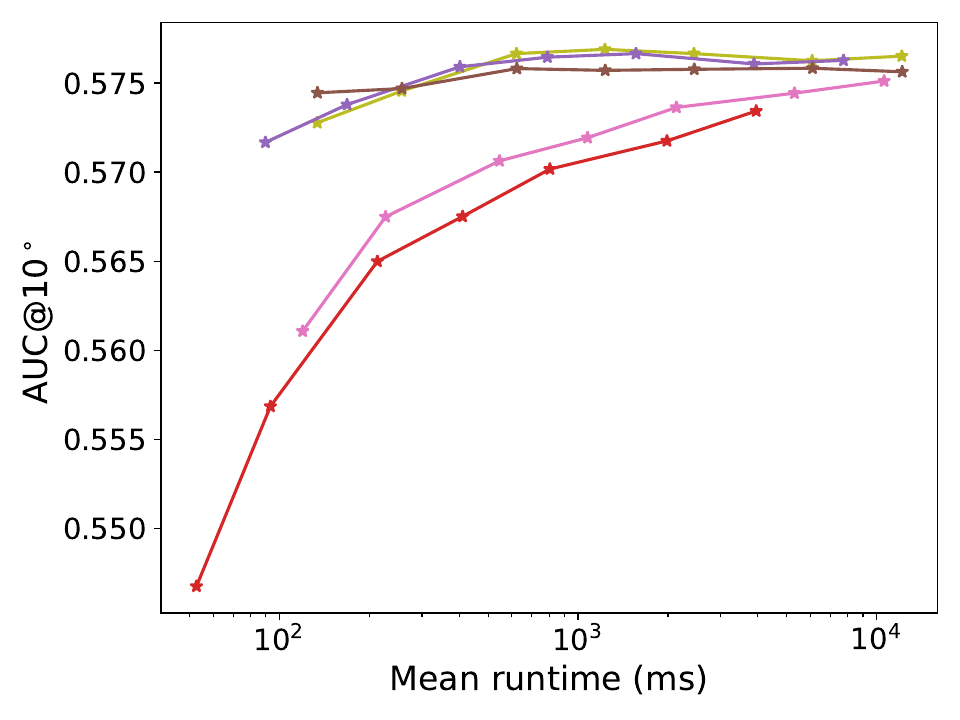}  
    \caption{Aachen~\cite{zhang2021aachen}}
\end{subfigure}
    
    \caption{We show the impact of the different strategies (\texttt{+F/+R/+ENM}) introduced in Sec.~3.2 of the main paper on the performance of our \sftm-based solvers on Cambridge Landmarks~\cite{kendall2015cambridge} and Aachen Day-Night v1.1~\cite{zhang2021aachen}. We report the AUC@10$^\circ$. We vary the number of Poselib RANSAC iterations ($\{100, 200, 500, 1000, 2000, 5000, 10000\}$). We use an epipolar threshold of 5px in RANSAC. Runtimes are averaged over all image triplets. }
    \label{fig:ablation_crenm_cambridge_aachen}
\end{figure*}

\PAR{Validation of \texttt{+F}/\texttt{+R}/\texttt{+ENM}.} 
Fig.~\ref{fig:poselib_ablation_threshold} ablates the impact of individual modifications (\texttt{+R}/\texttt{+C}/\texttt{+ENM}) proposed in Sec.~3.2 in the main paper on the speed-accuracy trade-off.
It especially highlights the importance of the refinement using the 4th point in the third view (\texttt{+R}).
Fig.~\ref{fig:poselib_ablation_threshold} also shows the performance of the top-performing solvers when different maximum epipolar thresholds are used within RANSAC. Compared to \sfafRCENM, the proposed \sftmd-based solvers are not as sensitive to the selection of the epipolar threshold.
The results presented in Fig.~\ref{fig:poselib_ablation_threshold} were obtained on the PhotoTourism dataset.
Fig.~\ref{fig:ablation_crenm_cambridge_aachen} shows results of the same ablation study, focused on the \sftm-based solvers, on the Cambridge Landmarks and Aachen Day-Night v1.1 datasets. 

\section{Experiments on real data}
\label{sec:exp}
In this section, we aim to further study the performance of the proposed methods, supplementing Sec.~4 (paragraph ``Experiments on real data") of main paper with more detailed evaluations. Section~\ref{sec:exp:details} presents results on individual scenes, extending the analysis in Fig.~4 of the main paper. Section~\ref{sec:threshold} investigates the impact of varying the RANSAC epipolar threshold on solver performance. These experiments extend Fig.~\ref{fig:poselib_ablation_threshold} (bottom) by comparing additional solvers across all three datasets. Section~\ref{sec:exp:timing} evaluates and compares the run-times of each of the proposed and state-of-the-art solvers. Section~\ref{sec:exp:outlier} explores the robustness of the solvers under varying inlier ratios using semi-synthetic data. 
In Section~\ref{sec:exp:measure}, we provide results using Poselib RANSAC~\cite{PoseLib} for an alternative pose error metric that considers errors across all three camera pairs. Section~\ref{sec:exp:gc_ransac} evaluates the solvers within the GC-RANSAC~\cite{barath2017graph} framework
for all three datasets.

\begin{figure*}[!t]
    \centering

\resizebox{1.0\linewidth}{!}{
\begin{tikzpicture} 

        \begin{axis}[%
        hide axis, xmin=0,xmax=0,ymin=0,ymax=0,
        legend style={draw=white!15!white, 
        line width = 1pt,
        legend  columns =9, 
        /tikz/every even column/.append style={column sep=0.5cm},
        }
        ]
        
        \addlegendimage{Seaborn1}        \addlegendentry{\sfhc~\cite{Hruby_cvpr2022}};
        \addlegendimage{Seaborn2}
        \addlegendentry{\sft};
        \addlegendimage{Seaborn3}
        \addlegendentry{\sfafRC};
        \addlegendimage{Seaborn4}
        \addlegendentry{\sftmRC}; 
        \addlegendimage{Seaborn5}
        \addlegendentry{\sftmdRC};
        \addlegendimage{black!30}
        \addlegendentry{w/o \texttt{ENM}};
        \addlegendimage{black!30,dash pattern=on 2pt off 1pt on 2pt off 1pt}
        \addlegendentry{w/ \texttt{ENM}};
        
        \end{axis}
    \end{tikzpicture}}

    \begin{subfigure}{0.24\linewidth}
    \includegraphics[width=\linewidth]{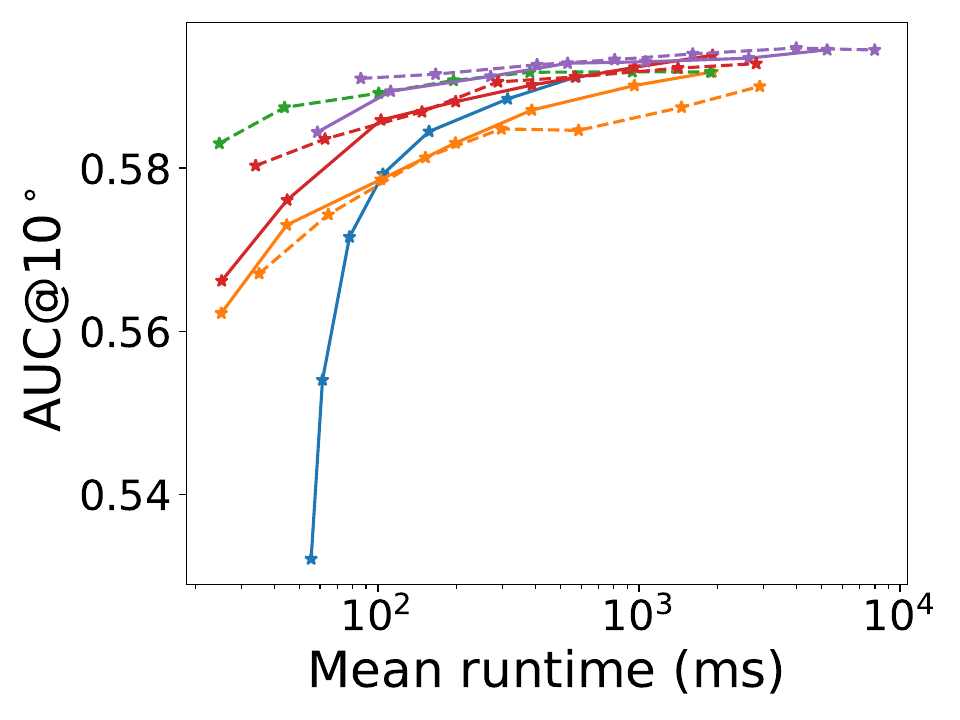}
    \caption{\textit{Great Court}}        
    \end{subfigure}
    \hfill
    \begin{subfigure}{0.24\linewidth}
    \includegraphics[width=\linewidth]{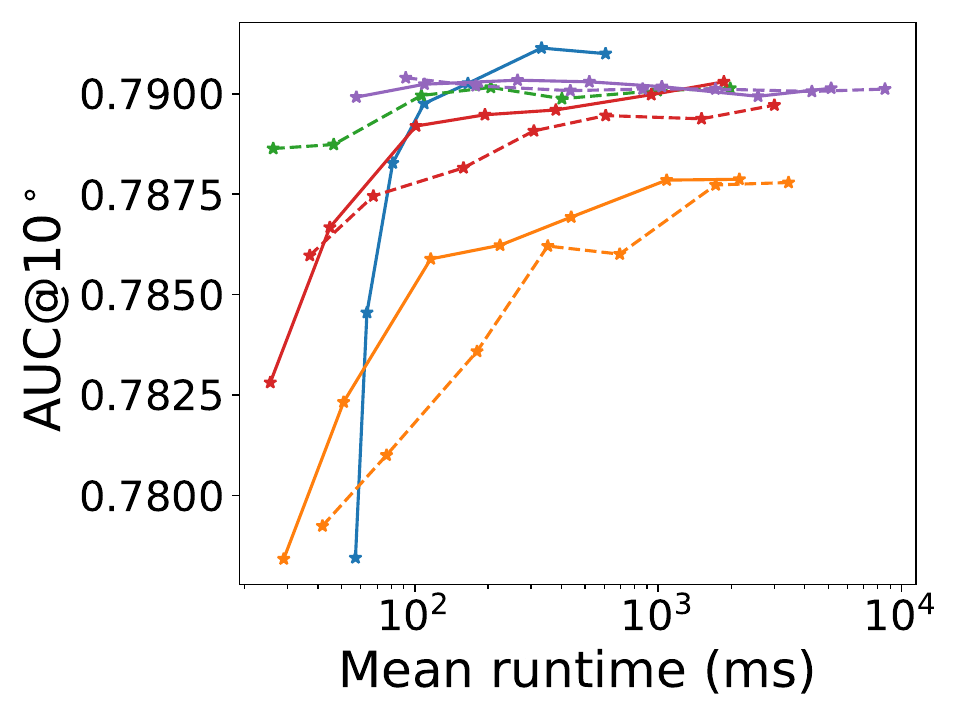}
    \caption{\textit{King's College}}        
    \end{subfigure}
    \hfill
    \hfill
    \begin{subfigure}{0.24\linewidth}
    \includegraphics[width=\linewidth]{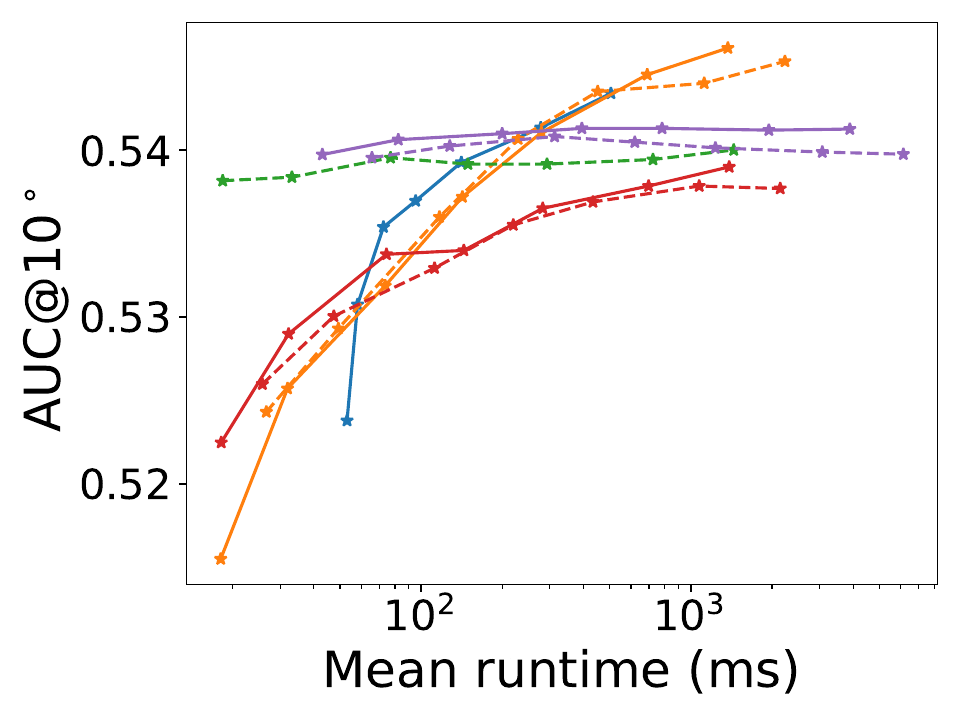}
    \caption{\textit{Old Hospital}}        
    \end{subfigure}
    \hfill        
    \begin{subfigure}{0.24\linewidth}
    \includegraphics[width=\linewidth]{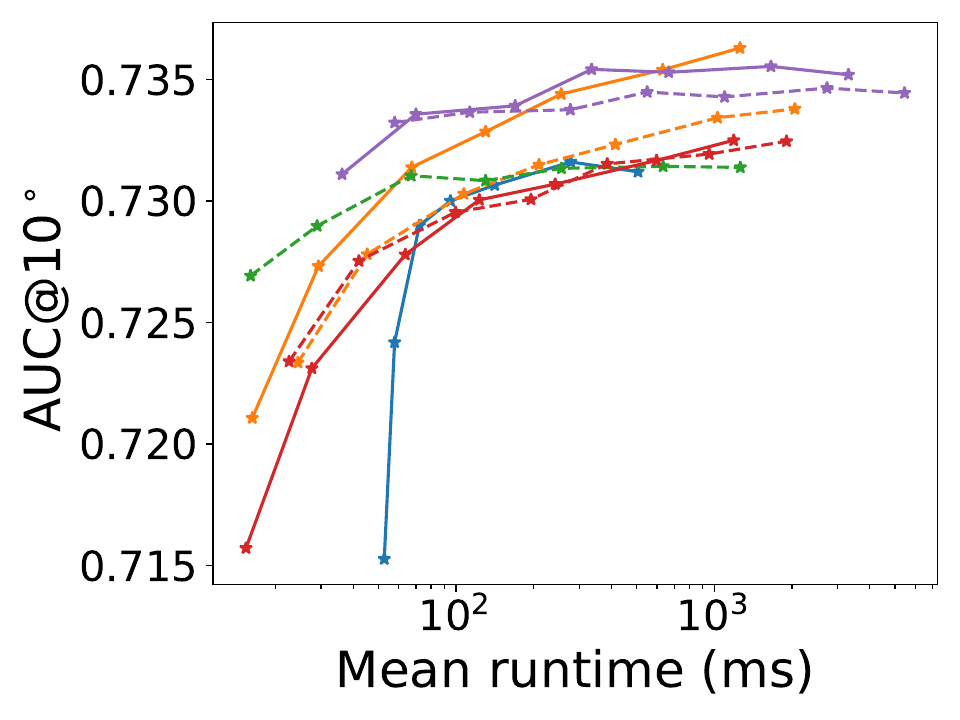}
    \caption{\textit{Shop Facade}}        
    \end{subfigure}

    \vspace{0.5ex}    
    
    \begin{subfigure}{0.24\linewidth}
    \includegraphics[width=\linewidth]{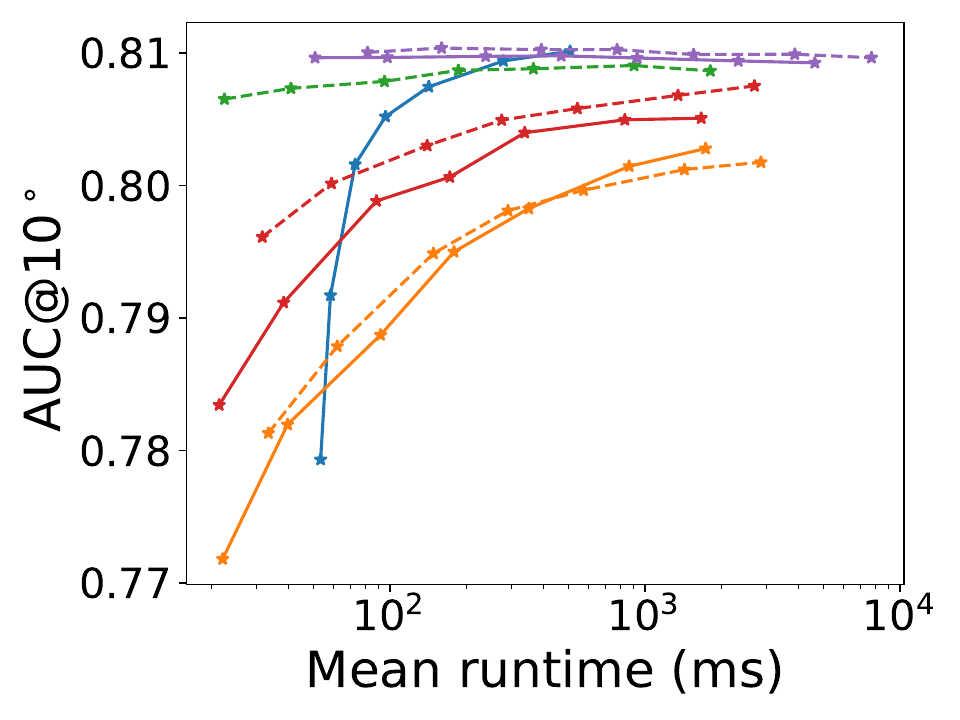}
    \caption{\textit{Brandenburg Gate}}        
    \end{subfigure}
    \hfill    
    \begin{subfigure}{0.24\linewidth}
    \includegraphics[width=\linewidth]{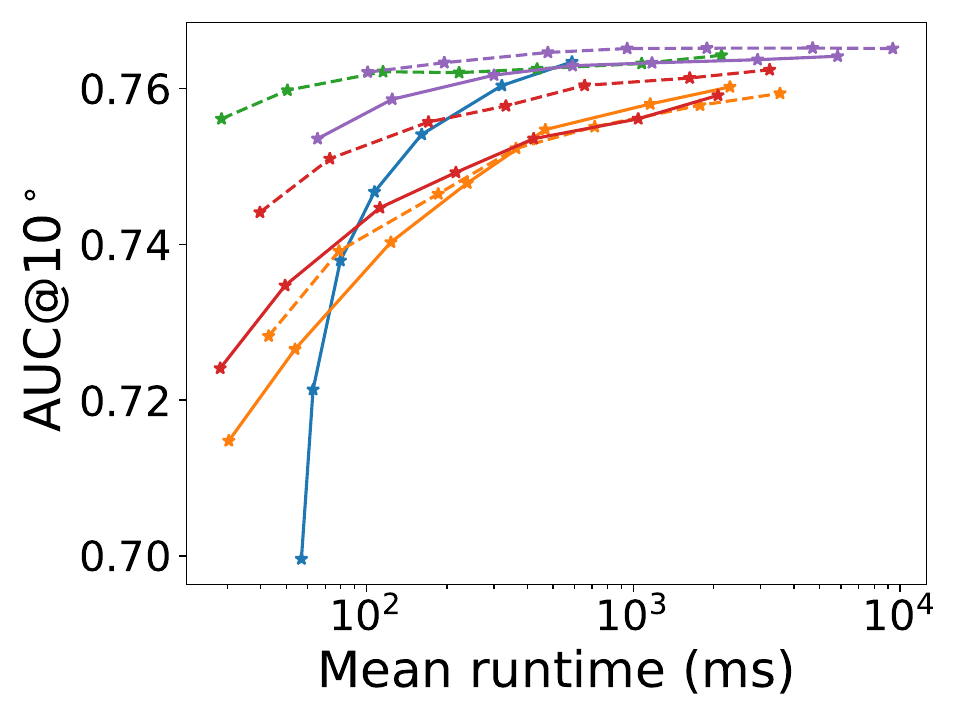}
    \caption{\textit{Buckingham Palace}}        
    \end{subfigure}
    \hfill
    \begin{subfigure}{0.24\linewidth}
    \includegraphics[width=\linewidth]{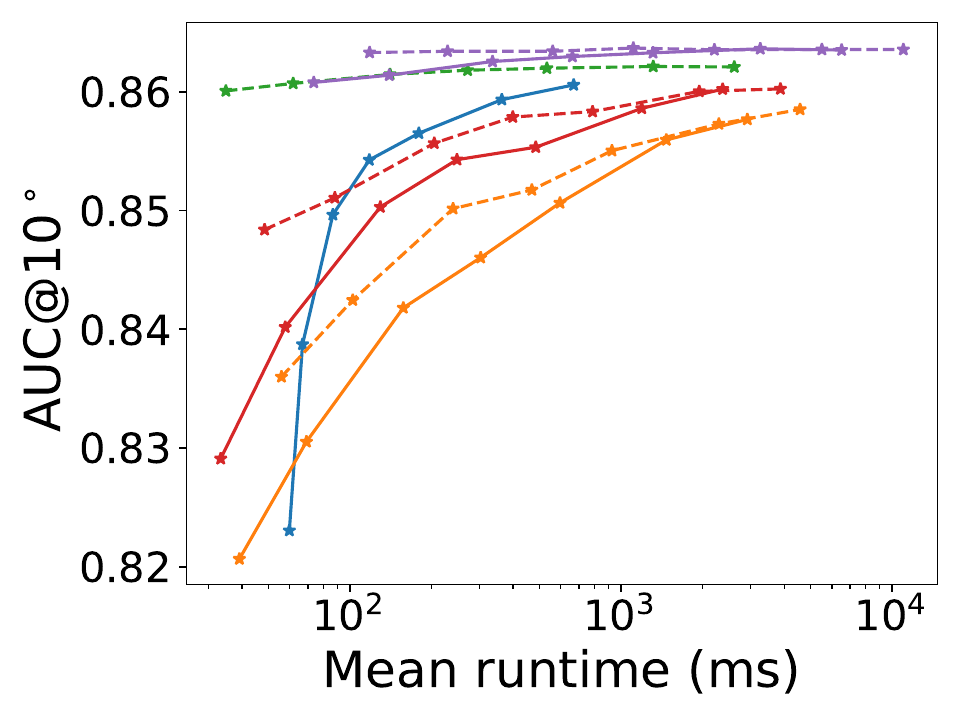}
    \caption{\textit{Colosseum Exterior}}        
    \end{subfigure}
    \hfill
    \begin{subfigure}{0.24\linewidth}
    \includegraphics[width=\linewidth]{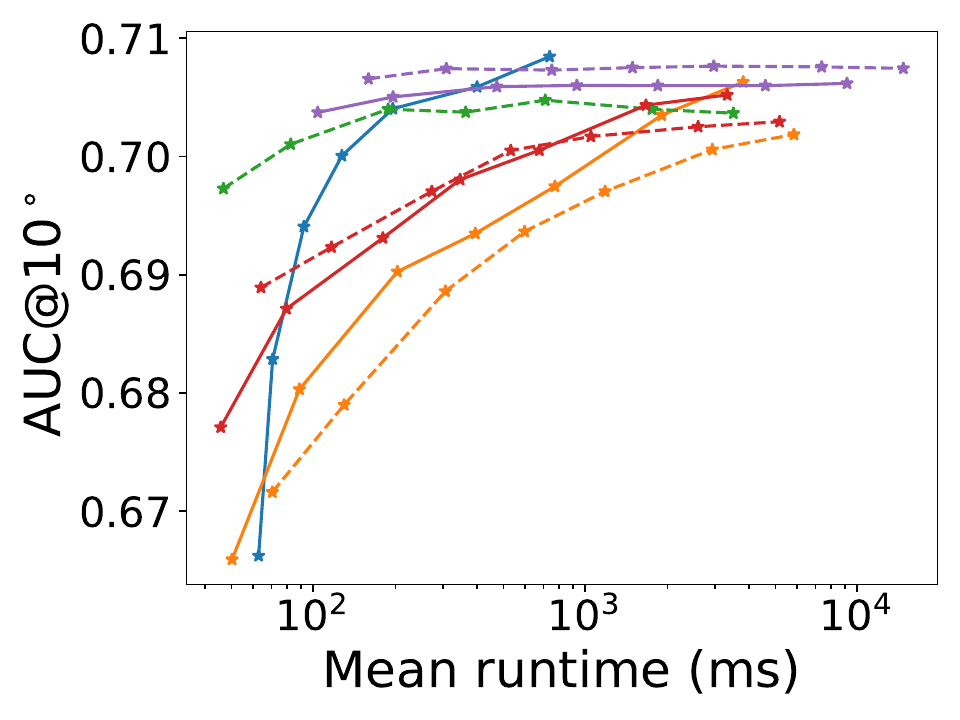}
    \caption{\textit{Grand Place Brussels}}       
    \end{subfigure}

    \vspace{0.5ex}
    
    \begin{subfigure}{0.24\linewidth}
    \includegraphics[width=\linewidth]{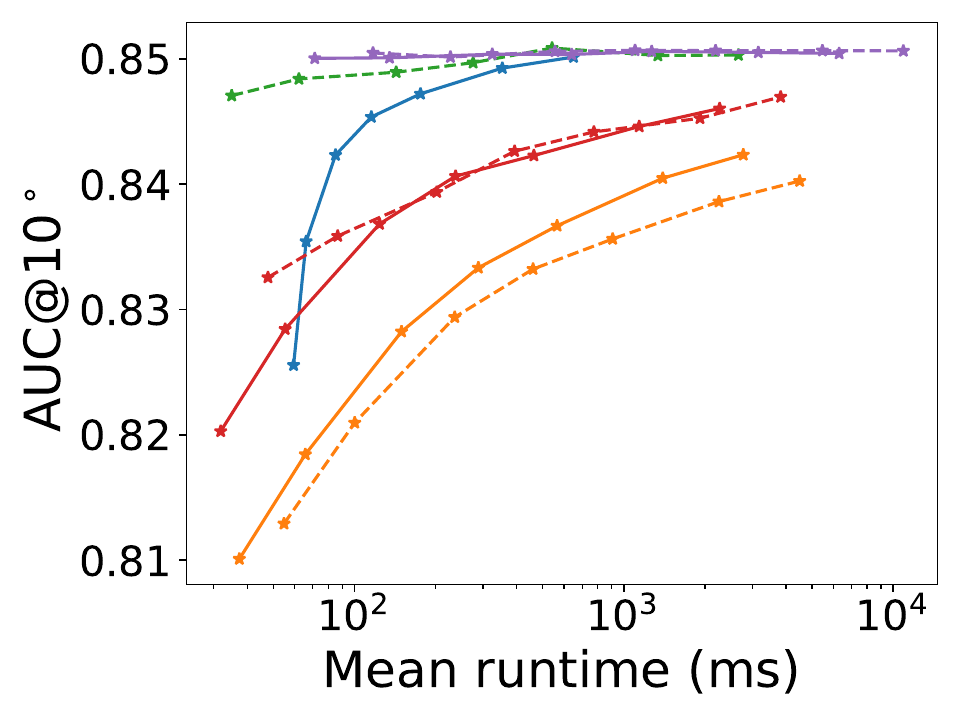}
    \caption{\textit{Notre Dame}}        
    \end{subfigure}
    \hfill
    \begin{subfigure}{0.24\linewidth}
    \includegraphics[width=\linewidth]{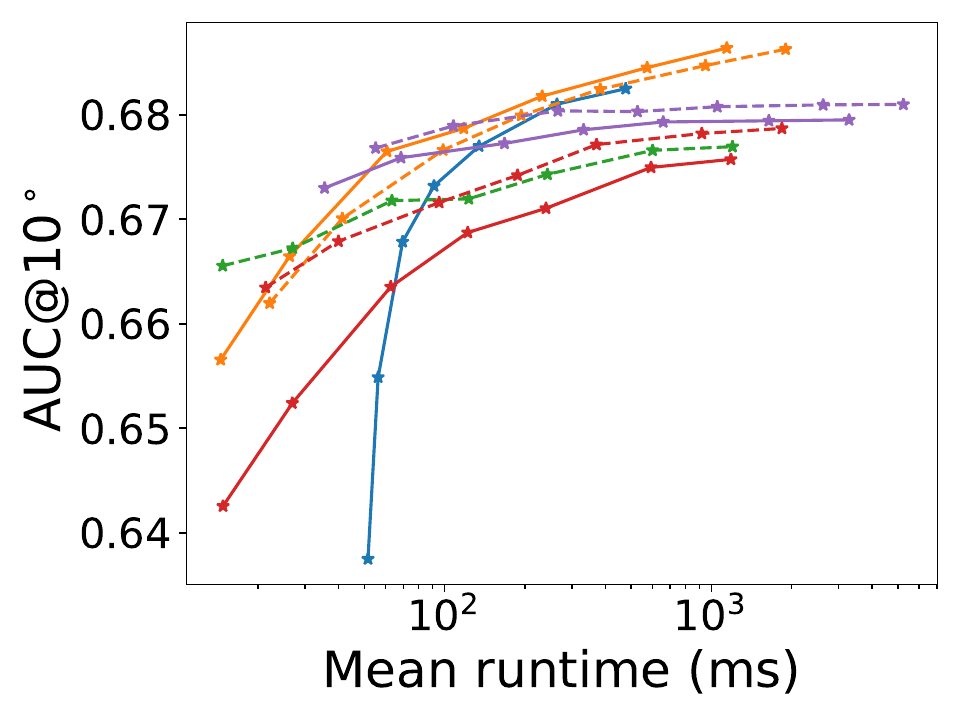}
    \caption{\textit{Palace of Westminster}} 
    \end{subfigure}
    \hfill        
    \begin{subfigure}{0.24\linewidth}
    \includegraphics[width=\linewidth]{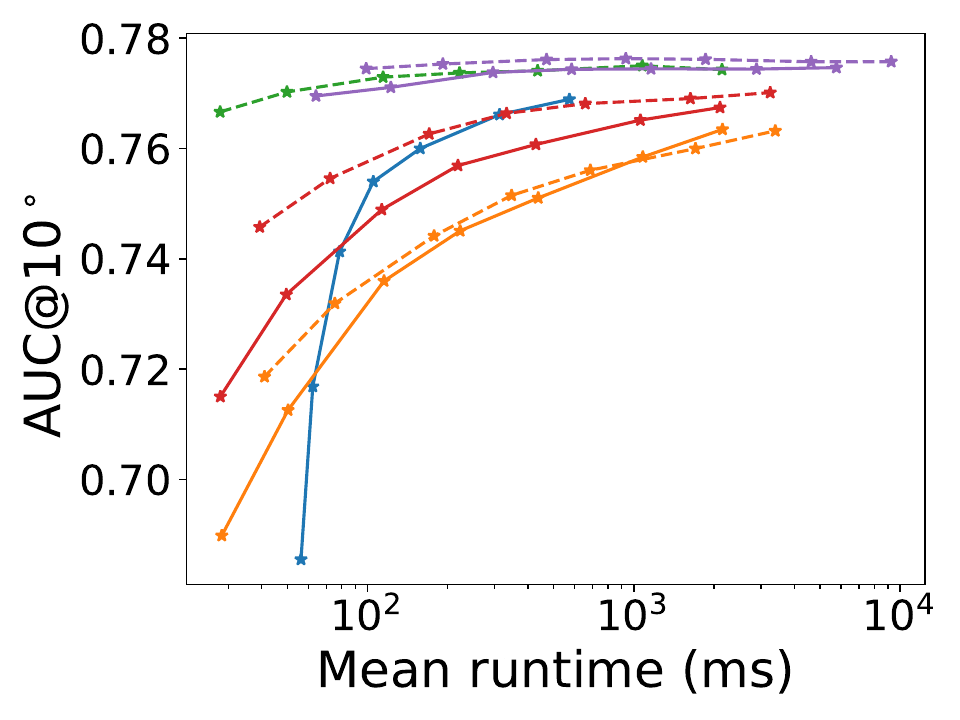}
    \caption{\textit{Pantheon Exterior}}        
    \end{subfigure}
    \hfill    
    \begin{subfigure}{0.24\linewidth}
    \includegraphics[width=\linewidth]{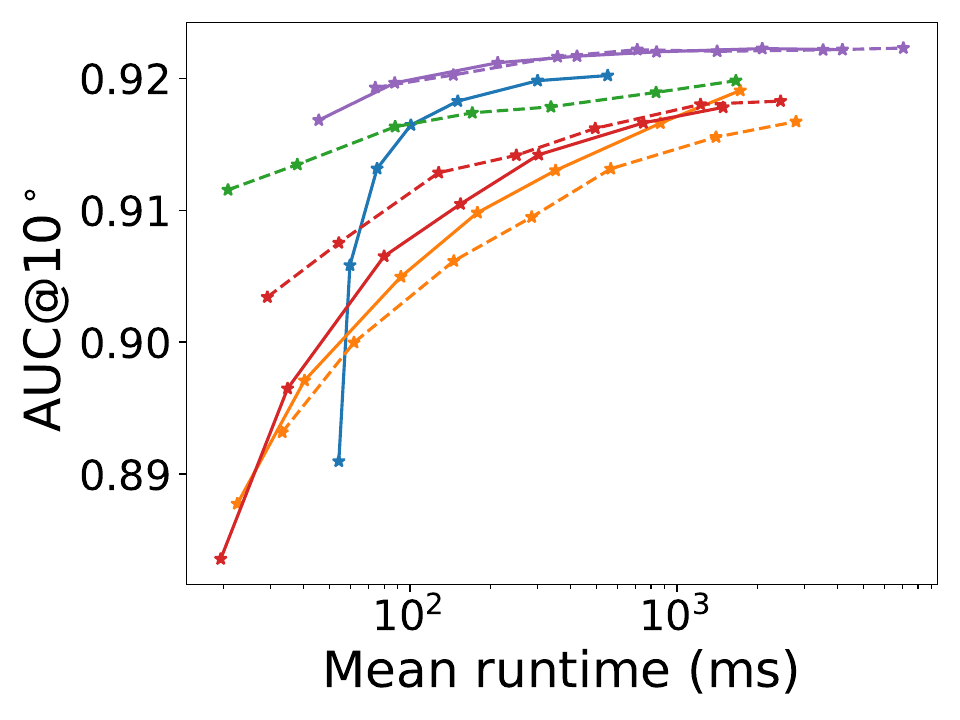}
    \caption{\textit{Sacre Coeur}}        
    \end{subfigure}

    \vspace{0.5ex}
    
    \begin{subfigure}{0.24\linewidth}
    \includegraphics[width=\linewidth]{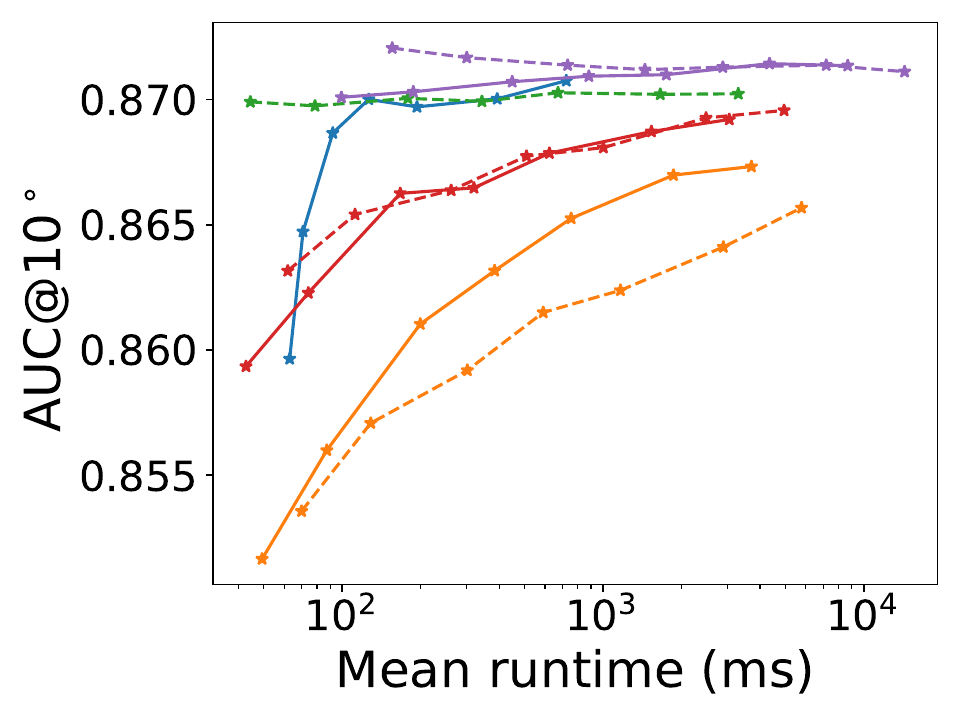}
    \caption{\textit{Reichstag}}        
    \end{subfigure}
    \begin{subfigure}{0.24\linewidth}
    \includegraphics[width=\linewidth]{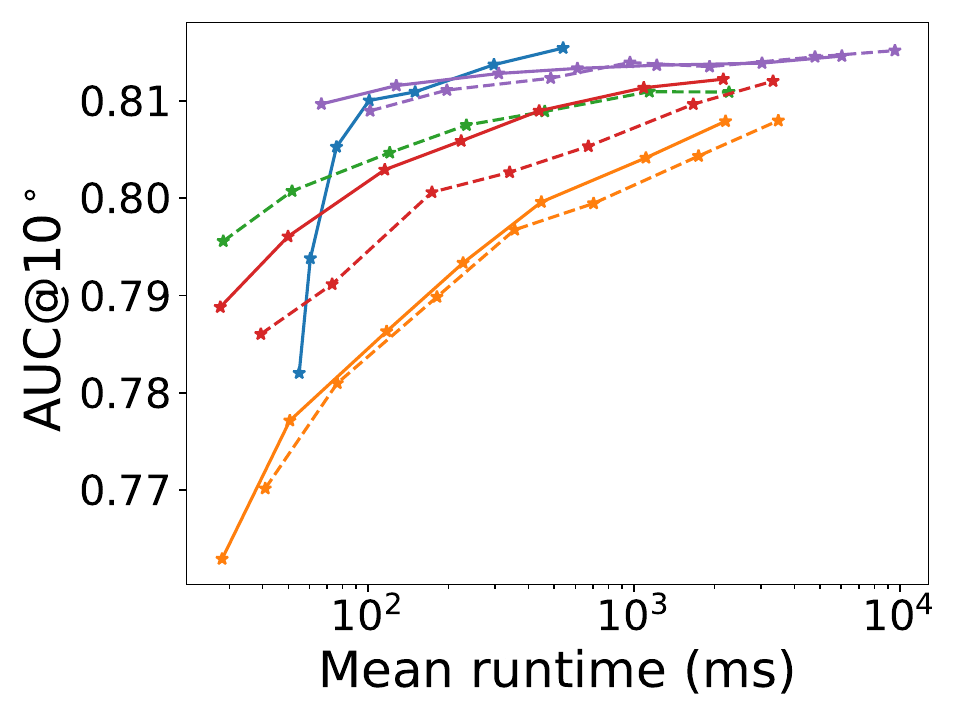}
    \caption{\textit{Taj Mahal}}        
    \end{subfigure}
    \hfill
    \begin{subfigure}{0.24\linewidth}
    \includegraphics[width=\linewidth]{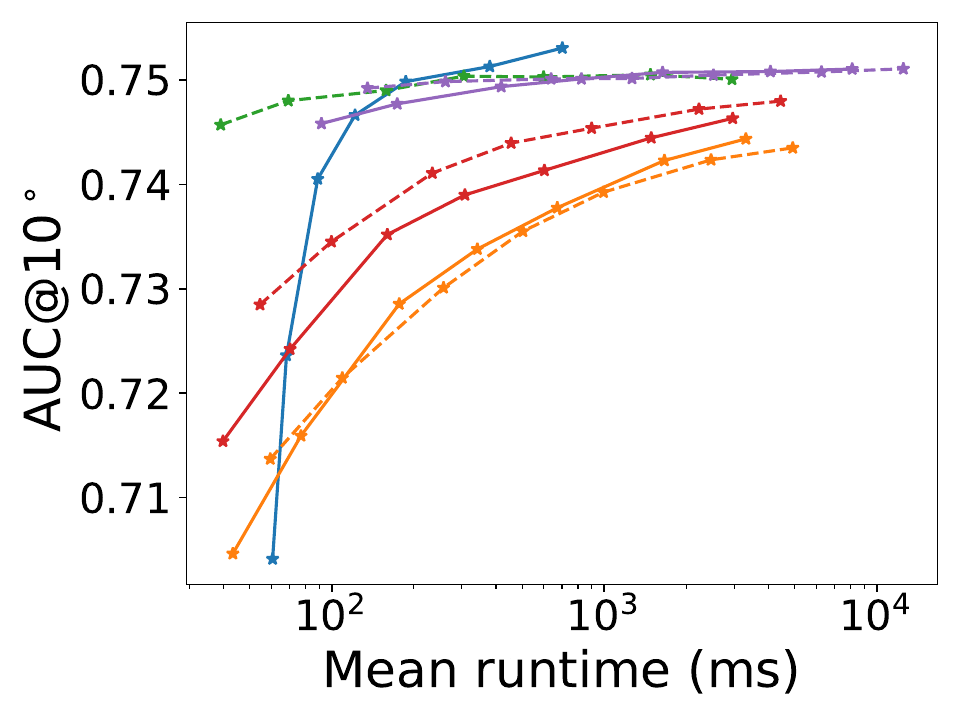}
    \caption{\textit{Temple Nara}}        
    \end{subfigure}
    \hfill
    \begin{subfigure}{0.24\linewidth}
    \includegraphics[width=\linewidth]{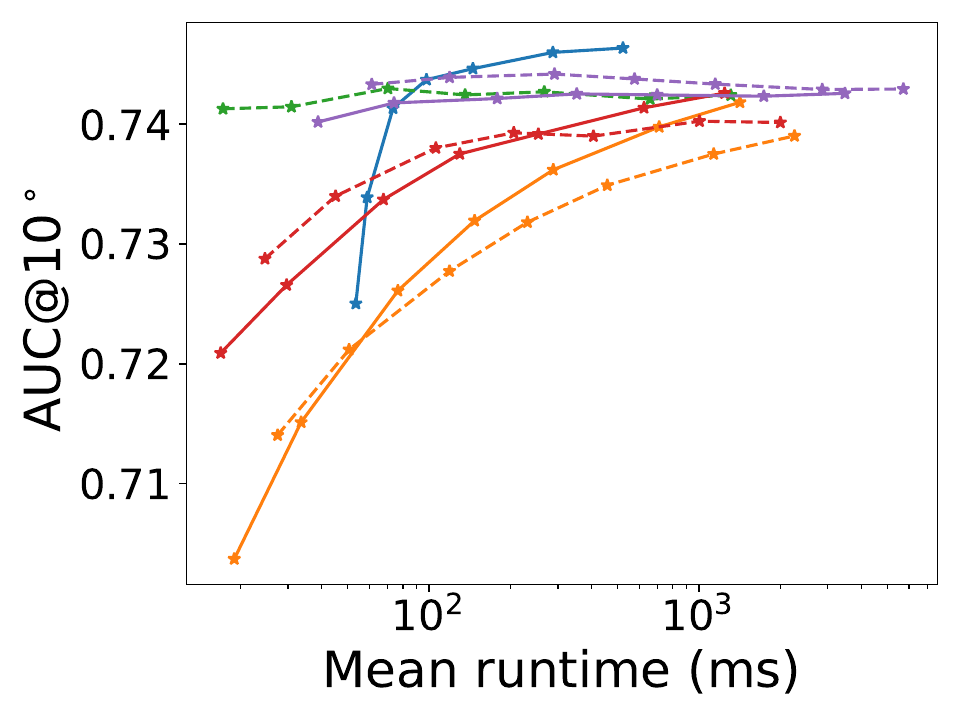}
    \caption{\textit{Trevi Fountain}}        
    \end{subfigure}
    
    \caption{Results for individual scenes from the Cambridge Landmarks~\cite{kendall2015cambridge} (a-d) and Phototourism~\cite{IMC2020} (e-p) scenes which were not presented in the main paper. We report the AUC@10$^\circ$ of the pose error and vary the number of Poselib RANSAC iterations ($\{100, 200, 500, 1000, 2000, 5000, 10000\}$). We use an epipolar threshold of 5px inside RANSAC. Runtimes are averaged over all image triplets.}
    \label{fig:poselib_graph_scenes}
\end{figure*}

\subsection{Results on individual scenes}
\label{sec:exp:details}
Fig.~4 in the main paper showed results jointly on all PhotoTourism~\cite{IMC2020} scenes (except \textit{St.~Peter's Square}), jointly on the 5 Cambridge Landmarks~\cite{kendall2015cambridge} scenes (except the Street scene, which is commonly not used due to issues with its ground truth), and Aachen Day-Night v1.1~\cite{zhang2021aachen}. 
It also showed results on one individual scene from~\cite{kendall2015cambridge}, \ie, the \textit{St. Mary's Church} scene.
In Fig.~\ref{fig:poselib_graph_scenes}, we provide results for the accuracy-speed trade-off evaluation for all remaining individual scenes of PhotoTourism~\cite{IMC2020} and Cambridge Landmarks~\cite{kendall2015cambridge}. 

As discussed in the main paper, 
the accuracy of the proposed approximate solvers is scene-dependent. 
This 
also applies to the state-of-the-art \sfhc solver~\cite{Hruby_cvpr2022}, since in this solver the scene needs to be similar enough to the training scenes for the MLP-based classifier to work well. 
The proposed \texttt{ENM} refitting suppresses to some extent the scene dependency of the proposed \sftm-based and \sfaf-based solvers. 
It can be seen that the proposed \sftmdRC solver consistently provides the best speed-accuracy trade-off both with and without \texttt{ENM} accross all scenes. \sfafRCENM provides a similar performance, 
typically beating \sfhc~\cite{Hruby_cvpr2022}. However,
it may perform worse for some specific scenes, \eg, \textit{Shop Facade} and  \textit{Palace of Westminster}.
In general, the results on individual scenes are consistent with the results from the main paper.

\begin{figure*}[!t]
    \centering
\begin{tikzpicture} 

        \begin{axis}[%
        hide axis, xmin=0,xmax=0,ymin=0,ymax=0,
        legend style={draw=white!15!white, 
        line width = 1pt,
        legend  columns =9, 
        /tikz/every even column/.append style={column sep=0.5cm},
        }
        ]
        
        \addlegendimage{Seaborn1}        \addlegendentry{\sfhc~\cite{Hruby_cvpr2022}};
        \addlegendimage{Seaborn2}
        \addlegendentry{\sft};
        \addlegendimage{Seaborn3}
        \addlegendentry{\sfafRC};
        \addlegendimage{Seaborn4}
        \addlegendentry{\sftmRC}; 
        \addlegendimage{Seaborn5}
        \addlegendentry{\sftmdRC};
        \end{axis}
    \end{tikzpicture}

   \begin{tabular}{cccc}
    \raisebox{0.4\height}{\rotatebox[origin=l]{90}{\small{Phototourism~\cite{IMC2020}}}} &    
    \includegraphics[width=0.3\textwidth]{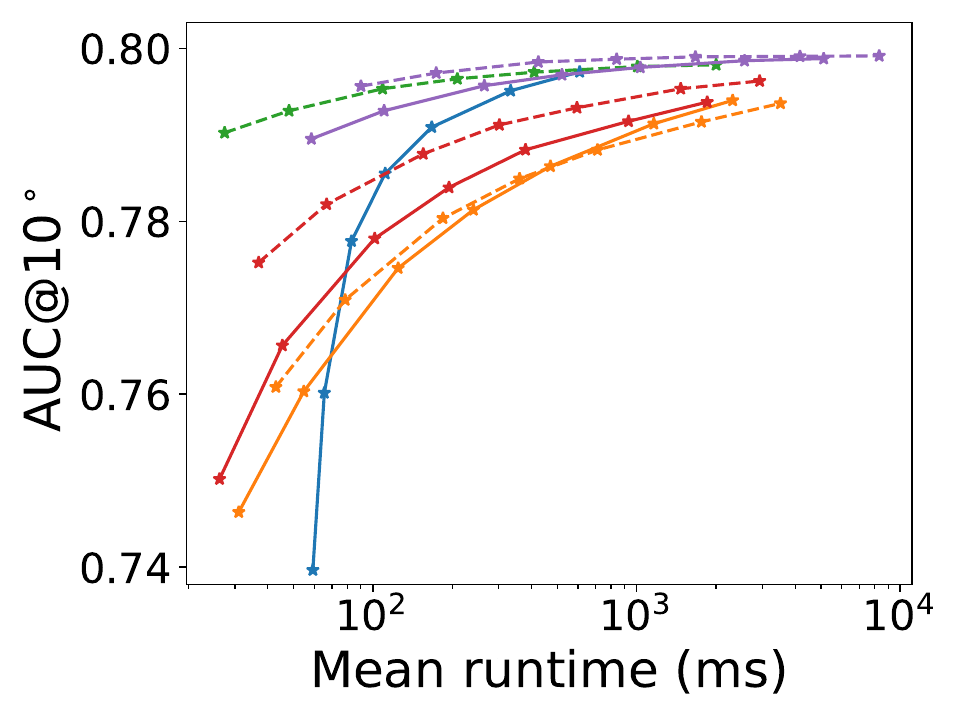}  
    &        
    \includegraphics[width=0.3\textwidth]{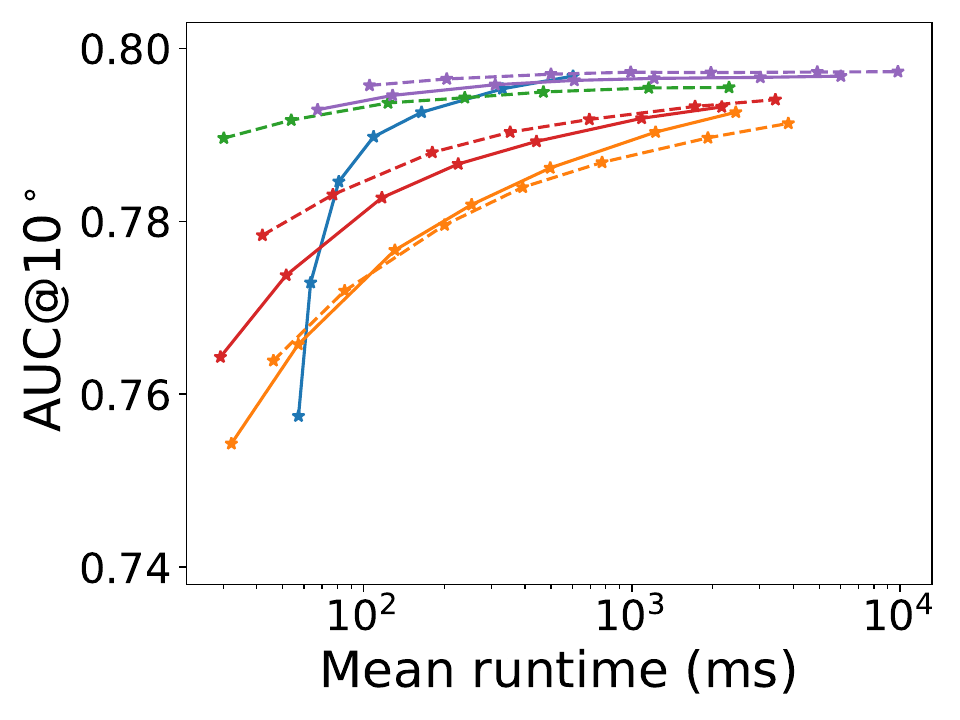}
    &
    \includegraphics[width=0.3\textwidth]{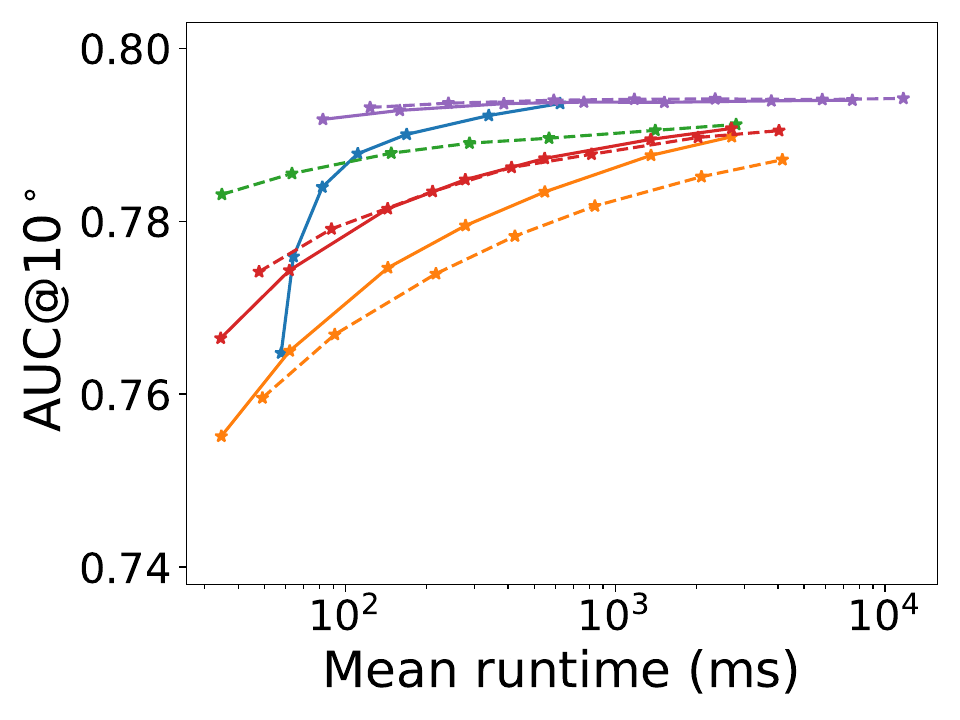}     \\

    \raisebox{0.55\height}{\rotatebox[origin=l]{90}{\small{Cambridge~\cite{kendall2015cambridge}}}} &
        \includegraphics[width=0.3\textwidth]{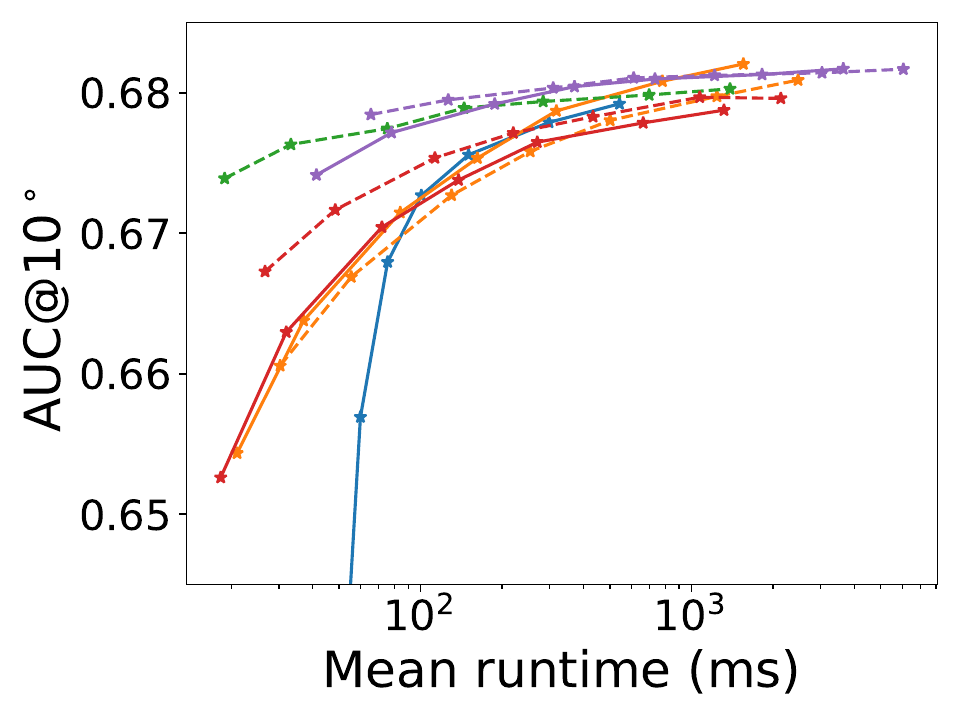}    
    &        
    \includegraphics[width=0.3\textwidth]{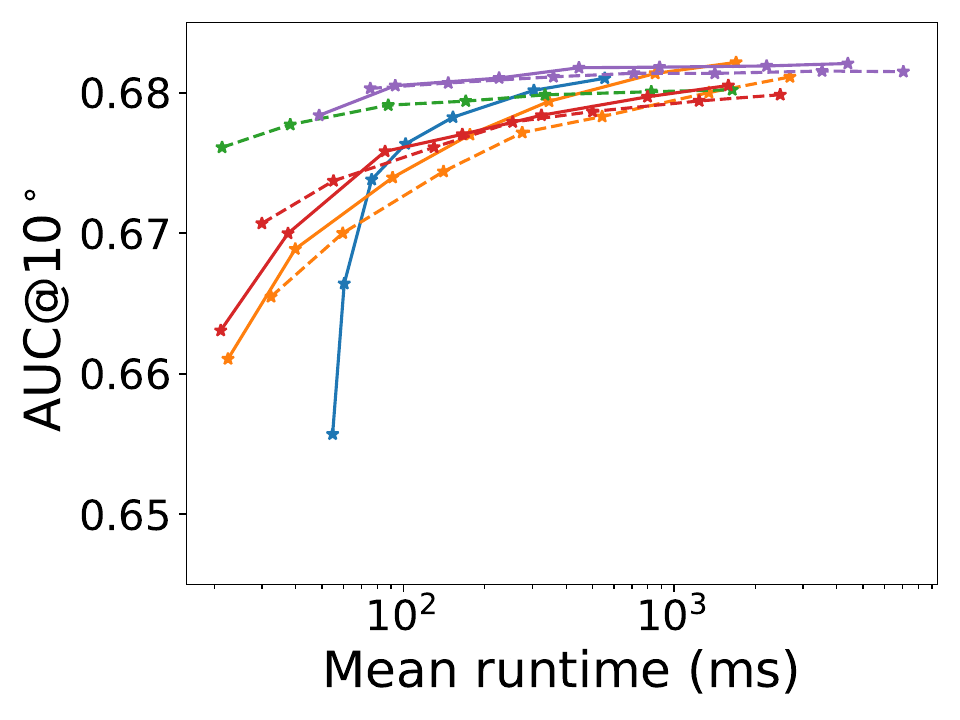}
    &
    \includegraphics[width=0.3\textwidth]{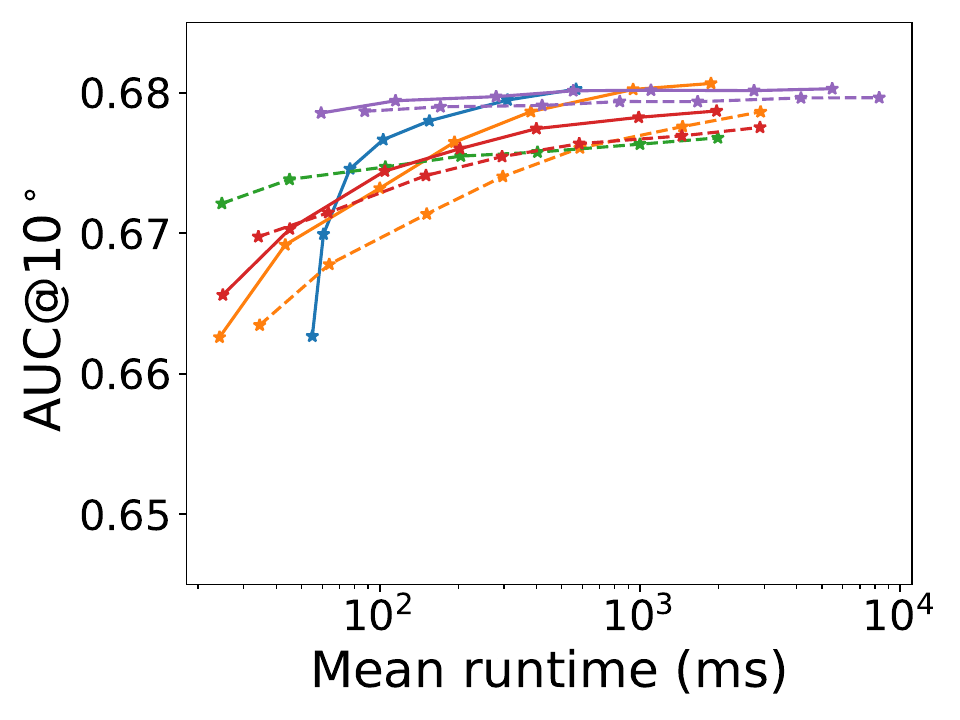} \\

    \raisebox{0.80\height}{\rotatebox[origin=l]{90}{\small{Aachen~\cite{zhang2021aachen}}}} &
    \includegraphics[width=0.3\textwidth]{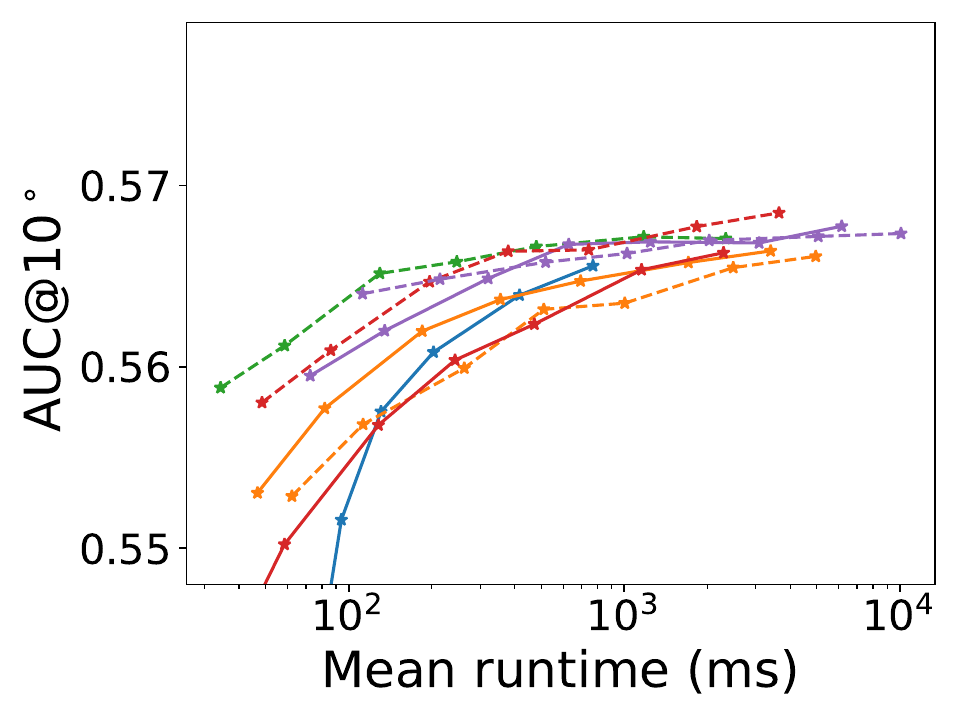}    
    &        
    \includegraphics[width=0.3\textwidth]{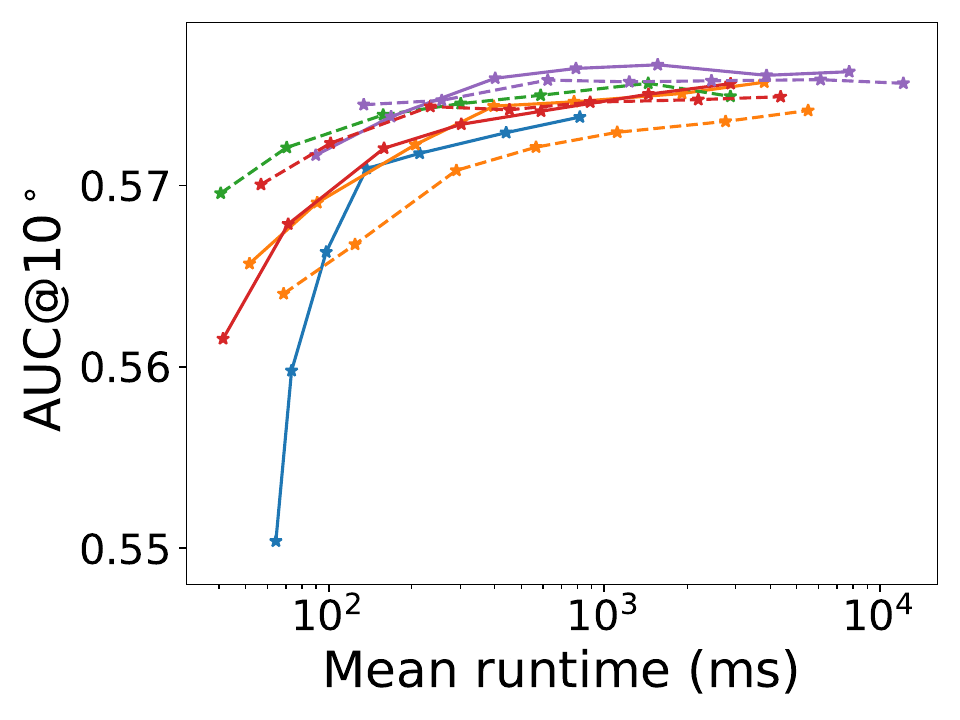}
    &
    \includegraphics[width=0.3\textwidth]{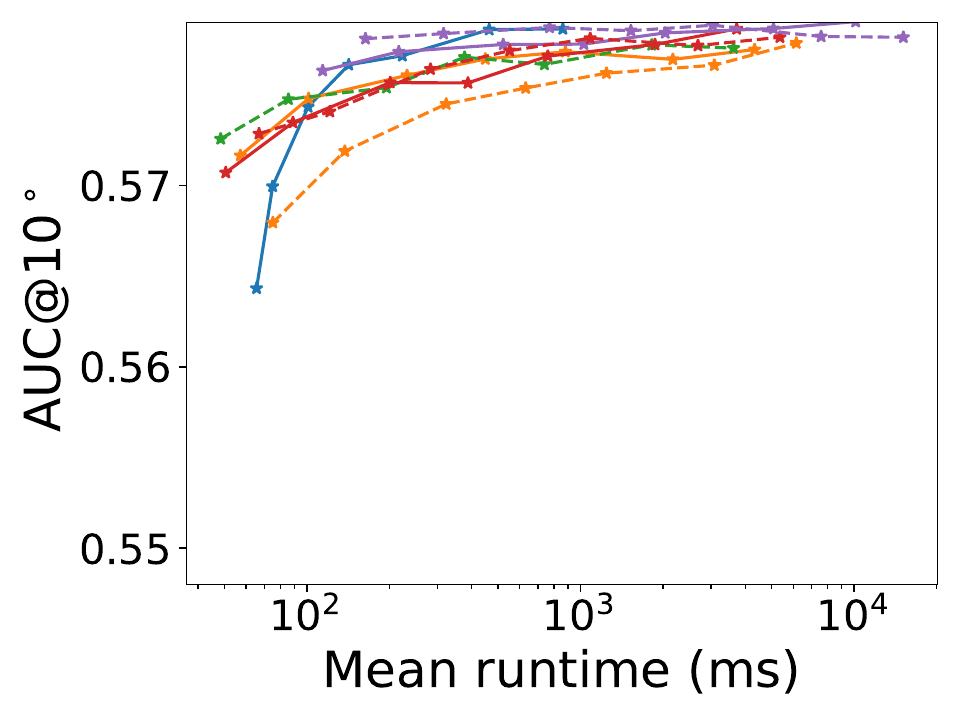}  \\

    ~&$t = 3\text{px}$&$t = 5\text{px}$&$t = 10\text{px}$
     
    \end{tabular}
    
    \caption{Speed-accuracy trade-off on all scenes from Phototourism~\cite{IMC2020}, except \textit{St.~Peter's Square}, 5 scenes from  Cambridge Landmarks~\cite{kendall2015cambridge}, and the Aachen Day-Night v1.1 dataset~\cite{zhang2021aachen}. 
    We report the AUC@10$^\circ$ of the pose error and vary the number of Poselib RANSAC iterations ($\{100, 200, 500, 1000, 2000, 5000, 10000\}$) for different maximum epipolar thresholds ($t$). Runtimes are averaged over all image triplets. }
    \label{fig:threshold_sensitivity}
\end{figure*}


\subsection{RANSAC threshold sensitivity}
\label{sec:threshold}

In Fig.~\ref{fig:poselib_ablation_threshold} (bottom), we provided experiments showing how the performance of the selected methods changes when we vary the RANSAC epipolar threshold. 
In Fig.~\ref{fig:threshold_sensitivity} we provide more extensive results comparing the methods with different thresholds on all three datsets. 

Similar to the results presented in the main paper, 
\sftmdRC in both variants (with and without \texttt{ENM}) shows consistently good performance even when using a different threshold in RANSAC. In contrast, \sfafRCENM performs worse than \sftmdRC when considering a higher epipolar threshold in RANSAC.

\subsection{Solver run-times}
\label{sec:exp:timing}
In this section, we present run-times of the proposed solvers as well as the state-of-the-art solvers for the relative pose problem of three calibrated cameras.
While the main paper reports run-time results for full RANSAC-based estimation, we now report the run-times of the individual solvers outside of RANSAC. 
To measure the run-times of the solvers\footnote{Note that for \sfhc solver, in Tab~\ref{tab:calib_time}, the time needed to load the weights of the network (or any other required data) is not added to the runtime of the solver. The data are loaded once per RANSAC, and the loading takes on average 45ms.}, we calculated the average run-time of each solver on more than 50k instances of the \textit{Sacre Coeur} scene of the PhotoTourism dataset~\cite{snavely2006photo}.
The run-times are reported in Table~\ref{tab:calib_time}. The experiments were performed on an Intel(R) Core(TM) i9-10900X CPU @ 3.70GHz.
In general, the implementations of all proposed solvers are not optimized for speed, and we still see room for speeding them up. 

\begin{table}[t!]
 \begin{center}
 \resizebox{1\columnwidth}{!}{
 \begin{tabular}{|l| c | c | c | c | c |}
    \hline
    & \sft & \sfhc & \sftm & \sftmd & \sfaf \\ 
    \hline
    Time ($\mu$s) & 77.90 & 66.06 & 83.92 & 218.71 & 61.12 \\
    \hline
\end{tabular}
 }
\end{center}
\caption{The average run-time, averaged over more than 50k instances of the \textit{Sacre Coeur} scene of the PhotoTourism dataset~\cite{snavely2006photo}, of the solvers for the 4p3v problem.}
\label{tab:calib_time}
\end{table}

\subsection{Outlier experiments}
\label{sec:exp:outlier}

\begin{figure*}
    \centering

\resizebox{1.0\linewidth}{!}{
\begin{tikzpicture} 

        \begin{axis}[%
        hide axis, xmin=0,xmax=0,ymin=0,ymax=0,
        legend style={draw=white!15!white, 
        line width = 1pt,
        legend  columns =9, 
        /tikz/every even column/.append style={column sep=0.5cm},
        }
        ]
        
        \addlegendimage{Seaborn1}        \addlegendentry{\sfhc~\cite{Hruby_cvpr2022}};
        \addlegendimage{Seaborn2}
        \addlegendentry{\sft};
        \addlegendimage{Seaborn3}
        \addlegendentry{\sfafRC};
        \addlegendimage{Seaborn4}
        \addlegendentry{\sftmRC}; 
        \addlegendimage{Seaborn5}
        \addlegendentry{\sftmdRC};
        \addlegendimage{black!30}
        \addlegendentry{w/o \texttt{ENM}};
        \addlegendimage{black!30,dash pattern=on 2pt off 1pt on 2pt off 1pt}
        \addlegendentry{w/ \texttt{ENM}};
        
        \end{axis}
    \end{tikzpicture}}
   
    \begin{subfigure}{0.3\textwidth}
    \includegraphics[width=\textwidth]{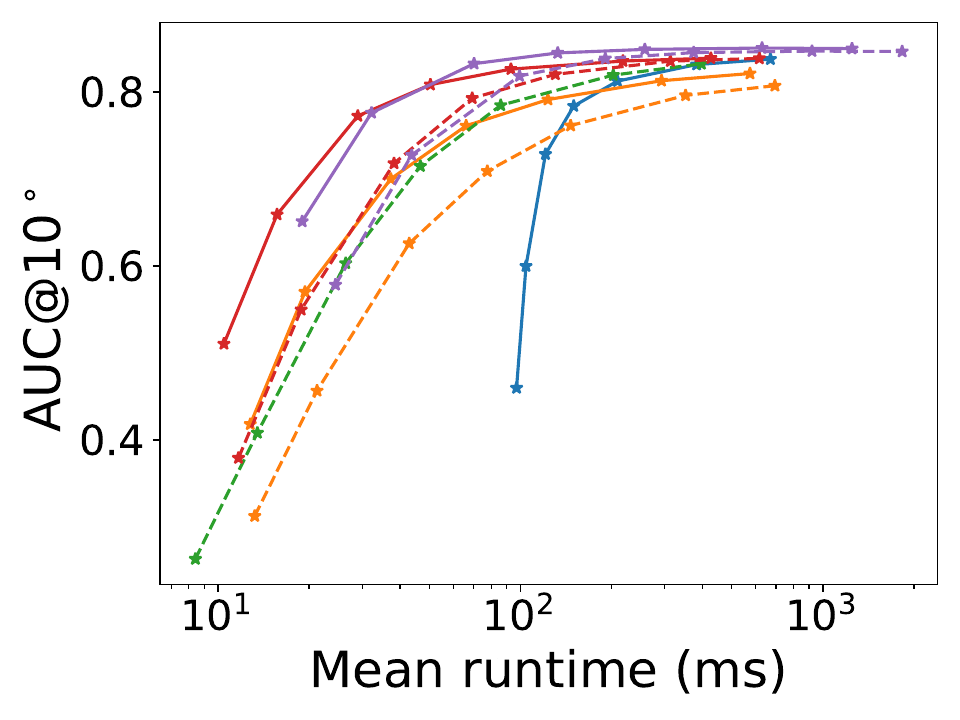}
    \caption{40\% Inlier Ratio}
    \end{subfigure}
    \hfill    
    \begin{subfigure}{0.3\textwidth}
    \includegraphics[width=\textwidth]{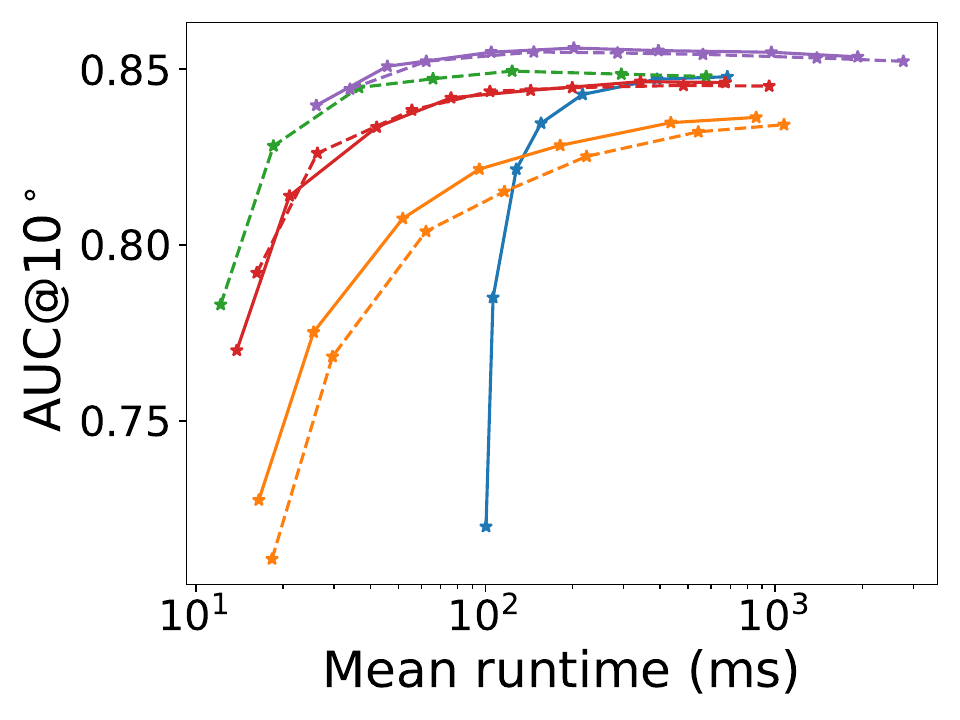}
    \caption{60\% Inlier Ratio}
    \end{subfigure}  
    \hfill
    \begin{subfigure}{0.3\textwidth}
    \includegraphics[width=\textwidth]{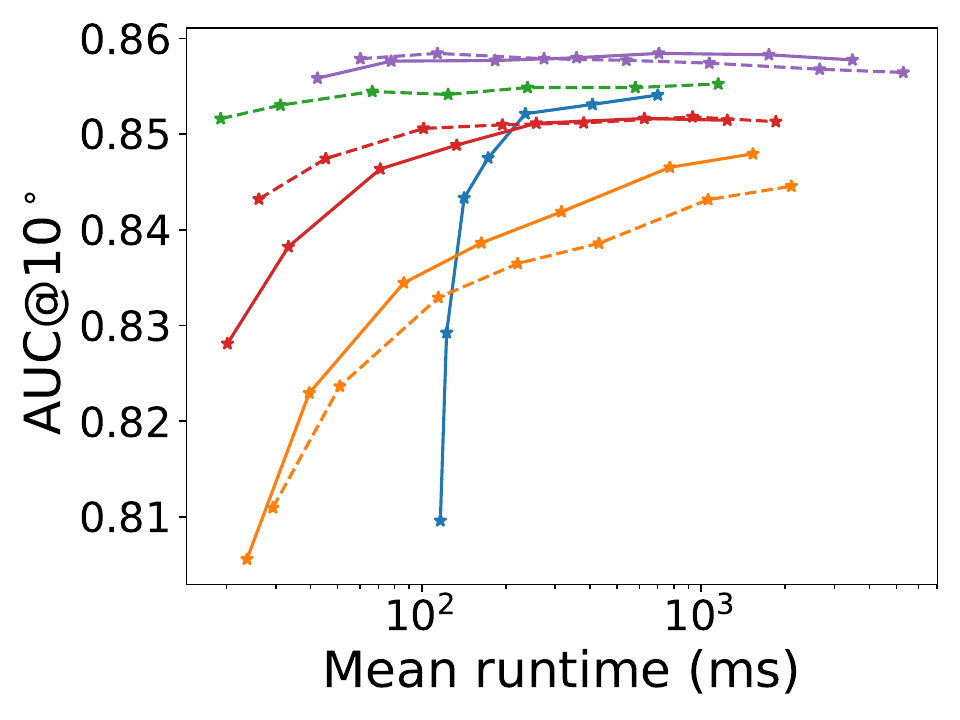}    
    \caption{80\% Inlier Ratio}
    \end{subfigure}
    
    \caption{Speed-accuracy trade-off on the \textit{Notre Dame} scene from Phototourism~\cite{IMC2020}. We perform semi-synthetic experiments where we add random outlier correspondences to modify the inlier ratios.
    }
    \label{fig:outliers}
\end{figure*}

To show how the different solvers perform even under varying inlier ratios, we perform a semi-synthetic experiment. 
We use the \textit{Notre Dame} scene from PhotoTourism~\cite{IMC2020}. We keep all inlier triplets \wrt a 5px epipolar threshold using the ground truth poses. 
We add additional synthetic outlier correspondences by generating random points in all three views. 
This allows us to study how the different methods perform when the inlier ratio changes. 
The results are shown in Fig.~\ref{fig:outliers}. 
As expected, the performance of all solvers decreases with lower inlier ratios. 
We also observe that \sftmdRC performs well even with a low inlier ratio. 
In contrast, the relative performance of \sfafRCENM worsens with a decreased inlier ratio. However, we note that even with an inlier ratio of 40\%, it still results in performance comparable to the baseline \sft solver. This suggests that a high inlier ratio is not necessary for the \texttt{ENM} to work well in conjunction with the solver \sfaf.

\begin{table*}[t]
    \centering
    \resizebox{1.0\linewidth}{!}{
\begin{tabular}{ l | c c | c c c | c | c c | c c c | c| c c | c c c | c}
    \toprule
    \multicolumn{1}{c |}{~} & \multicolumn{6}{c |}{Phototourism~\cite{IMC2020}} & \multicolumn{6}{c |}{Cambridge Landmarks~\cite{kendall2015cambridge}} & \multicolumn{6}{c}{Aachen Day-Night v1.1~\cite{zhang2021aachen}} \\
    \toprule
    Estimator & AVG $(^\circ)$ $\downarrow$ & MED $(^\circ)$ $\downarrow$ & AUC@5 $\uparrow$ & @10 $\uparrow$ & @20 $\uparrow$ & Runtime (ms) $\downarrow$& AVG $(^\circ)$ $\downarrow$ & MED $(^\circ)$ $\downarrow$ & AUC@5 $\uparrow$ & @10 $\uparrow$ & @20 $\uparrow$ & Runtime (ms) $\downarrow$& AVG $(^\circ)$ $\downarrow$ & MED $(^\circ)$ $\downarrow$ & AUC@5 $\uparrow$ & @10 $\uparrow$ & @20 $\uparrow$ & Runtime (ms) $\downarrow$\\
    \midrule

\sfhc~\cite{Hruby_cvpr2022} & 10.28 & \phantom{1}2.81 & 46.97 & 61.37 & 73.16 & \phantom{1}64.10 & 13.95 & \phantom{1}4.58 & 29.54 & 48.68 & 65.12 & 60.11 & 23.46 & \phantom{1}6.58 & 28.56 & 41.20 & 53.12 & \phantom{1}67.26 \\
\midrule\sft & \phantom{1}9.02 & \phantom{1}2.81 & 47.06 & 61.89 & 74.12 & \phantom{1}33.77 & 12.39 & \phantom{1}4.57 & 29.56 & 49.03 & 65.94 & 24.04 & 21.17 & \phantom{1}6.33 & 29.07 & 41.90 & 54.13 & \phantom{1}53.34 \\
\sftENM & \phantom{1}8.77 & \phantom{1}2.73 & 47.93 & 62.68 & 74.73 & \phantom{1}48.79 & 12.00 & \phantom{1}4.52 & 29.76 & 49.35 & 66.25 & 34.82 & 21.39 & \phantom{1}6.42 & 28.91 & 41.73 & 53.78 & \phantom{1}71.61 \\
\midrule \sfaf & 58.51 & 46.32 & 16.16 & 21.93 & 27.71 & \phantom{1}16.58 & 59.75 & 43.50 & 13.59 & 22.93 & 31.44 & 13.33 & 53.53 & 41.38 & 17.05 & 23.62 & 30.10 & \phantom{1}32.04 \\
\sfafENM & \phantom{1}8.44 & \phantom{1}2.60 & 49.46 & 64.10 & 75.85 & \phantom{1}40.45 & 11.86 & \phantom{1}4.46 & 30.17 & 49.89 & 66.74 & 28.35 & 21.00 & \phantom{1}6.23 & 29.08 & 42.05 & 54.32 & \phantom{1}62.44 \\
\sfafR & 53.59 & 38.24 & 18.31 & 24.91 & 31.34 & \phantom{1}\underline{16.32} & 54.17 & 24.26 & 16.60 & 27.32 & 36.63 & \underline{12.48} & 51.81 & 38.69 & 17.55 & 24.35 & 31.05 & \phantom{1}\underline{29.56} \\
\sfafRC & 56.31 & 43.55 & 16.89 & 23.07 & 29.15 & \phantom{1}\textbf{11.10} & 55.99 & 30.64 & 15.71 & 26.07 & 35.10 & \phantom{1}\textbf{9.35} & 54.12 & 42.69 & 16.64 & 23.04 & 29.50 & \phantom{1}\textbf{19.75} \\
\sfafRCENM & \phantom{1}8.45 & \phantom{1}2.59 & 49.59 & 64.22 & 75.92 & \phantom{1}32.37 & 11.78 & \phantom{1}4.41 & 30.41 & 50.12 & 66.93 & 23.43 & 21.11 & \phantom{1}6.29 & 29.02 & 41.97 & 54.18 & \phantom{1}42.36 \\
\midrule\sftm & \phantom{1}9.98 & \phantom{1}3.01 & 45.16 & 60.02 & 72.55 & \phantom{1}35.18 & 13.63 & \phantom{1}4.73 & 28.66 & 47.83 & 64.70 & 25.15 & 23.44 & \phantom{1}6.82 & 27.90 & 40.53 & 52.69 & \phantom{1}54.89 \\
\sftmENM & \phantom{1}8.86 & \phantom{1}2.75 & 47.68 & 62.54 & 74.67 & \phantom{1}48.83 & 12.10 & \phantom{1}4.52 & 29.80 & 49.31 & 66.17 & 34.93 & 21.51 & \phantom{1}6.40 & 28.91 & 41.73 & 53.83 & \phantom{1}70.87 \\
\sftmR & \phantom{1}9.24 & \phantom{1}2.74 & 47.74 & 62.44 & 74.47 & \phantom{1}41.52 & 12.76 & \phantom{1}4.53 & 29.77 & 49.27 & 66.12 & 30.76 & 22.06 & \phantom{1}6.39 & 28.92 & 41.69 & 53.70 & \phantom{1}60.40 \\
\sftmRC & \phantom{1}9.19 & \phantom{1}2.71 & 48.10 & 62.73 & 74.69 & \phantom{1}30.88 & 12.84 & \phantom{1}4.53 & 29.74 & 49.29 & 66.11 & 22.72 & 22.06 & \phantom{1}6.42 & 28.70 & 41.53 & 53.62 & \phantom{1}42.38 \\
\sftmRCENM & \phantom{1}8.52 & \phantom{1}2.62 & 49.16 & 63.82 & 75.60 & \phantom{1}44.51 & 11.79 & \phantom{1}4.44 & 30.27 & 49.91 & 66.72 & 32.48 & 20.93 & \phantom{1}6.21 & 29.16 & 42.16 & 54.46 & \phantom{1}58.08 \\
\midrule\sftmd & \phantom{1}9.28 & \phantom{1}2.94 & 45.92 & 61.05 & 73.66 & \phantom{1}83.70 & 12.68 & \phantom{1}4.59 & 29.35 & 48.88 & 65.89 & 59.23 & 22.09 & \phantom{1}6.42 & 28.60 & 41.55 & 53.76 & 125.02 \\
\sftmdENM & \phantom{1}8.44 & \phantom{1}2.73 & 48.01 & 63.02 & 75.20 & 125.66 & \underline{11.61} & \phantom{1}4.47 & 30.04 & 49.75 & 66.68 & 89.05 & \underline{20.80} & \phantom{1}6.22 & 29.26 & 42.16 & 54.32 & 175.54 \\
\sftmdR & \phantom{1}\underline{8.32} & \phantom{1}\underline{2.58} & \underline{49.79} & \underline{64.53} & \underline{76.30} & 100.61 & 11.93 & \phantom{1}\underline{4.40} & \underline{30.47} & \underline{50.35} & \underline{67.22} & 73.94 & 21.15 & \phantom{1}\underline{6.12} & \underline{29.35} & \underline{42.31} & \underline{54.53} & 138.53 \\
\sftmdRC & \phantom{1}8.39 & \phantom{1}2.58 & 49.69 & 64.41 & 76.18 & \phantom{1}71.73 & 12.06 & \phantom{1}4.41 & 30.42 & 50.25 & 67.08 & 52.84 & 21.38 & \phantom{1}6.15 & 29.24 & 42.22 & 54.43 & \phantom{1}92.89 \\
\sftmdRCENM & \phantom{1}\textbf{7.99} & \phantom{1}\textbf{2.56} & \textbf{49.93} & \textbf{64.67} & \textbf{76.45} & 112.60 & \textbf{11.36} & \phantom{1}\textbf{4.39} & \textbf{30.54} & \textbf{50.40} & \textbf{67.30} & 81.98 & \textbf{20.64} & \phantom{1}\textbf{6.08} & \textbf{29.40} & \textbf{42.48} & \textbf{54.67} & 139.19 \\
\midrule
    \end{tabular}}
    \caption{Experiments with the alternative evaluation measure described in Sec.~\ref{sec:exp:measure}. Results for different solvers implemented in the PoseLib framework~\cite{PoseLib} on all scenes from the PhotoTourism~\cite{IMC2020}, 5 scenes from the Cambridge Landmarks~\cite{kendall2015cambridge}, and the Aachen Day-Night v1.1~\cite{zhang2021aachen} datasets. 
    We mark the \textbf{best} and \underline{second best} results. Reported runtimes are for the whole RANSAC.}
    \label{tab:max_error}
\end{table*}

\subsection{Alternative evaluation measure}
\label{sec:exp:measure}
For the evaluation in the main paper, we defined the pose error as $\text{max}\left(0.5 (\M R_{err}^{12} + \M R_{err}^{13}), 0.5 (\V t_{err}^{12} + \V t_{err}^{13})\right)$, where $\M R_{err}^{ij}$ and $\V t_{err}^{ij}$ are the angular errors of rotation and translation (both in degrees) for camera pair $ij$. 
The 4p3v problem also includes the estimation of $\M R_{23}$ and $\V t_{23}$ since the relative scale of $\V t_{12}$ and $\V t_{13}$ is recovered. We therefore also present results for the pose error defined as 
\begin{equation}
P_{err} = \text{max} \left(\M R_{err}^{12}, \M R_{err}^{13}, \M R_{err}^{23}, \V t_{err}^{12}, \V t_{err}^{13}, \V t_{err}^{23}\right) \enspace.    
\label{eq:max_error}
\end{equation}
The results equivalent to Tab.~2 from the main paper using this pose error definition are presented in Tab.~\ref{tab:max_error}. 
A speed-accuracy comparison equivalent to Fig.~4 in the main paper is presented in Fig.~\ref{fig:max_error}. The overall ranking of the methods remains the same under both the metric used in the main paper and the alternative described in this section. 

\subsection{Comparison with \cite{DBLP:journals/ijcv/NisterS06}}
{
The authors of
~\cite{DBLP:journals/ijcv/NisterS06} kindly shared their source code with us.   
Unfortunately, we were not able to run the part of the code that samples epipole candidates from a 10-degree polynomial curve (appropriately sampling the curve is hard as the epipole can be arbitrarily far from the image center). At the same time, the authors were also not able to run it. 

Based on the working parts of the code, we tested an oracle version of~\cite{DBLP:journals/ijcv/NisterS06}, where instead of sampling the 10-degree polynomial curve, the oracle gives us the correct epipole. Given a sample close to the correct epipole, \cite{DBLP:journals/ijcv/NisterS06} performs comparable to our M-based solvers. In practice it is hard to find good samples (the epipole can be arbitrarily far from the image center). \cite{DBLP:journals/ijcv/NisterS06} report using 40-1,000 samples with additional local optimization for robust estimation. 
Even then, \cite{DBLP:journals/ijcv/NisterS06} show that this approach performs worse than \texttt{5pt+P3P} on synthetic data. In contrast, 
two additional samples in our $\delta$-based 
solvers 
already lead to better 
accuracy than \texttt{5pt+P3P}.}

\begin{figure*}
    \centering

\resizebox{1.0\linewidth}{!}{
\begin{tikzpicture} 

        \begin{axis}[%
        hide axis, xmin=0,xmax=0,ymin=0,ymax=0,
        legend style={draw=white!15!white, 
        line width = 1pt,
        legend  columns =9, 
        /tikz/every even column/.append style={column sep=0.5cm},
        }
        ]
        
        \addlegendimage{Seaborn1}        \addlegendentry{\sfhc~\cite{Hruby_cvpr2022}};
        \addlegendimage{Seaborn2}
        \addlegendentry{\sft};
        \addlegendimage{Seaborn3}
        \addlegendentry{\sfafRC};
        \addlegendimage{Seaborn4}
        \addlegendentry{\sftmRC}; 
        \addlegendimage{Seaborn5}
        \addlegendentry{\sftmdRC};
        \addlegendimage{black!30}
        \addlegendentry{w/o \texttt{ENM}};
        \addlegendimage{black!30,dash pattern=on 2pt off 1pt on 2pt off 1pt}
        \addlegendentry{w/ \texttt{ENM}};
        
        \end{axis}
    \end{tikzpicture}}
   
    \begin{subfigure}{0.3\textwidth}
    \includegraphics[width=\textwidth]{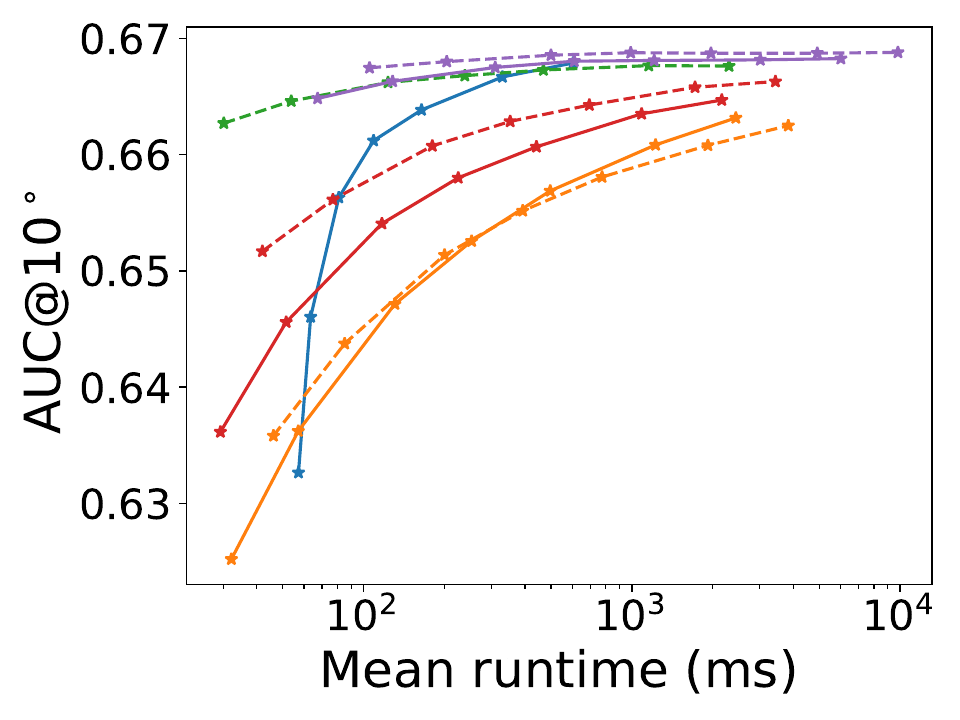}
    \caption{Phototourism~\cite{IMC2020}}
    \end{subfigure}
    \hfill    
    \begin{subfigure}{0.3\textwidth}
    \includegraphics[width=\textwidth]{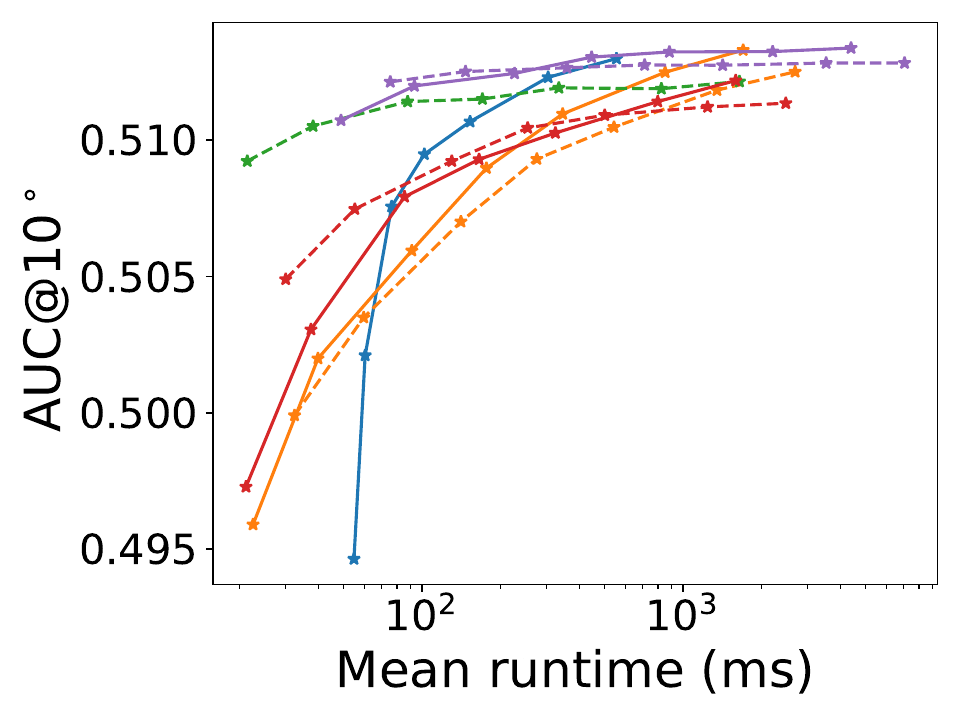}
    \caption{Cambridge Landmarks~\cite{kendall2015cambridge}}
    \end{subfigure}  
    \hfill
    \begin{subfigure}{0.3\textwidth}
    \includegraphics[width=\textwidth]{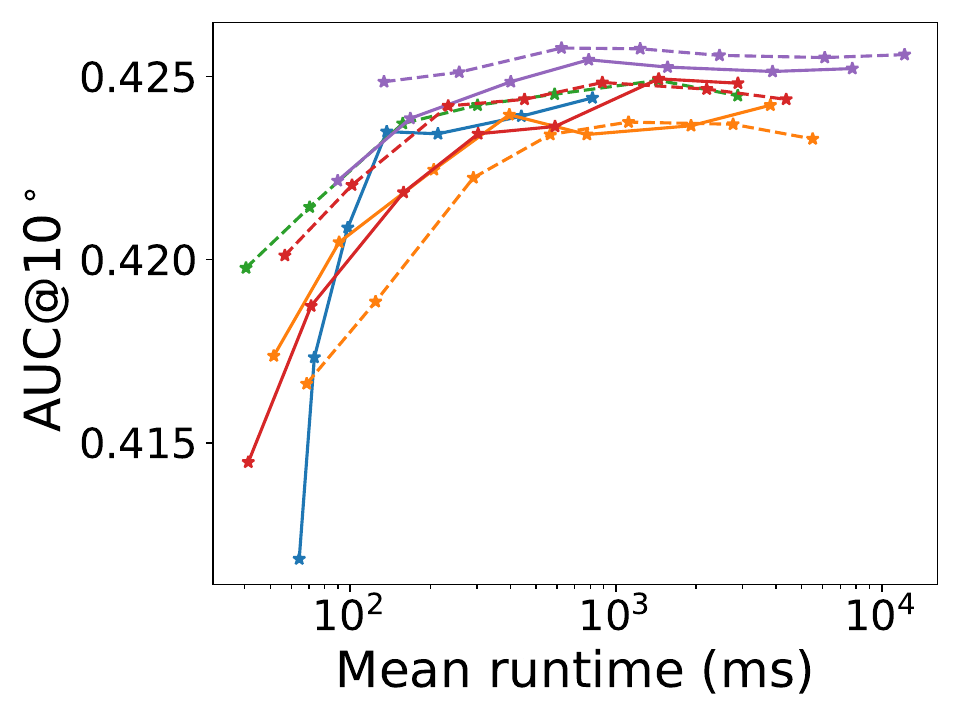}    
    \caption{\textit{Aachen Day-Night v1.1}~\cite{zhang2021aachen}}
    \end{subfigure}
    
    \caption{Experiments with the alternative evaluation measure described in Sec.~\ref{sec:exp:measure}: Speed-accuracy trade-off on (a) all scenes from PhotoTourism~\cite{IMC2020}, except \textit{St.~Peter's Square}, (b) 5  Cambridge Landmarks~\cite{kendall2015cambridge} scenes, and (c) the Aachen Day-Night v1.1~\cite{zhang2021aachen} dataset.
    We report the AUC@10$^\circ$ using the alternative definition of the pose error \eqref{eq:max_error}. We vary the number of Poselib RANSAC iterations ($\{100, 200, 500, 1000, 2000, 5000, 10000\}$). We use an epipolar threshold of 5px in RANSAC. Runtimes are averaged over all image triplets.}
    \label{fig:max_error}
\end{figure*}

\begin{table}[t]
    \centering
    \resizebox{1.0\linewidth}{!}{
\begin{tabular}{ l | c c | c c c | c}
    \toprule
    \multicolumn{7}{c}{Phototourism~\cite{IMC2020}} \\
    \midrule
    Estimator & AVG $(^\circ)$ $\downarrow$ & MED $(^\circ)$ $\downarrow$ & AUC@5 $\uparrow$ & @10 $\uparrow$ & @20 $\uparrow$ & Runtime (s) $\downarrow$\\
    \midrule
\sfhc~\cite{Hruby_cvpr2022} & 5.34 & 1.89 & 59.19 & 72.25 & 82.10 & 2.95 \\ \hline
\sft & 5.18 & 1.86 & 59.48 & 72.35 & 82.16 & \bf{1.99} \\
\sftENM & 5.15 & 1.86 & 59.53 & 72.39 & 82.20 & 2.05 \\
\hline
\sfaf & 5.75 & 1.90 & 59.02 & 71.87 & 81.61 & 2.34 \\
\sfafR & 5.69 & \underline{1.72} & 61.59 & 73.83 & 82.88 & 2.98 \\
\sfafRC & 5.76 & 1.73 & 61.53 & 73.79 & 82.83 & 3.00 \\
\sfafRCENM & 5.34 & \underline{1.72} & 61.71 & 74.00 & 83.09 & 2.87 \\
\hline
\sftm & 5.21 & 1.88 & 59.35 & 72.29 & 82.15 & \bf{1.99} \\
\sftmR & 4.94 & \underline{1.72} & \underline{61.91} & 74.21 & 83.36 & 2.71 \\ 
\sftmRC & 4.94 & \underline{1.72} & 61.88 & 74.22 & 83.35 & 2.71\\ 
\sftmRCENM & 4.93 & \underline{1.72} & 61.84 & 74.21 & 83.38 & 2.73 \\
\hline
\sftmd & 5.09 & 1.89 & 59.41 & 72.50 & 82.38 & \underline{2.01}\\
\sftmdR & 4.90 & \bf{1.71} & 61.90 & 74.26 & 83.42 & 2.76\\
\sftmdRC & \underline{4.88} & \bf{1.71} & \bf{61.95} & \bf{74.31} & \underline{83.47} & 2.75\\ 
\sftmdRCENM & \bf{4.86} & \underline{1.72} & 61.90 & \underline{74.29} & \bf{83.48} & 2.84 \\ 

\midrule
\toprule
    \multicolumn{7}{c}{Cambridge Landmarks~\cite{kendall2015cambridge}} \\
    \midrule

\sfhc~\cite{Hruby_cvpr2022} & 8.13 & 3.05 & 43.75 & 60.73 & 73.93 & 2.37 \\ \hline
\sft & 8.01 & 3.09 & 43.17 & 60.17 & 73.67 & \bf{2.31} \\
\sftENM & 8.09 & 3.11 & 43.15 & 60.02 & 73.50 & 2.48\\
\hline
\sfaf & 8.59 & 3.11 & 43.16 & 59.98 & 73.15 & 2.62\\
\sfafR & 7.95 & \bf{2.80} & 46.27 & 63.20 & 75.86 & 2.96 \\
\sfafRC & 8.05 & \underline{2.81} & 46.28 & 63.17 & 75.75 & 2.98 \\
\sfafRCENM & 7.75 & \underline{2.81} & 46.38 & 63.20 & 75.86 & 2.98 \\
\hline
\sftm & 7.95 & 3.08 & 43.32 & 60.37 & 73.75 & \underline{2.34}\\
\sftmR & 7.22 & \bf{2.80} & \underline{46.48} & \bf{63.52} & \underline{76.30} & 2.86 \\ 
\sftmRC & 8.05 & \underline{2.81} & 46.28 & 63.17 & 75.75 & 2.98 \\ 
\sftmRCENM & 7.75 & \underline{2.81} & 46.38 & 63.20 & 75.86 & 2.98\\
\hline
\sftmd & 7.82 & 3.06 & 43.58 & 60.60 & 73.99 & 2.41 \\
\sftmdR & \underline{7.16} & \bf{2.80} & \bf{46.52} & \underline{63.51} & 76.25 & 3.01 \\
\sftmdRC & \bf{7.12} & \underline{2.81} & 46.33 & 63.47 & \bf{76.35} & 2.95 \\ 
\sftmdRCENM & 7.19 & 2.80 & 46.42 & 63.44 & 76.24 & 3.29 \\ 

\midrule
\toprule
    \multicolumn{7}{c}{Aachen Day-Night v1.1~\cite{zhang2021aachen}} \\
    \midrule

\sfhc~\cite{Hruby_cvpr2022} & 10.73 & 3.84 & 39.94 & 53.01 & 64.77 & 1.90 \\ \hline
\sft & 10.76 & 3.91 & 39.54 & 52.67 & 64.56 & \bf{1.77} \\
\sftENM & 10.78 & 3.79 & 39.86 & 53.06 & 64.79 & 1.89\\
\hline
\sfaf & 11.06 & 3.88 & 39.54 & 52.63 & 64.19 & 2.24\\
\sfafR & 10.09 & 3.50 & 42.64 & 55.75 & 67.01 & 3.32 \\
\sfafRC & 10.22 & 3.49 & 42.58 & 55.75 & 66.87 & 3.38 \\
\sfafRCENM & 10.23 & 3.48 & 42.58 & 55.75 & 66.88 & 3.32 \\
\hline
\sftm & 10.62 & 3.83 & 39.62 & 52.92 & 64.74 & \underline{1.87}\\
\sftmR & 10.11 & \underline{3.46} & \underline{42.69} & \underline{55.90} & \underline{67.11} & 3.17 \\ 
\sftmRC & 10.22 & 3.53 & 42.58 & 55.62 & 66.71 & 3.18 \\ 
\sftmRCENM & \underline{10.06} & 3.52 & 42.55 & 55.76 & 67.00 & 3.20\\
\hline
\sftmd & 10.55 & 3.84 & 39.86 & 53.30 & 65.15 & 1.89 \\
\sftmdR & \bf{10.04} & \bf{3.45} & \bf{42.91} & \bf{56.02} & \bf{67.15} & 3.24 \\
\sftmdRC & 10.15 & 3.47 & 42.65 & 55.76 & 66.94 & 3.17 \\ 
\sftmdRCENM & \bf{10.04} & 3.50 & 42.67 & 55.82 & 66.95 & 3.33 \\ 

\bottomrule

    \end{tabular}}
    \caption{Results for different solvers and strategies implemented in the GC-RANSAC framework~\cite{barath2017graph} for all scenes from the PhotoTourism~\cite{IMC2020}, the Cambridge Landmarks~\cite{kendall2015cambridge} and Aachen Day-Night v1.1~\cite{zhang2021aachen} datasets. 
    We mark the \textbf{best} and \underline{second best} results. Runtimes are reported in seconds for the whole RANSAC with early termination (0.9999 confidence, minimum 100 iterations) and the epipolar threshold set to 5px.}
    \label{tab:gcr_phototourism}
\end{table}

\begin{figure*}[!t]
    \centering
    \resizebox{0.7\linewidth}{!}{
\begin{tikzpicture} 
        \begin{axis}[%
        hide axis, xmin=0,xmax=0,ymin=0,ymax=0,
        legend style={draw=white!15!white, 
        line width = 1pt,
        legend  columns =6, 
        /tikz/every even column/.append style={column sep=0.2cm},
        }
        ]
        
        \addlegendimage{Seaborn3}
        \addlegendentry{\sfaf};
        \addlegendimage{Seaborn4}
        \addlegendentry{\sftm}; 
        \addlegendimage{Seaborn5}
        \addlegendentry{\sftmd};
        \addlegendimage{black!30,dash pattern=on 2pt off 1pt on 2pt off 1pt}
        \addlegendentry{\texttt{+R}};
        \addlegendimage{black!30,dash pattern=on 2pt off 1pt on 0.66pt off 1pt}
        \addlegendentry{\texttt{+R+F}};
        \addlegendimage{black!30,dash pattern=on 0.5pt off 1pt on 0.5pt off 1pt}
        \addlegendentry{\texttt{+R+F+ENM}};
        \end{axis}
    \end{tikzpicture}}
    \begin{subfigure}{0.3\textwidth}
        \includegraphics[width=\linewidth]{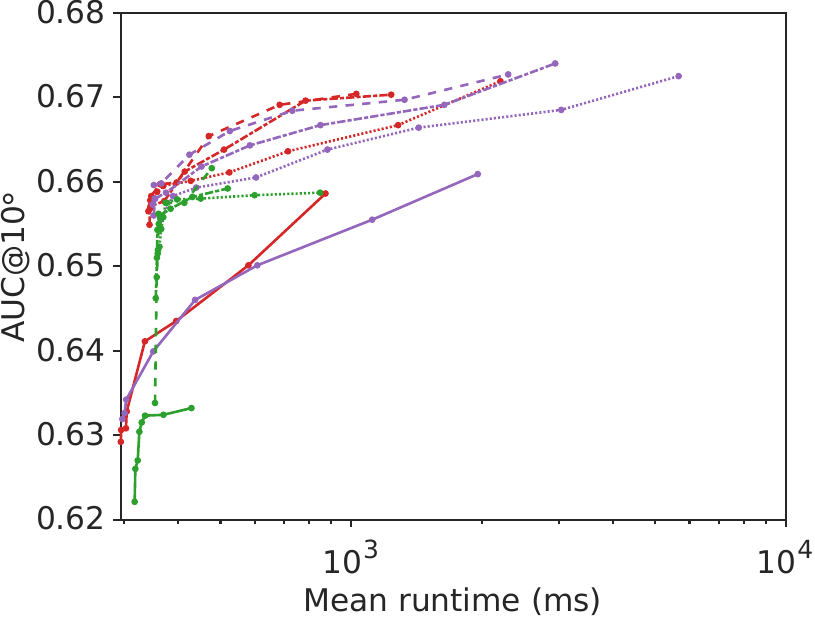}
    \end{subfigure}
    \hfill
    \begin{subfigure}{0.3\textwidth}
        \includegraphics[width=\linewidth]{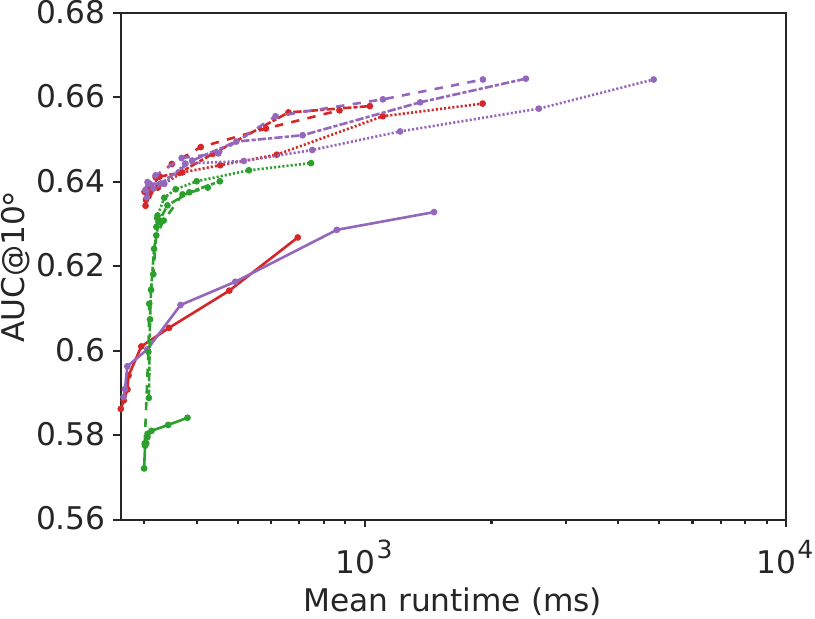}
    \end{subfigure}
    \hfill
    \begin{subfigure}{0.3\textwidth}
        \includegraphics[width=\linewidth]{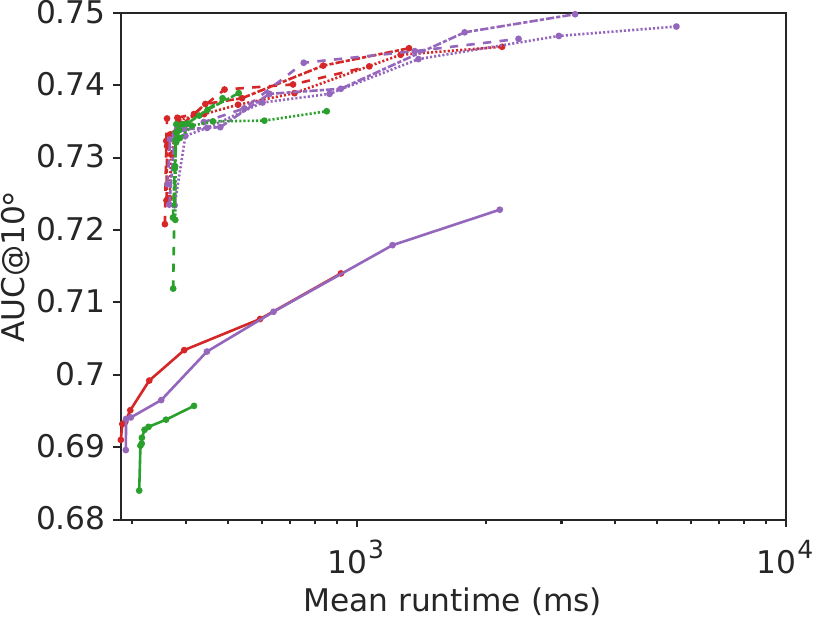}
    \end{subfigure}
    \caption{Speed-accuracy evaluation of various solvers for three view relative pose estimation, evaluated using GC-RANSAC~\cite{barath2017graph} on the (left) \textit{St. Mary's Church}, (middle) \textit{Shop Facade}, and (right) \textit{King's College} scenes from the Cambridge Landmarks dataset~\cite{kendall2015cambridge}. We report the AUC@10$^\circ$ of the pose error and vary the number of  RANSAC iterations ($\{5, 10, 20, 50, 100, 200, 500, 1000\}$) with fixed 5px epipolar threshold.}
    \label{fig:gcr_graphs}
\end{figure*}

\subsection{GC-RANSAC}
\label{sec:exp:gc_ransac}
Besides PoseLib's RANSAC implementation, we also evaluated and compared our proposed solvers with the state-of-the-art solvers inside the GC-RANSAC~\cite{barath2017graph} framework. 

In GC-RANSAC, local optimization (LO) is performed using non-minimal solvers that fit models to larger-than-minimal samples. We use the non-minimal version\footnote{The non-minimal version of the \texttt{5pt} solver~\cite{Nister-5pt-PAMI-2004}  uses the last four vectors from the SVD/QR decomposition of a $n \times 9$ matrix instead of the 4-dim null space of a $5\times 9$ matrix to parameterize the unknown essential matrix.} of the \texttt{5pt} solver~\cite{Nister-5pt-PAMI-2004} and the non-minimal absolute pose \texttt{DLSPnP}~\cite{dlspnp} solver.\footnote{This may not be the most efficient way how to perform non-minimal refitting. However, since all methods use the same LO, it is sufficient for a fair comparison.}
In contrast to LO used in Poselib RANSAC, where the estimated model is used as an initialization of the Levenberg–Marquardt algorithm, in GC-RANSAC, the estimated model is used only to score inliers. 

Tab.~\ref{tab:gcr_phototourism} shows the results 
for GC-RANSAC with PROSAC sampling~\cite{prosac2005chum} and a 5px epipolar threshold for all scenes from the PhotoTourism~\cite{IMC2020} dataset, 5 scenes from the Cambridge Landmarks~\cite{kendall2015cambridge} dataset, and the Aachen Day-Night v1.1~\cite{zhang2021aachen} dataset. 
Similarly to what was observed for Poselib RANSAC (see Table 2 in the main paper), with the suggested modifications, all the proposed
solvers outperform the state-of-the-art 
\sfhc solver~\cite{Hruby_cvpr2022} and the baseline \sft solver in terms of pose accuracy with comparable runtimes. 
Again, the $\delta$-based solvers provide, in general, the best speed-accuracy trade-off.

Due to a different LO, there are several differences compared to the results from Poselib RANSAC. Since in GC-RANSAC the estimated model is used only to score inliers, it does not need to be as precise as in Poselib RANSAC. Thus, even \sfaf solver without any modification provides reasonably precise results.\footnote{Note that the model of \sfaf is refitted in the LO step with the non-minimal \texttt{5pt} solver. This is similar to the refitting used in \texttt{ENM}, \ie, GC-RANSAC's local optimization includes some form of \texttt{ENM}, explaining why the \sfaf performs quite well.} 
In contrast to this, in Poselib RANSAC, the \sfaf solver without any modification results in large errors (see Table 2 in the main paper). 
Without refitting using \texttt{ENM}, the affine model estimated for the first two views in the \sfaf solver is 
not sufficiently precise to provide a good initialization for Levenberg–Marquardt-based optimization in Poselib's LO.
Still, even for GC-RANSAC, the pure \sfaf solver performs 
worse than the remaining variants of the proposed \sfaf-based and \sftm-based solvers.

Another difference is in refitting using \texttt{ENM}. For GC-RANSAC, the effect of \texttt{ENM} is not as significant as for Poselib RANSAC. 
When applied without refinement (\texttt{+R}), the early non-minimal refitting (\texttt{ENM}), in general, increases the precision of solvers. 
When combined with \texttt{+R}, the improvement is not very visible. 
The reason is that the model returned after refining the initial, approximate model estimated by the \sftm-based and \sfaf-based solvers on the $4^{th}$ correspondence 
is usually sufficiently accurate to score inliers.
Moreover, in the LO of GC-RANSAC, this approximate model is refitted using the non-minimal \texttt{5pt} solver (which is similar to the refitting that is used in \texttt{ENM}), and the non-minimal \texttt{DLSPnP} solver~\cite{dlspnp}. Similarly to \texttt{ENM}, the filtering (\texttt{+F}) that uses the $4^{th}$ correspondence does not bring an improvement that is as visible as for Poselib RANSAC. This is because for GC-RANSAC, the speedup obtained using the filtering \texttt{+F} is not as significant, compared to the longer running times of the LO part of GC-RANSAC. 

On the other hand, the remaining two suggested modifications, \ie, the $\delta$-based solvers and the refinement (\texttt{+R}) using the $4^{th}$ correspondence bring visible improvements. 
This behavior is also visible in Figure~\ref{fig:gcr_graphs}. Here we present an ablation study on the effects of the various modifications ($\delta$ and \texttt{+F/+R/+ENM}, which were introduced in Sec. 3.2 of the main paper) on the \sftm-based and \sfaf-based solvers. 
The results are reported on the \textit{St.~Mary's Church}, \textit{Shop Facade}, and \textit{King's College} scenes from the Cambridge Landmarks dataset~\cite{kendall2015cambridge}. These results especially highlight the importance of the refinement using the $4^{th}$ point in the third view (\texttt{+R}).
On the other hand, the benefits of \sftmd-based solvers over the \sftm-based solvers that are also visible in Table~\ref{tab:gcr_phototourism} are not so significant as in Poselib RANSAC. 
Still, \sftmd-based solvers lead to improved pose accuracy. 

{\small
\bibliographystyle{ieeenat_fullname}
\bibliography{bibliography}

@String(CVPR= {IEEE Conf. Comput. Vis. Pattern Recog.})

@String(ICCV= {Int. Conf. Comput. Vis.})

@String(ECCV= {Eur. Conf. Comput. Vis.})

@String(ICPR = {Int. Conf. Pattern Recog.})

@String(PAMI  = {IEEE TPAMI})

@String(CVPR  = {CVPR})

@String(ICCV  = {ICCV})

@String(ECCV  = {ECCV})

@String(ICPR  = {ICPR})

@string{icpr = "International Conference on Pattern Recognition (ICPR)"}

@string{eccv = "European Conference on Computer Vision (ECCV)"}

@string{iccv = "International Conference on Computer Vision (ICCV)"}

@string{cvpr = "Computer Vision and Pattern Recognition (CVPR)"}

@inproceedings{larsson2017efficient,
  title={Efficient Solvers for Minimal Problems by Syzygy-based Reduction},
  author={Larsson, Viktor and {\AA}str{\"o}m, Kalle and Oskarsson, Magnus},
  booktitle=cvpr,
  pages={820--829},
  year={2017}
 }

@inproceedings{detone2018superpoint,
  title={{Superpoint: Self-supervised interest point detection and description}},
  author={DeTone, Daniel and Malisiewicz, Tomasz and Rabinovich, Andrew},
  pages={224--236},
  booktitle={Proceedings of the IEEE Conference on Computer Vision and Pattern Recognition Workshops},
  year={2018}
}

@inproceedings{lindenberger2023lightglue,
  title={LightGlue: Local Feature Matching at Light Speed},
  author={Lindenberger, Philipp and Sarlin, Paul-Edouard and Pollefeys, Marc},
  booktitle=iccv,
  pages={17627--17638},
  year={2023},

}

@ARTICLE{Fischler-Bolles-ACM-1981,
  author = {Fischler, M. A. and Bolles, R. C.},
  title = {Random sample consensus: a paradigm for model fitting with applications
	to image analysis and automated cartography},
  journal = {Commun. ACM},
  year = {1981},
  volume = {24},
  pages = {381--395},
  number = {6},
  month = jun,
  issn = {0001-0782},
  issue_date = {June 1981},
  publisher = {ACM},
  url = {http://doi.acm.org/10.1145/358669.358692}
}

@article{Nister-5pt-PAMI-2004,
	Author = {Nist\'er, D.},
	Journal = {{IEEE} Transactions on Pattern Analysis and Machine Intelligence},
	Month = jun,
	Number = {6},
	Pages = {756--770},
	Title = {An Efficient Solution to the Five-Point Relative Pose Problem},
	Volume = {26},
	Year = {2004}}

@article{Stewenius-ISPRS-2006,
	Author = {Stewenius, H. and Engels, C. and Nist\'er, D.},
	Journal = {ISPRS J. of Photogrammetry and Remote Sensing},
	Pages = {284--294},
	Title = {Recent developments on direct relative orientation},
	Volume = {60},
	Year = {2006}}

@Article{Snavely-IJCV-2008,
  author = 	 {Snavely, N. and Seitz, S. M. and Szeliski, R.},
  title = 	 {Modeling the world from internet photo collections.},
  journal = 	 {International Journal Computer Vision},
  year = 	 {2008},
  volume = 	 {80},
  number = 	 {2},
  pages = 	 {189--210},
  month = 	 {Nov.},
}

@INPROCEEDINGS{Chum-2003,
  title={Locally optimized RANSAC},
  author={Chum, Ond{\v{r}}ej and Matas, Ji{\v{r}}{\'\i} and Kittler, Josef},
  booktitle={Pattern Recognition},
  pages={236--243},
  year={2003},
  publisher={Springer Berlin Heidelberg}
}

@INPROCEEDINGS{Stewenius-CVPR-2005,
author = {H.~Stewenius and D.~Nister and F.~Kahl and F.~Schaffalitzky},
title = {A minimal solution for relative pose with unknown focal length},
booktitle ="{IEEE} Conference on Computer Vision and Pattern Recognition (CVPR 2005)",
year = {2005}
}

@article{Sattler16PAMI,
author = {T. Sattler and B. Leibe and L. Kobbelt},
Journal = {{IEEE} Transactions on Pattern Recognition and Machine Intelligence},
title = {{Efficient \& Effective Prioritized Matching for Large-Scale Image-Based Localization}},
year = {2016},
  volume={39},
  number={9},
  pages={1744--1756},
  publisher={IEEE}
}

@article{DBLP:journals/pami/RaguramCPMF13,
  author    = {R. Raguram and
               O. Chum and
               M. Pollefeys and
               J. Matas and
               J.{-}M. Frahm},
  title     = {{USAC:} {A} Universal Framework for Random Sample Consensus},
  journal   = {{IEEE} Transactions on Pattern Recognition and Machine Intelligence},
  volume    = {35},
  number    = {8},
  pages     = {2022--2038},
  year      = {2013},
  url       = {http://dx.doi.org/10.1109/TPAMI.2012.257},
  doi       = {10.1109/TPAMI.2012.257},
  bibsource = {dblp computer science bibliography, http://dblp.org}
}

@article{DBLP:journals/ram/ScaramuzzaF11,
  author    = {D. Scaramuzza and
               F. Fraundorfer},
  title     = {Visual Odometry [Tutorial]},
  journal   = {{IEEE} Robot. Automat. Mag.},
  volume    = {18},
  number    = {4},
  pages     = {80--92},
  year      = {2011},
  url       = {http://dx.doi.org/10.1109/MRA.2011.943233},
  doi       = {10.1109/MRA.2011.943233},
  bibsource = {dblp computer science bibliography, http://dblp.org}
}

@article{DBLP:journals/ijcv/NisterS06,
  author    = {D. Nist{\'{e}}r and
               F. Schaffalitzky},
  title     = {Four Points in Two or Three Calibrated Views: Theory and Practice},
  journal   = {International Journal of Computer Vision},
  volume    = {67},
  number    = {2},
  pages     = {211--231},
  year      = {2006},
  url       = {http://dx.doi.org/10.1007/s11263-005-4265-x},
  doi       = {10.1007/s11263-005-4265-x},
  bibsource = {dblp computer science bibliography, http://dblp.org}
}

@article{Quan2006,
  author    = {L. Quan and
               B. Triggs and
               B. Mourrain},
  title     = {Some Results on Minimal Euclidean Reconstruction from Four Points},
  journal   = {Journal of Mathematical Imaging and Vision},
  volume    = {24},
  number    = {3},
  pages     = {341--348},
  year      = {2006},
  url       = {http://dx.doi.org/10.1007/s10851-005-3632-0},
  doi       = {10.1007/s10851-005-3632-0},
  bibsource = {dblp computer science bibliography, http://dblp.org}
}

@inproceedings{DBLP:conf/cvpr/KukelovaP07,
  author    = {Z. Kukelova and
               T. Pajdla},
  title     = {A minimal solution to the autocalibration of radial distortion},
  booktitle = {{IEEE} Conference on Computer Vision and Pattern
               Recognition {(CVPR 2007)}},
  year      = {2007},
  url       = {http://dx.doi.org/10.1109/CVPR.2007.383063},
  doi       = {10.1109/CVPR.2007.383063},
  bibsource = {dblp computer science bibliography, http://dblp.org}
}

@Article{Torr97a,
   author = "Torr, P.~H.~S. and Zisserman, A.",
   title = "Robust Parameterization and Computation of the
   Trifocal Tensor",
   journal = "Image and Vision Computing",
   pages = "591--605",
   volume = "15",
   year = "1997",
   URL = "http://www.robots.ox.ac.uk/~vgg",
}

@article{Oskarsson_bmvc2004,
title = {Minimal projective reconstruction for combinations of points and lines in three views},
journal = {Image and Vision Computing},
volume = {22},
number = {10},
pages = {777-785},
year = {2004},
note = {British Machine Vision Computing 2002},
issn = {0262-8856},
doi = {https://doi.org/10.1016/j.imavis.2004.02.004},
url = {https://www.sciencedirect.com/science/article/pii/S0262885604000459},
author = {Magnus Oskarsson and Andrew Zisserman and Kalle \AA str\"{o}m},
}

@article{Quan_pami95,
    author = {Long Quan},
    title = {Invariants of six points and projective reconstruction from
three uncalibrated images},
    journal = {IEEE Transactions on Pattern Analysis and Machine Intelligence (PAMI)},
    year = {1995},
    volume = {17},
    number = {1},
    pages = {34--46}
}

@inproceedings{Duff_PL1P,
  author    = {Timothy Duff and
               Kathl{\'{e}}n Kohn and
               Anton Leykin and
               Tom{\'{a}}s Pajdla},
  title     = {PL\({}_{\mbox{1}}\)P - Point-Line Minimal Problems Under Partial Visibility
               in Three Views},
  booktitle = {{ECCV} {(26)}},
  series    = {Lecture Notes in Computer Science},
  volume    = {12371},
  pages     = {175--192},
  publisher = {Springer},
  year      = {2020}
}

@inproceedings{Fabbri_CVPR2020,
  author    = {Ricardo Fabbri and
               Timothy Duff and
               Hongyi Fan and
               Margaret H. Regan and
               David da Costa de Pinho and
               Elias P. Tsigaridas and
               Charles W. Wampler and
               Jonathan D. Hauenstein and
               Peter J. Giblin and
               Benjamin B. Kimia and
               Anton Leykin and
               Tom{\'{a}}s Pajdla},
  title     = {{TRPLP} - Trifocal Relative Pose From Lines at Points},
  booktitle = {{CVPR}},
  pages     = {12070--12080},
  publisher = {Computer Vision Foundation / {IEEE}},
  year      = {2020}
}

@book{SommeseAndrewJ2005Tnso,
author = {Sommese, Andrew J and Sommese, Andrew J and Wampler, Charles W},
address = {Singapore},
isbn = {9789812561848},
language = {eng},
publisher = {World Scientific Publishing Co. Pte. Ltd},
title = {The numerical solution of systems of polynomials arising in engineering and science},
year = {2005},
}

@inproceedings{leonardos_cvpr2015,
  title={A metric
parametrization for trifocal tensors with non-colinear pinholes},
  author={Leonardos, S. and Tron, R. and  Daniilidis, K.},
  booktitle=cvpr,
  pages={259--267},
  year={2015},
  organization={IEEE}
}

@article{Martyushev16,
  author    = {Evgeniy V. Martyushev},
  title     = {On Some Properties of Calibrated Trifocal Tensors},
  journal   = {Journal of Mathematical Imaging and Vision},
  volume    = {58},
  year      = {2017},
     number = {2},
    pages = {321--332}
}

@article{Aholt2014,
 ISSN = {00255718, 10886842},
 URL = {http://www.jstor.org/stable/24488612},
 author = {Aholt, Chris and Oeding, Luke},
 journal = {Mathematics of Computation},
 number = {289},
 pages = {2553--2574},
 publisher = {American Mathematical Society},
 title = {THE IDEAL OF THE TRIFOCAL VARIETY},
 urldate = {2023-03-05},
 volume = {83},
 year = {2014}
}

@ARTICLE{holt1995,
  author={Holt, R.J. and Netravali, A.N.},
  journal={IEEE Transactions on Pattern Analysis and Machine Intelligence}, 
  title={Uniqueness of solutions to three perspective views of four points}, 
  year={1995},
  volume={17},
  number={3},
  pages={303-307},
  doi={10.1109/34.368195}}

@InProceedings{lambda-twist,
author = {Persson, Mikael and Nordberg, Klas},
title = {Lambda Twist: An Accurate Fast Robust Perspective Three Point (P3P) Solver},
booktitle = {Proceedings of the European Conference on Computer Vision (ECCV)},
month = {September},
  pages={318--332},
year = {2018}
}

@InProceedings{Hruby_cvpr2022,
    author    = {Hruby, Petr and Duff, Timothy and Leykin, Anton and Pajdla, Tomas},
    title     = {Learning To Solve Hard Minimal Problems},
    booktitle = {Proceedings of the IEEE/CVF Conference on Computer Vision and Pattern Recognition (CVPR)},
    month     = {June},
    year      = {2022},
    pages     = {5532-5542}
}

@inproceedings{bujnak_cvpr2008,
  title={A general solution to the P4P problem for camera with unknown focal length},
  author={Bujnak, Martin and Kukelova, Zuzana and Pajdla, Tomas},
  booktitle={2008 IEEE Conference on Computer Vision and Pattern Recognition},
  pages={1--8},
  year={2008},
  organization={IEEE}
}

@inproceedings{larsson2019revisiting,
  title={Revisiting radial distortion absolute pose},
  author={Larsson, Viktor and Sattler, Torsten and Kukelova, Zuzana and Pollefeys, Marc},
  booktitle={Proceedings of the IEEE/CVF International Conference on Computer Vision},
  pages={1062--1071},
  year={2019}
}

@inproceedings{kukelova2013real,
  title={Real-time solution to the absolute pose problem with unknown radial distortion and focal length},
  author={Kukelova, Zuzana and Bujnak, Martin and Pajdla, Tomas},
  booktitle={Proceedings of the IEEE International Conference on Computer Vision},
  pages={2816--2823},
  year={2013}
}

@article{Kileel2017,
  author    = {Kileel, J.},
  title     = {Minimal problems for the calibrated trifocal variety.},
  journal   = {SIAM Journal on Applied Algebra and Geometry},
  volume    = {1},
  year      = {2017},
     number = {1},
    pages = {575--598}
}

@inproceedings{barath2017graph,
  title={Graph-{C}ut {RANSAC}},
  author={Barath, D. and Matas, J.},
  booktitle={Proceedings of the IEEE conference on computer vision and pattern recognition},
  pages={6733--6741},
  year={2018}
}

@inproceedings{Rodehorst2017,
  author    = {Volker Rodehorst},
  title     = {Evaluation of the metric trifocal tensor for relative three-view orientation},
  booktitle = {Digital Proceedings, International Conference on the Applications of Computer Science and Mathematics in Architecture and Civil Engineering : July 20 - 22 2015, Bauhaus-University Weimar},
  editor    = {Klaus G{\"u}rlebeck and Tom Lahmer},
  doi       = {10.25643/bauhaus-universitaet.2817},
  year      = {2017},
}

@misc{PoseLib,
  title = {{PoseLib - Minimal Solvers for Camera Pose Estimation}},
  author = {Viktor Larsson},
  URL = {https://github.com/vlarsson/PoseLib},
  year = {2020}
}

@article{IMC2020,
  title={Image matching across wide baselines: From paper to practice},
  author={Jin, Yuhe and Mishkin, Dmytro and Mishchuk, Anastasiia and Matas, Jiri and Fua, Pascal and Yi, Kwang Moo and Trulls, Eduard},
  journal={International Journal of Computer Vision},
  volume={129},
  number={2},
  pages={517--547},
  year={2021},
  publisher={Springer}
}

@incollection{snavely2006photo,
  title={Photo tourism: exploring photo collections in 3D},
  author={Snavely, Noah and Seitz, Steven M and Szeliski, Richard},
  booktitle={ACM siggraph 2006 papers},
  pages={835--846},
  year={2006}
}

@InProceedings{Higgins-4p3v91,
author="Longuet-Higgins, H. Christopher",
editor="Mowforth, Peter",
title="A method of obtaining the relative positions of 4 points from 3 perspective projections",
booktitle="BMVC91",
year="1991",
publisher="Springer London",
address="London",
pages="86--94",
isbn="978-1-4471-1921-0"
}

@ARTICLE{xu_ortho4p3v,
  author={Xu, G. and Sugimoto, N.},
  journal={IEEE Transactions on Pattern Analysis and Machine Intelligence}, 
  title={A linear algorithm for motion from three weak perspective images using Euler angles}, 
  year={1999},
  volume={21},
  number={1},
  pages={54-57},
  doi={10.1109/34.745734}}

@article{hane20173d,
  title={3D visual perception for self-driving cars using a multi-camera system: Calibration, mapping, localization, and obstacle detection},
  author={H{\"a}ne, Christian and Heng, Lionel and Lee, Gim Hee and Fraundorfer, Friedrich and Furgale, Paul and Sattler, Torsten and Pollefeys, Marc},
  journal={Image and Vision Computing},
  volume={68},
  pages={14--27},
  year={2017},
  publisher={Elsevier}
}

@inproceedings{Castle08ISWC,
 author = {Robert Castle and Georg Klein and David W. Murray},
 title = {Video-rate Localization in Multiple Maps for Wearable Augmented Reality},
 booktitle={ISWC},
year = {2008},
}

@InProceedings{Schops_2017_CVPR,
author = {Schops, Thomas and Schonberger, Johannes L. and Galliani, Silvano and Sattler, Torsten and Schindler, Konrad and Pollefeys, Marc and Geiger, Andreas},
title = {{A Multi-View Stereo Benchmark With High-Resolution Images and Multi-Camera Videos}},
booktitle = {Proceedings of the IEEE Conference on Computer Vision and Pattern Recognition (CVPR)},
month = {July},
year = {2017}
}

@InProceedings{Schoenberger2016CVPR,
author = {Sch\"{o}nberger, Johannes L. and Frahm, Jan-Michael},
title = {{Structure-From-Motion Revisited}},
  booktitle={Proceedings of the IEEE conference on computer vision and pattern recognition},
  pages={4104--4113},
  year={2016}
}

@inproceedings{perd2006epipolar,
  title={Epipolar geometry from two correspondences},
  author={Perdoch, Michal and Matas, Jiri and Chum, Ondrej},
  booktitle={18th International Conference on Pattern Recognition (ICPR'06)},
  volume={4},
  pages={215--219},
  year={2006},
  organization={IEEE}
}

@inproceedings{pritts-cvpr2018,
  author    = {James Pritts and
               Zuzana Kukelova and
               Viktor Larsson and
               Ondrej Chum},
  title     = {Radially-Distorted Conjugate Translations},
  booktitle = {2018 {IEEE} Conference on Computer Vision and Pattern Recognition,
               {CVPR} 2018, Salt Lake City, UT, USA, June 18-22, 2018},
  pages     = {1993--2001},
  publisher = {{IEEE} Computer Society},
  year      = {2018},
  url       = {http://openaccess.thecvf.com/content\_cvpr\_2018/html/Pritts\_Radially-Distorted\_Conjugate\_Translations\_CVPR\_2018\_paper.html},
  doi       = {10.1109/CVPR.2018.00213},
  bibsource = {dblp computer science bibliography, https://dblp.org}
}

@inproceedings{pritts_ivcnz13,
  author    = {James Pritts and
               Ondrej Chum and
               Jiri Matas},
  title     = {Approximate models for fast and accurate epipolar geometry estimation},
  booktitle = {28th International Conference on Image and Vision Computing New Zealand,
               {IVCNZ} 2013, Wellington, New Zealand, November 27-29, 2013},
  pages     = {106--111},
  publisher = {{IEEE}},
  year      = {2013},
  url       = {https://doi.org/10.1109/IVCNZ.2013.6727000},
  doi       = {10.1109/IVCNZ.2013.6727000},
  bibsource = {dblp computer science bibliography, https://dblp.org}
}

@inproceedings{duff2019plmp,
  title={PLMP-point-line minimal problems in complete multi-view visibility},
  author={Duff, Timothy and Kohn, Kathlen and Leykin, Anton and Pajdla, Tomas},
  booktitle={Proceedings of the IEEE/CVF International Conference on Computer Vision},
  pages={1675--1684},
  year={2019}
}

@InProceedings{Ding_2023_ICCV,
    author    = {Ding, Yaqing and Chien, Chiang-Heng and Larsson, Viktor and \r{A}str\"om, Karl and Kimia, Benjamin},
    title     = {Minimal Solutions to Generalized Three-View Relative Pose Problem},
    booktitle = {Proceedings of the IEEE/CVF International Conference on Computer Vision (ICCV)},
    month     = {October},
    year      = {2023},
    pages     = {8156-8164}
}

@inproceedings{kendall2015cambridge,
  title={Posenet: A convolutional network for real-time 6-dof camera relocalization},
  author={Kendall, Alex and Grimes, Matthew and Cipolla, Roberto},
  booktitle={Proceedings of the IEEE international conference on computer vision},
  pages={2938--2946},
  year={2015}
}

@InProceedings{Cin_2024_CVPR,
    author    = {Cin, Andrea Porfiri Dal and Duff, Timothy and Magri, Luca and Pajdla, Tomas},
    title     = {Minimal Perspective Autocalibration},
    booktitle = {Proceedings of the IEEE/CVF Conference on Computer Vision and Pattern Recognition (CVPR)},
    month     = {June},
    year      = {2024},
    pages     = {5064-5073}
}

@InProceedings{Chien_2022_CVPR,
    author    = {Chien, Chiang-Heng and Fan, Hongyi and Abdelfattah, Ahmad and Tsigaridas, Elias and Tomov, Stanimire and Kimia, Benjamin},
    title     = {GPU-Based Homotopy Continuation for Minimal Problems in Computer Vision},
    booktitle = {Proceedings of the IEEE/CVF Conference on Computer Vision and Pattern Recognition (CVPR)},
    month     = {June},
    year      = {2022},
    pages     = {15765-15776}
}

@article{zhang2021aachen,
  title={Reference pose generation for long-term visual localization via learned features and view synthesis},
  author={Zhang, Zichao and Sattler, Torsten and Scaramuzza, Davide},
  journal={International Journal of Computer Vision},
  volume={129},
  number={4},
  pages={821--844},
  year={2021},
  publisher={Springer}
}

@Inbook{Zhang2014,
author="Zhang, Zhengyou",
editor="Ikeuchi, Katsushi",
title="Weak Perspective Projection",
bookTitle="Computer Vision: A Reference Guide",
year="2014",
publisher="Springer US",
address="Boston, MA",
pages="877--883",
isbn="978-0-387-31439-6",
doi="10.1007/978-0-387-31439-6_115",
url="https://doi.org/10.1007/978-0-387-31439-6_115"
}

@InProceedings{BEEP,
author="Goshen, Liran
and Shimshoni, Ilan",
editor="Leonardis, Ale{\v{s}}
and Bischof, Horst
and Pinz, Axel",
title="Balanced Exploration and Exploitation Model Search for Efficient Epipolar Geometry Estimation",
booktitle="Computer Vision -- ECCV 2006",
year="2006",
publisher="Springer Berlin Heidelberg",
address="Berlin, Heidelberg",
pages="151--164",
isbn="978-3-540-33835-2"
}

@article{Horaud97,
author = {Horaud, Radu and Dornaika, Fadi and Lamiroy, Bart},
title = {Object Pose: The Link between Weak Perspective,Paraperspective, and Full Perspective},
year = {1997},
issue_date = {March 1997},
publisher = {Kluwer Academic Publishers},
address = {USA},
volume = {22},
number = {2},
issn = {0920-5691},
url = {https://doi.org/10.1023/A:1007940112931},
doi = {10.1023/A:1007940112931},
journal = {Int. J. Comput. Vision},
month = mar,
pages = {173–189},
numpages = {17}
}

@inproceedings{lebedaLO,
  title={Fixing the Locally Optimized RANSAC},
  author={Lebeda, Karel and Matas, Jir{\'\i} and Chum, Ondrej},
  booktitle={Procedings of the British Machine Vision Conference 2012},
  year={2012},
  organization={British Machine Vision Association},
 volume={44},
  number={9},
  pages={4961--4974},
}

@misc{hartley2006multiple,
  title={Multiple view geometry in computer vision},
  author={Hartley, Richard and Zisserman, Andrew},
  year={2006},
  publisher={Cambridge university press Cambridge}
}

@article{oberkampf1996,
  title={Iterative pose estimation using coplanar feature points},
  author={Oberkampf, Denis and DeMenthon, Daniel F and Davis, Larry S},
  journal={Computer Vision and Image Understanding},
  volume={63},
  number={3},
  pages={495--511},
  year={1996},
  publisher={Elsevier}
}

@INPROCEEDINGS{dlspnp,
  author={Hesch, Joel A. and Roumeliotis, Stergios I.},
  booktitle={2011 International Conference on Computer Vision}, 
  title={A Direct Least-Squares (DLS) method for PnP}, 
  year={2011},
  volume={},
  number={},
  pages={383-390},
  doi={10.1109/ICCV.2011.6126266}}

@article{Collins2014InfinitesimalPP,
  title={Infinitesimal Plane-Based Pose Estimation},
  author={Toby Collins and Adrien Bartoli},
  journal={International Journal of Computer Vision},
  year={2014},
  volume={109},
  pages={252-286},
  url={https://api.semanticscholar.org/CorpusID:8674118}
}

@inproceedings{prosac2005chum,
  title={Matching with PROSAC-progressive sample consensus},
  author={Chum, Ondrej and Matas, Jiri},
  booktitle={2005 IEEE computer society conference on computer vision and pattern recognition (CVPR'05)},
  volume={1},
  pages={220--226},
  year={2005},
  organization={IEEE}
}
}

\end{document}